\documentclass{article}

\usepackage{microtype}
\usepackage{subcaption}
\usepackage{booktabs}
\usepackage{multirow}
\usepackage[table,dvipsnames]{xcolor}
\usepackage{pifont}

\usepackage{hyperref}

\PassOptionsToPackage{noend}{algorithmic}
\usepackage[accepted]{icml2026}

\usepackage{amsmath}
\usepackage{amssymb}
\usepackage{mathtools}
\usepackage{amsthm}
\usepackage{makecell}
\usepackage{pifont}

\usepackage[capitalize,noabbrev]{cleveref}

\usepackage{graphicx}

\usepackage{soul}

\theoremstyle{plain}
\newtheorem{theorem}{Theorem}[section]
\newtheorem{proposition}[theorem]{Proposition}

\newtheorem{corollary}[theorem]{Corollary}
\theoremstyle{definition}

\theoremstyle{remark}

\newcommand{\NND}{\ensuremath{\text{NND}}}
\newcommand{\NN}{\ensuremath{\text{NN}}}
\newcommand{\Var}{\text{Var}}

\usepackage[disable,textsize=tiny]{todonotes}
\icmltitlerunning{GICDM}

\begin{document}
\twocolumn[
  \icmltitle{GICDM: Mitigating Hubness for\\Reliable Distance-Based Generative Model Evaluation}
  \icmlsetsymbol{equal}{*}

  \begin{icmlauthorlist}
    \icmlauthor{Nicolas Salvy}{in,CS,UPS}
    \icmlauthor{Hugues Talbot}{CS,in,UPS}
    \icmlauthor{Bertrand Thirion}{in,CEA,UPS}
  \end{icmlauthorlist}

  \icmlaffiliation{in}{Inria, Palaiseau, France}
  \icmlaffiliation{CS}{CentraleSup\'elec, Gif-sur-Yvette, France}
  \icmlaffiliation{UPS}{Universit\'e Paris-Saclay, Gif-sur-Yvette, France}
  \icmlaffiliation{CEA}{CEA, Gif-sur-Yvette, France}

  \icmlcorrespondingauthor{Nicolas Salvy}{nicolas.salvy@inria.fr}

  \icmlkeywords{Hubness, Curse of Dimensionality, Generative Model Evaluation, Density, Coverage}

  \vskip 0.3in
]

\printAffiliationsAndNotice{}

\begin{abstract}
  Generative model evaluation commonly relies on high-dimensional embedding spaces to compute distances between samples.
  We show that dataset representations in these spaces are affected by the hubness phenomenon, which distorts nearest-neighbor
  relationships and biases distance-based metrics.
  Building on the classical \emph{Iterative} Contextual Dissimilarity Measure (ICDM),
  we introduce Generative ICDM (GICDM), a method to correct neighborhood estimation for both real and generated data.
  We introduce a multi-scale extension to improve empirical behavior. 
  Extensive experiments on synthetic and real benchmarks demonstrate that GICDM resolves hubness-induced failures, restores reliable metric behavior, and improves alignment with human assessment.
\end{abstract}

\section{Introduction}

Generative models have achieved significant progress in recent years, 
enabling the synthesis of data for arbitrary modalities, 
including critical areas such as medical imaging \cite{pinaya2022brain,koetzier2024generating,bluethgen2024vision}. 
Evaluating the quality of generated data is essential to ensure its reliability for downstream applications.

Yet, structured data, such as images, are high-dimensional, which makes direct density
estimation infeasible. A common approach for evaluating generative models is to
use \emph{simple} distributional approximations. For example, Fréchet Inception Distance (FID) variants \cite{heusel2017gans,stein2024exposing} assume normality of the distribution.
This enables evaluation via a single score, typically the distance between the real and synthetic approximated distributions.
However, beyond the untested assumptions, this aggregate score makes it difficult to diagnose
whether a generative model lacks realism or diversity~\cite{sajjadi2018assessing}.

Pairs of fidelity and coverage metrics aim to measure these two aspects separately,
typically using distances \cite{sajjadi2018assessing,naeem2020reliable,salvy2026enhanced}.
A synthetic sample is then considered realistic (high fidelity) if it is sufficiently close to real data,
while a real point is considered covered if there are synthetic samples sufficiently close to it.
The closeness threshold is usually defined locally
as the distance from each real point to its $k$-th nearest real neighbor.

\begin{figure}[t]
  \begin{center}
    \subfloat[Visualization]
    {\hspace{-2.9em}\vspace{-0.3em}\includegraphics[width=0.79\columnwidth]{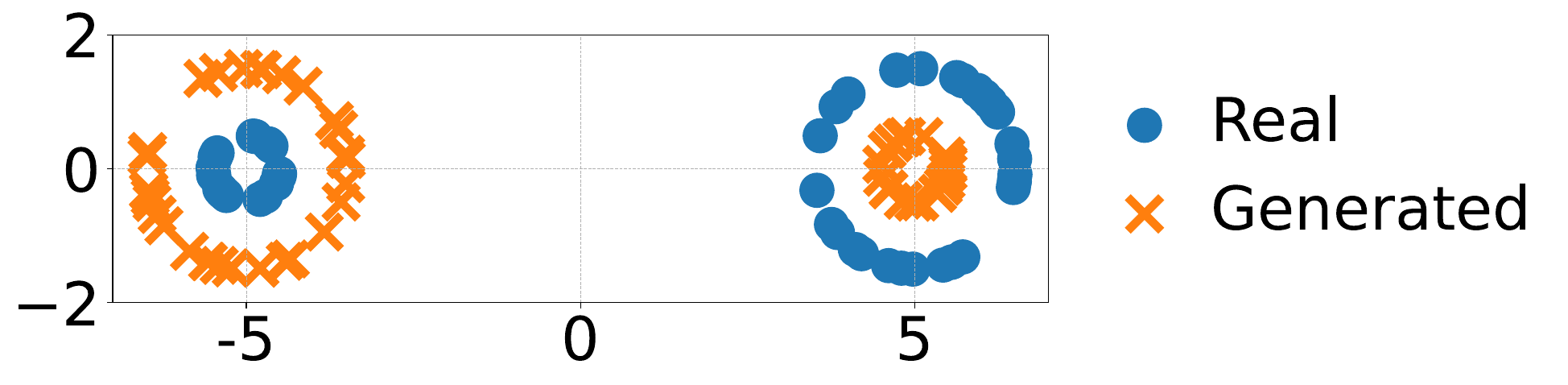}}
    \vfill
    \vspace{0.6em}
    \subfloat[Corresponding metric scores]{\vspace{-0.3em}\includegraphics[width=1\columnwidth]{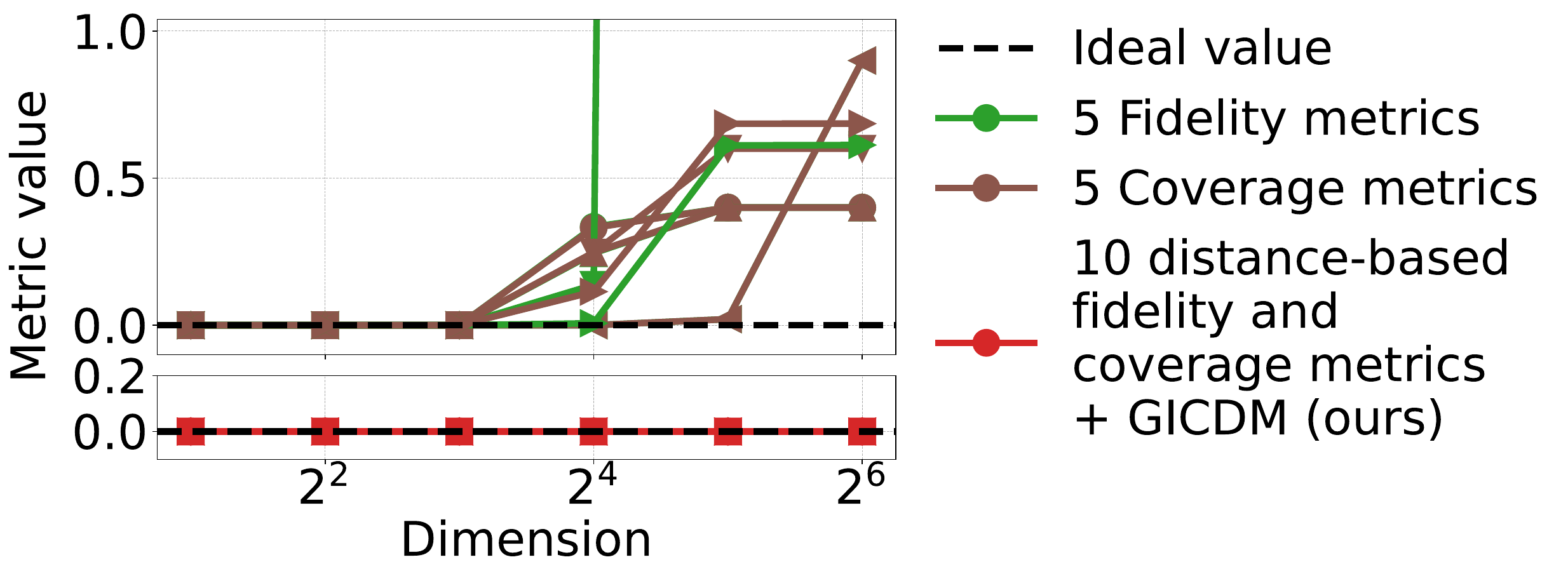}}
    \caption{
    \textbf{High dimension challenges distance-based metrics}:
    The scenario in \textit{(a)} consists of 
    real samples uniformly drawn from a 60/40 mixture of two hyperspheres, and 
    generated samples uniformly drawn from a mixture with swapped radii and proportions. 
    Standard metrics for this scenario are plotted in \textit{(b)} as the dimension increases.
    The real and generated sets are disjoint, so all metrics should score 0; 
    however, none of the standard metrics (top) do in high dimensions due to hubness: on the right side generated points spuriously appear closer to real points than real points are to each other.
    After applying GICDM (bottom), their scores correctly remain at 0. See \Cref{sec:hypersphere_metric_wise_results} for individual metric results.
    }
    \label{fig:hypersphere_test}
  \end{center}
\end{figure}

A recent position paper argued that all existing fidelity and diversity metrics are flawed \cite{raisa2025position}.
This paper introduced a synthetic benchmark of tests for evaluating generative model metrics 
and showed that all current metrics fail to meet at least 40\% of the success criteria.
While subsequent work has led to improvements \cite{salvy2026enhanced}, many failures persist.
We argue that the underlying cause of many reported failures is related to high-dimensionality,
specifically the \emph{hubness} phenomenon.
As shown in \Cref{fig:hypersphere_test}, all metrics incorrectly yield non-zero scores for two disjoint uniform distributions on mixtures of hyperspheres in high dimensions. This happens because hubness distorts distances, making generated points spuriously appear closer to real points than real points are to each other.

Distance-based evaluation assumes that feature space distances are meaningful: 
close points should be structurally and semantically similar, and distant points should be structurally or semantically different. In practice, evaluation is performed in pre-trained embedding spaces,
as feature extractors provide richer semantic representations than the raw observation space. 
However, modern embedding spaces are usually high-dimensional 
(e.g., 2048 for InceptionV3 \cite{DBLP:conf/cvpr/SzegedyVISW16},
4096 for DINOv3 \cite{DBLP:journals/corr/abs-2508-10104},
1024 for CLAP \cite{DBLP:conf/icassp/ElizaldeDIW23}),
making them vulnerable to the curse of dimensionality \cite{Bellman1961}.

A key aspect of this curse, known as \emph{hubness} \cite{radovanovic2010hubs},
undermines the reliability of nearest-neighbor relationships in high-dimensional spaces
\cite{beyer1999nearest,aggarwal2001surprising}. Specifically, certain points, called \emph{hubs},
appear disproportionately often among the $k$-nearest-neighbors of other data points,
even when they are semantically unrelated \cite{pachet2004improving}. 
At the same time, many other points, called \emph{antihubs}, never appear as nearest-neighbors,
meaning they effectively vanish from distance-based evaluation.
This phenomenon arises from the structure of high-dimensional data distributions and is
not simply due to limited sample size \cite{radovanovic2010hubs}.

Hubness has been recognized as problematic in various domains such as image recognition \cite{DBLP:conf/iccp2/TomasevBMN11} and recommender systems \cite{hara2015reducing}.
We show that hubness also affects modern embedding spaces (\Cref{tab:hubness_measures}).

To enable reliable distance-based generative model evaluation in high-dimensional spaces,
we aim to mitigate hubness while preserving metric validity.
Our goals are:
(1) to reduce hubness in the real dataset so that
nearest-neighbor
relationships are meaningful and symmetric,
(2) to preserve the relative positioning of
generated
points
with
respect to real data, and
(3) to ensure that the evaluation of each
generated point is independent of the
others, meaning its assessment depends solely on the real data.
Standard hubness
mitigation methods focus only on in-sample reduction and do not directly satisfy these requirements \cite{feldbauer2019comprehensive}.

In this paper, we introduce GICDM, a hubness reduction method tailored for distance-based generative model evaluation. Our contributions are:
\begin{itemize}
  \item \textbf{Demonstrating hubness}: We show that common embedding spaces for
  generative model evaluation exhibit hubness, and find that the
  \emph{Iterative} Contextual Dissimilarity Measure (ICDM)
  \cite{jegou2010contextual} is effective at reducing hubness in real datasets.
  \item \textbf{Hubness mitigation for generative model evaluation}:
  We show that hubness has a major impact on distance-based fidelity and coverage metrics.
  We adapt the ICDM method for evaluating generative models and enhance it with a careful filtering strategy.
  The result is GICDM, a hubness reduction method for pairs of real and generated sets.
\end{itemize}

This paper is organized as follows. \Cref{sec:hubness_background} reviews the hubness phenomenon; \Cref{sec:prior_work} surveys hubness reduction methods; \Cref{sec:related_work_metrics} discusses related evaluation metrics; \Cref{sec:method} introduces GICDM; \Cref{sec:experiments} presents experimental results; and \Cref{sec:conclusion} concludes.

\section{Background: Hubness in High-Dimensional Distributions}
\label{sec:hubness_background}

The \emph{$k$-occurrence} of a point $x \in D$, denoted $O_k(x)$, is
the number of times $x$ appears among the $k$-nearest-neighbors of
other points in the dataset $D$: $O_k(x) = \# \{y \in D \setminus \{x\} \mid x \in \mathcal{N}_k(y)\}$, 
where $\mathcal{N}_k(y)$ is the set of $k$-nearest-neighbors of $y$ in $D$.
As dimensionality increases, the distribution
of $k$-occurrences becomes increasingly skewed to the right, with the emergence of
\emph{hubs}: points with unusually high $O_k$ values.
Conversely, this skew leaves a massive number of points with $O_k=0$, known as \emph{antihubs}
(see \Cref{fig:gaussian_occurrences}). This phenomenon is called \emph{hubness}
\cite{radovanovic2010hubs}.
Unlike in low dimensions, where nearest-neighbor relationships are often
reciprocal, hubness makes these relationships highly asymmetric, rendering the
notion of "neighborhood" unreliable for many data points \cite{feldbauer2018fast, feldbauer2019comprehensive}.

\begin{figure}[b]
    \begin{center}
        \subfloat[$d=4$]{\includegraphics[width=0.23\textwidth]{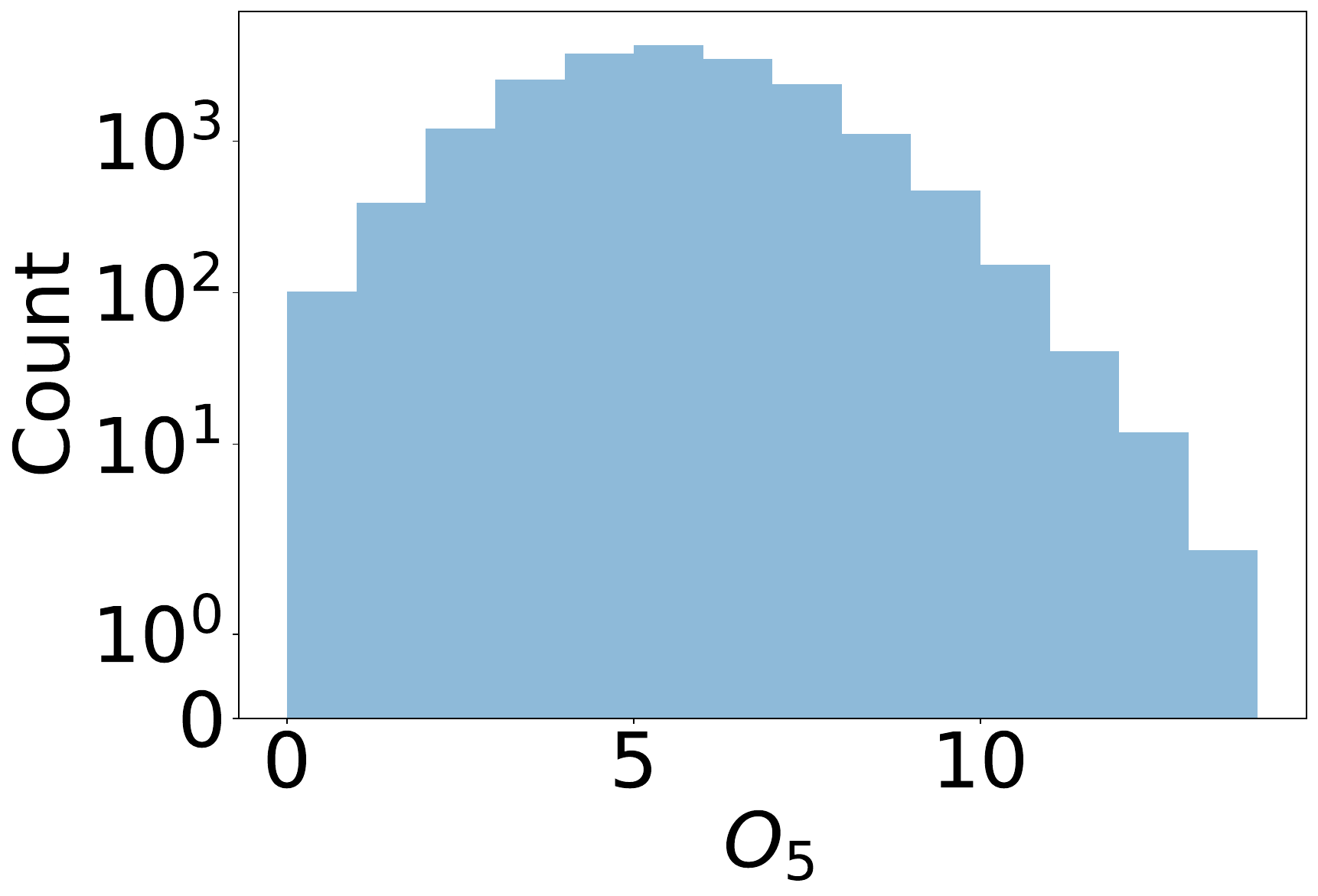}\label{fig:low_dim_occurrences}}
        \hfill
        \subfloat[$d=32$]{\includegraphics[width=0.23\textwidth]{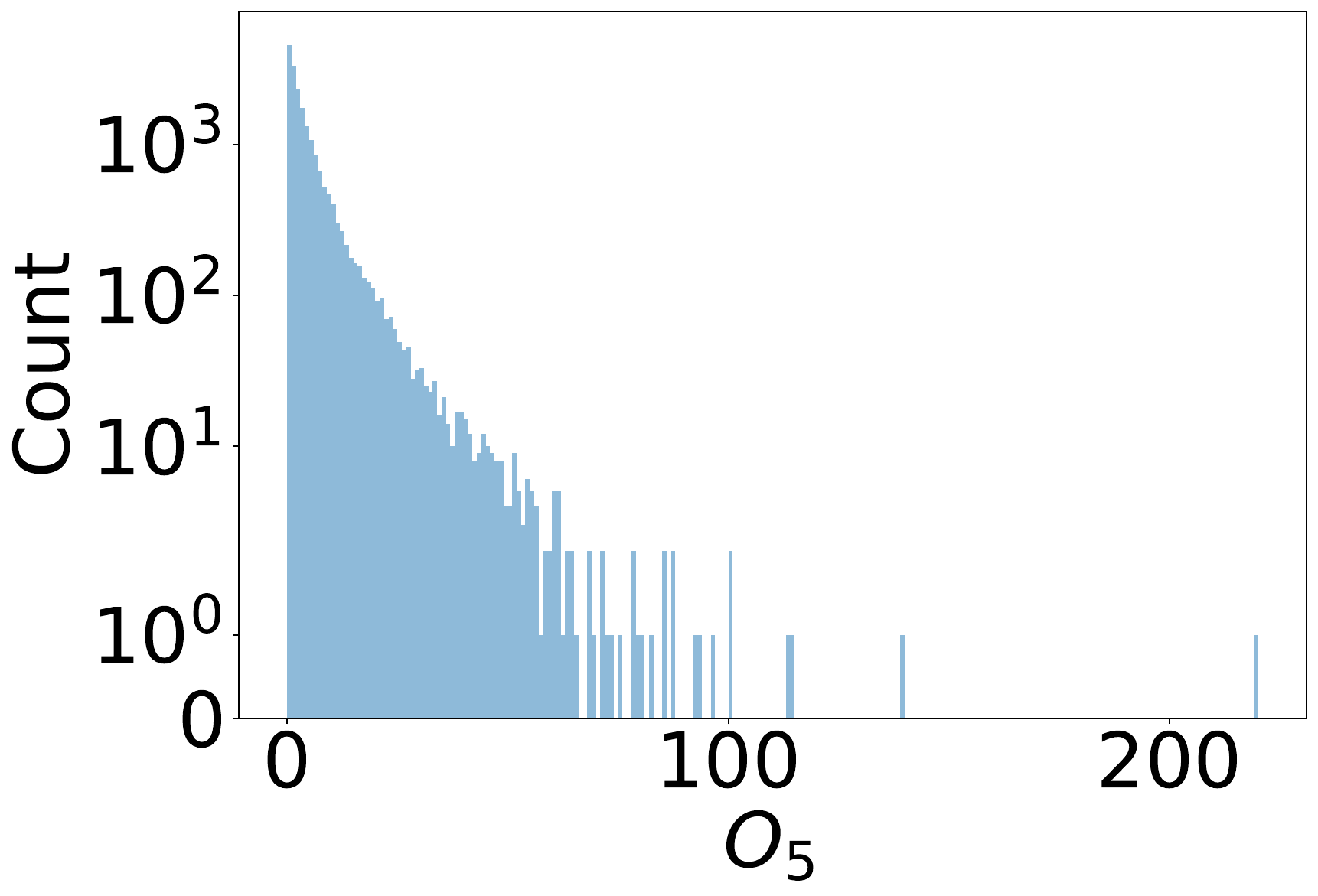}}
        \caption{
            \textbf{Distribution of $k$-occurrences} ($k=5$) for $20000$ samples
            from a standard Gaussian in \textit{(a)} $d=4$ and \textit{(b)} $d=32$.
            The y-axis shows the count (log scale) for each $k$-occurrence value (x-axis).
            As dimension increases, the distribution skews right, leading to the emergence of \emph{hubs} (frequent nearest-neighbors, high $O_k$) and \emph{antihubs} (never neighbors, $O_k=0$).
            }
        \label{fig:gaussian_occurrences}
    \end{center}
\end{figure}

\subsection{Gaussian case}
\label{sec:gaussian_case}

This phenomenon can be intuitively understood in the Gaussian case. Let $X, Y \sim \mathcal{N}(0, I_d)$ be independent standard Gaussian random vectors in $\mathbb{R}^d$.

\textbf{Spherical shell.} The squared norm $\|X\|^2$ follows a chi-squared distribution
with $d$ degrees of freedom, $\|X\|^2 \sim \chi^2(d)$, with mean $d$ and
standard deviation $\sqrt{2d}$. Therefore, in high dimensions, standard Gaussian samples concentrate on a thin spherical shell of radius $\sqrt{d}$ and thickness $O(d^{1/4})$ \cite{vershynin2018high}. See \Cref{fig:hubness_evolution} (top) for an illustration.

\textbf{Orthogonality.}
Two independent high-dimensional Gaussian vectors tend to be nearly orthogonal
\cite{vershynin2018high}.
The inner product $X \cdot Y = \sum_{i=1}^d X_i Y_i$
has a mean of $\sum_{i=1}^d \mathbb{E}[X_i]\mathbb{E}[Y_i] = 0$
and a variance of $\sum_{i=1}^d \Var(X_i Y_i) = d$.
By the law of large numbers,
$\frac{X \cdot Y}{d} \xrightarrow[]{\mathbb P} 0$ as $d \to \infty$.
Similarly,
$\frac{\|X\|^2}{d} \xrightarrow[]{\mathbb P} 1$ and
$\frac{\|Y\|^2}{d} \xrightarrow[]{\mathbb P} 1$. Thus, their cosine
similarity converges to zero:
\begin{equation*}
  \frac{X \cdot Y}{\|X\|\|Y\|} = \frac{X \cdot Y/d}{\sqrt{(\|X\|^2/d)(\|Y\|^2/d)}} \xrightarrow[]{\mathbb P} 0.
\end{equation*}
Hence, as the dimension $d$ increases, any pair of independent standard Gaussian
vectors becomes almost orthogonal.

The hubness phenomenon can then be interpreted via the Pythagorean theorem:
the squared distance between two independent samples $x$ and $y$ is approximately the sum of their squared distances to the origin,
$\|x - y\|^2 \approx \|x\|^2 + \|y\|^2$. 
As a result, points slightly closer to the origin than others tend to appear as nearest-neighbors for many points,
becoming hubs due to their relative centrality, while points far from the origin become antihubs (see \Cref{fig:occ_vs_dist_gaussian}). Importantly, removing those central points just causes the next most central points to become hubs. Hubness is not caused by a few outliers, but an inherent property of the distribution \cite{radovanovic2010hubs}.

\begin{figure}[t]
  \begin{center}
    \subfloat[Gaussian]{\includegraphics[width=0.30\textwidth]{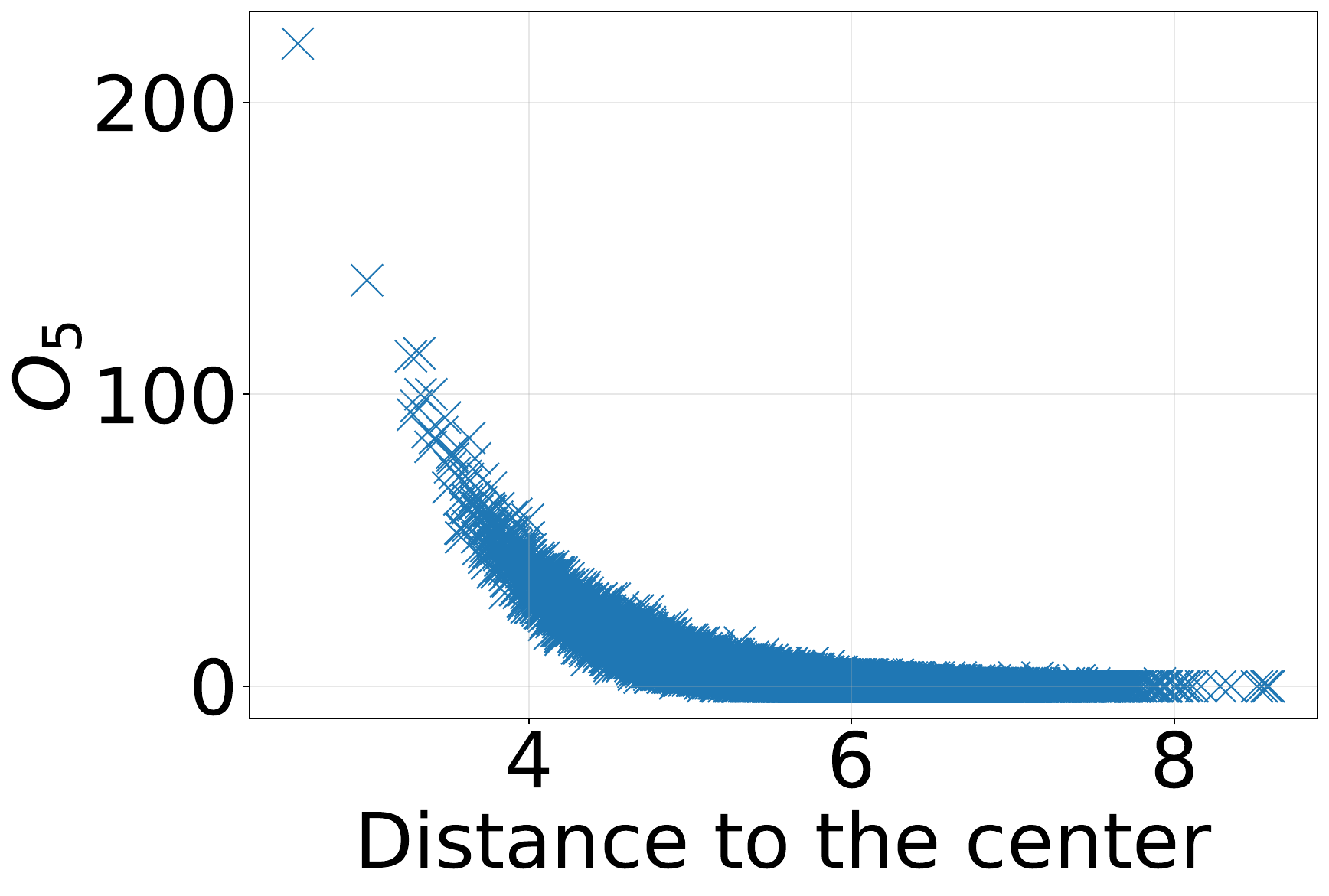}\label{fig:occ_vs_dist_gaussian}}
    \hfill
    \subfloat[Sphere]{\includegraphics[width=0.174\textwidth]{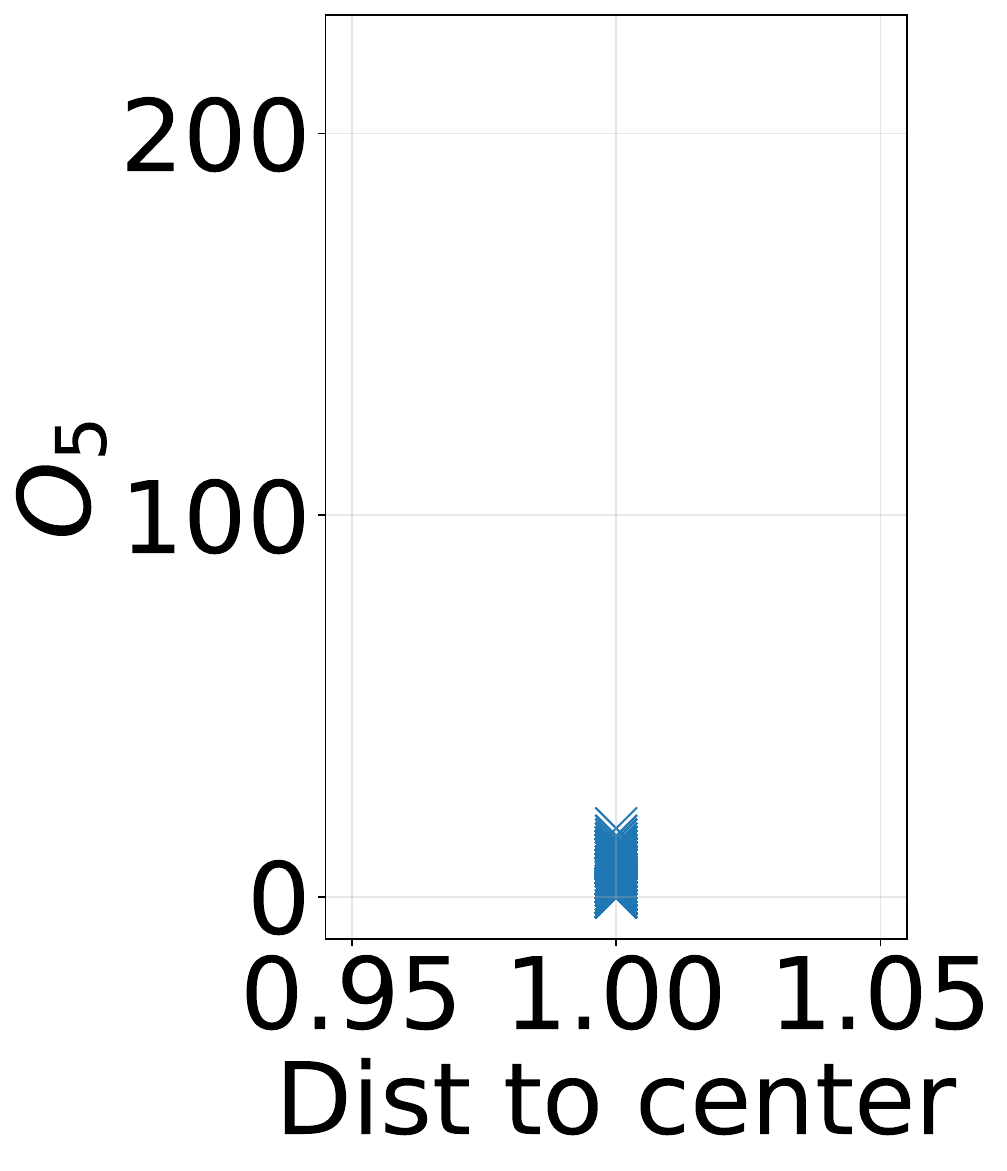}\label{fig:occ_vs_dist_sphere}}
    \caption{
      \textbf{$5$-Occurrence vs Centrality} ($N=20000$, $d=32$): 
      \textit{(a)} (adapted from \citet{radovanovic2010hubs}, Fig. 3) Standard Gaussian:
      hubness aligns with centrality; central points have higher occurrences.
      \textit{(b)} Uniform sphere: removing centrality eliminates hubness.
      So, high dimensionality alone does not cause hubness.
        }
    \label{fig:occ_vs_dist}
  \end{center}
\end{figure}

\subsection{Origin of Hubness: Density Gradient}

\begin{figure}[t]
  \begin{center}
    \begin{minipage}[c]{0.8\columnwidth}
      \subfloat[Gaussian illustration]{\includegraphics[width=0.48\textwidth, trim=1.3cm 2cm 1.3cm 1.3cm, clip]{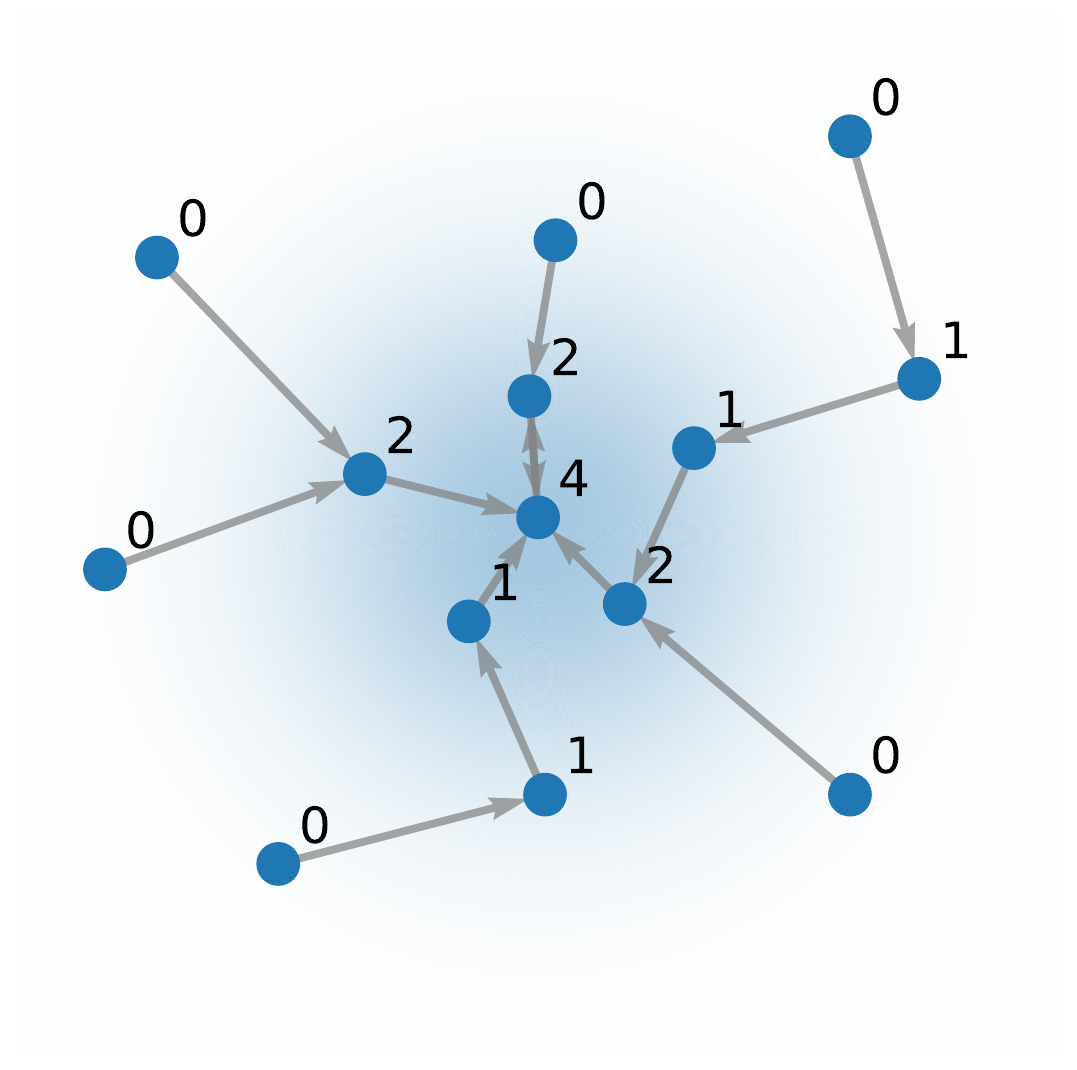}\label{fig:intuitive_gaussian}}
      \hfill
      \subfloat[Uniform Square]{\includegraphics[width=0.48\textwidth]{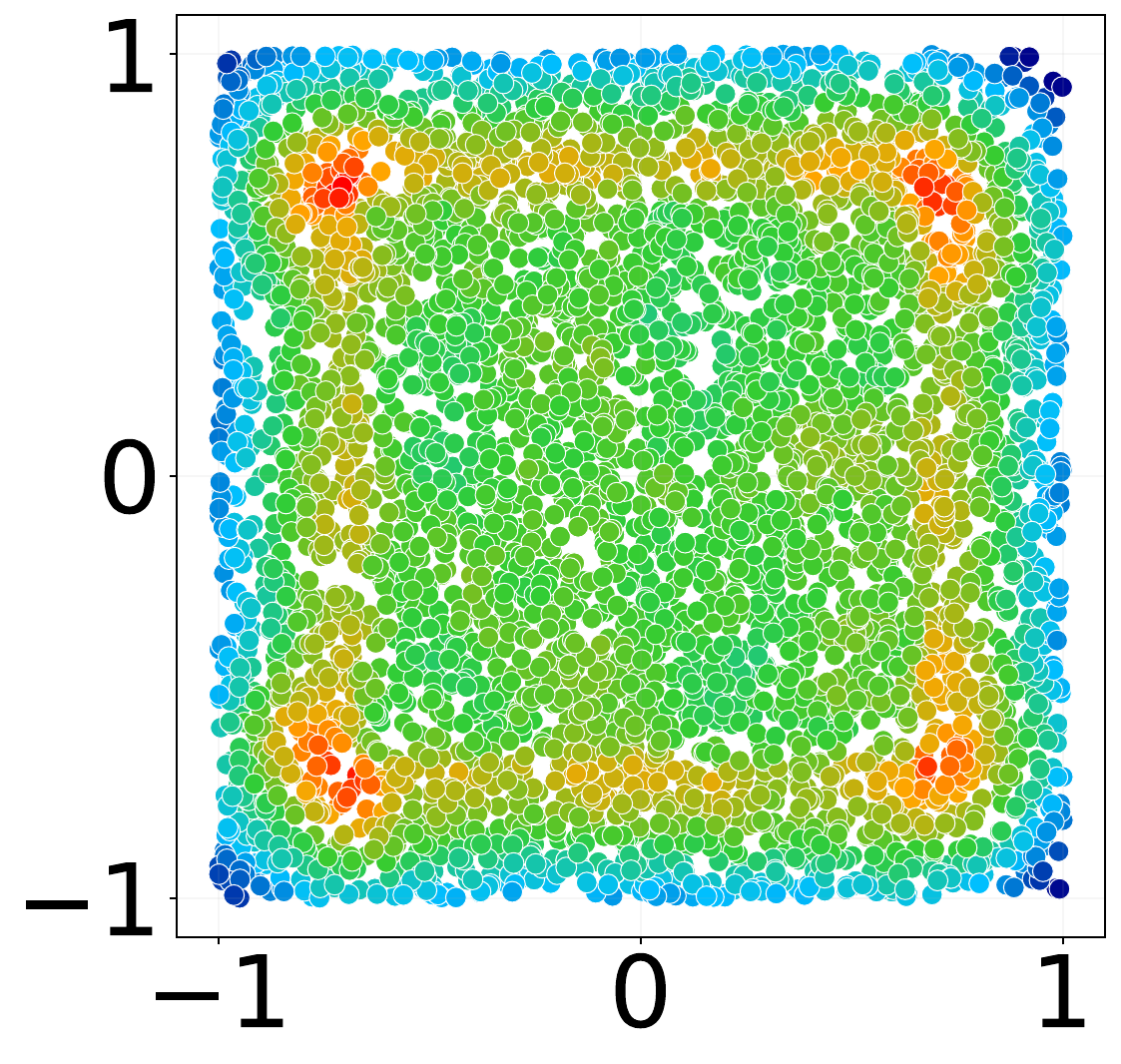}\label{fig:intuitive_square}}

        \vspace{0.3em}

      \centering
      \subfloat[Post NICDM]{
        \includegraphics[width=0.48\textwidth, trim=1.3cm 2cm 1.3cm 1.3cm, clip]{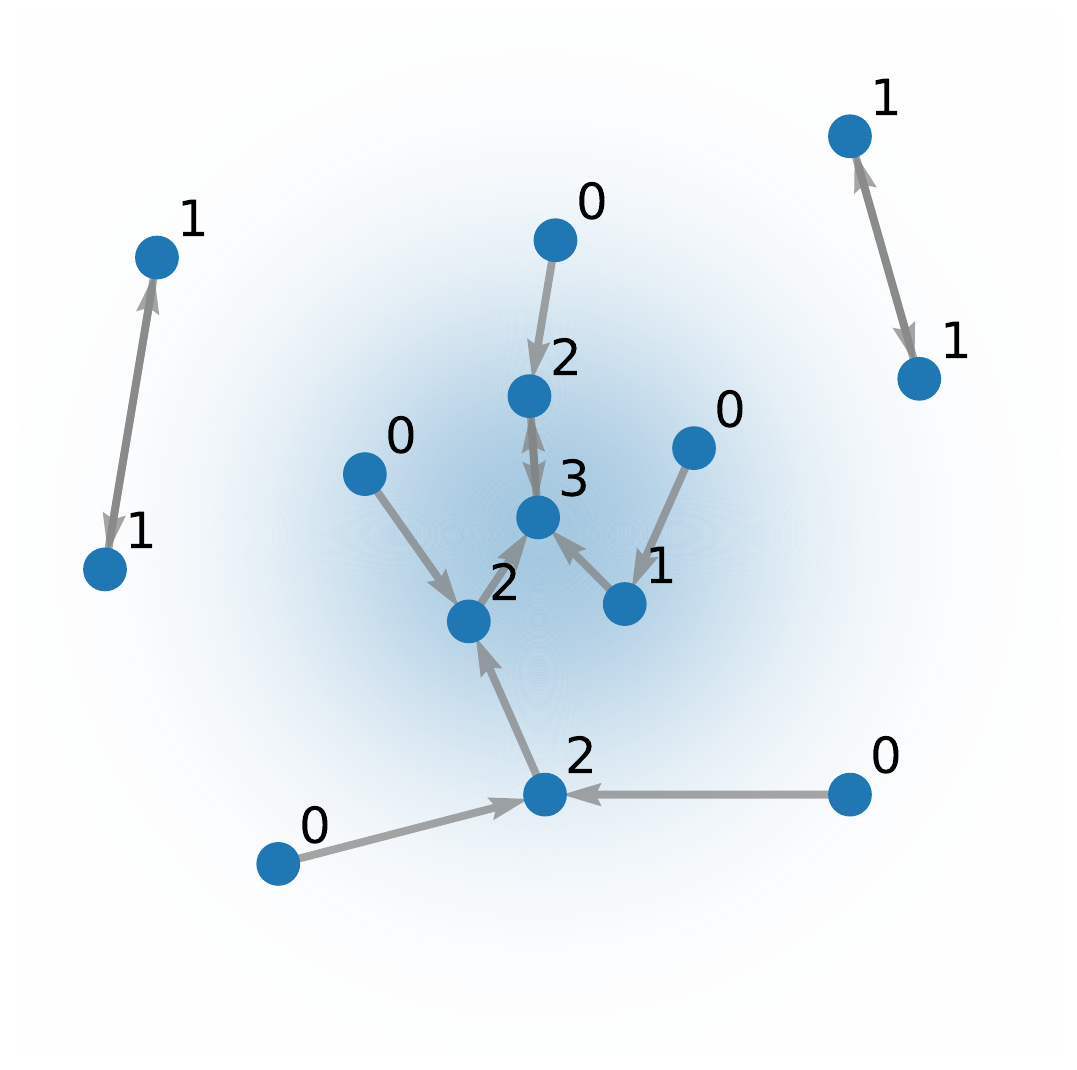}
        \hfill
        \includegraphics[width=0.48\textwidth]{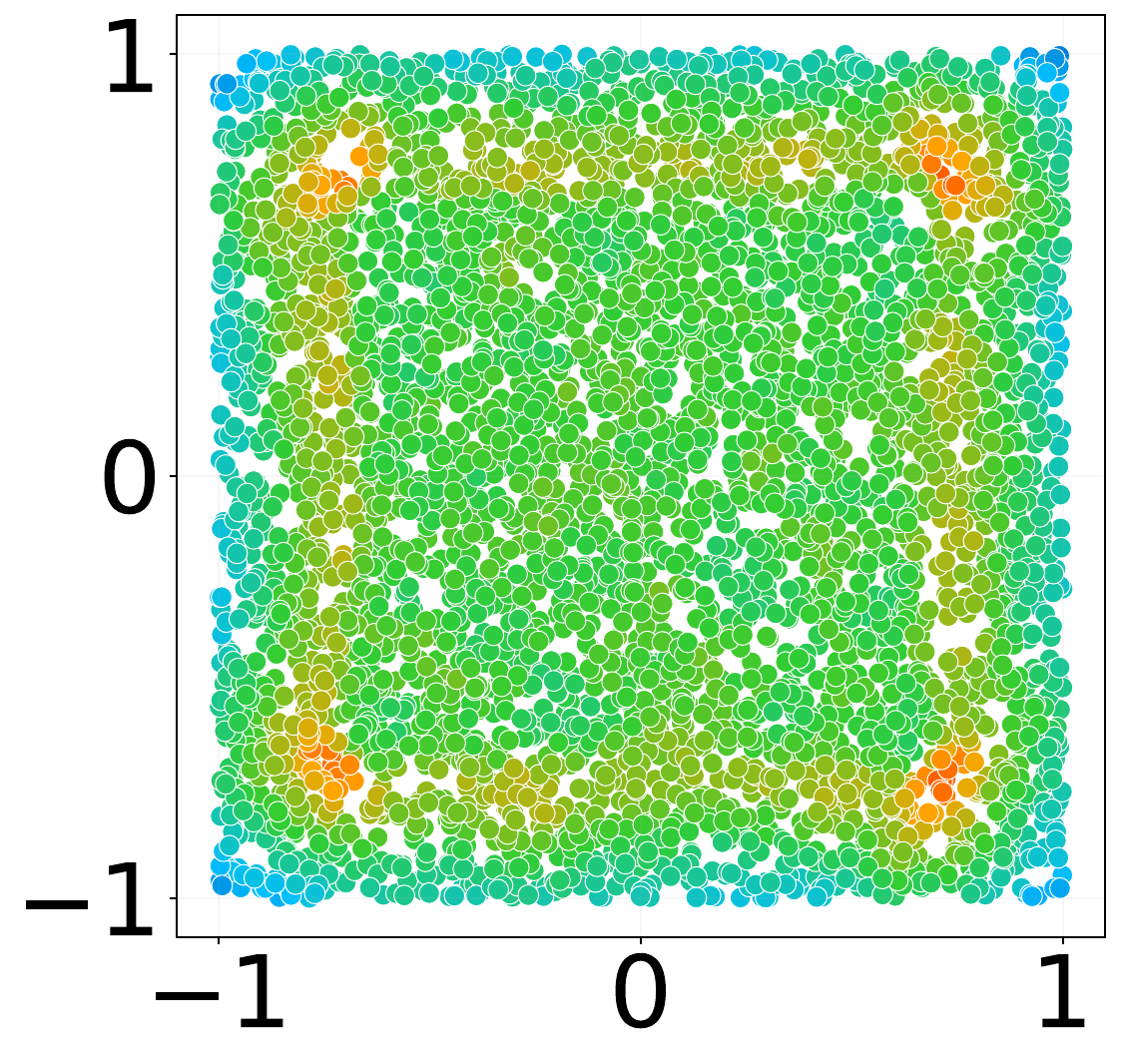}
        \label{fig:intuitive_reduction_nicdm}
      }

      \vspace{0.3em}

      \subfloat[Post ICDM]{
        \includegraphics[width=0.48\textwidth, trim=1.3cm 2cm 1.3cm 1.3cm, clip]{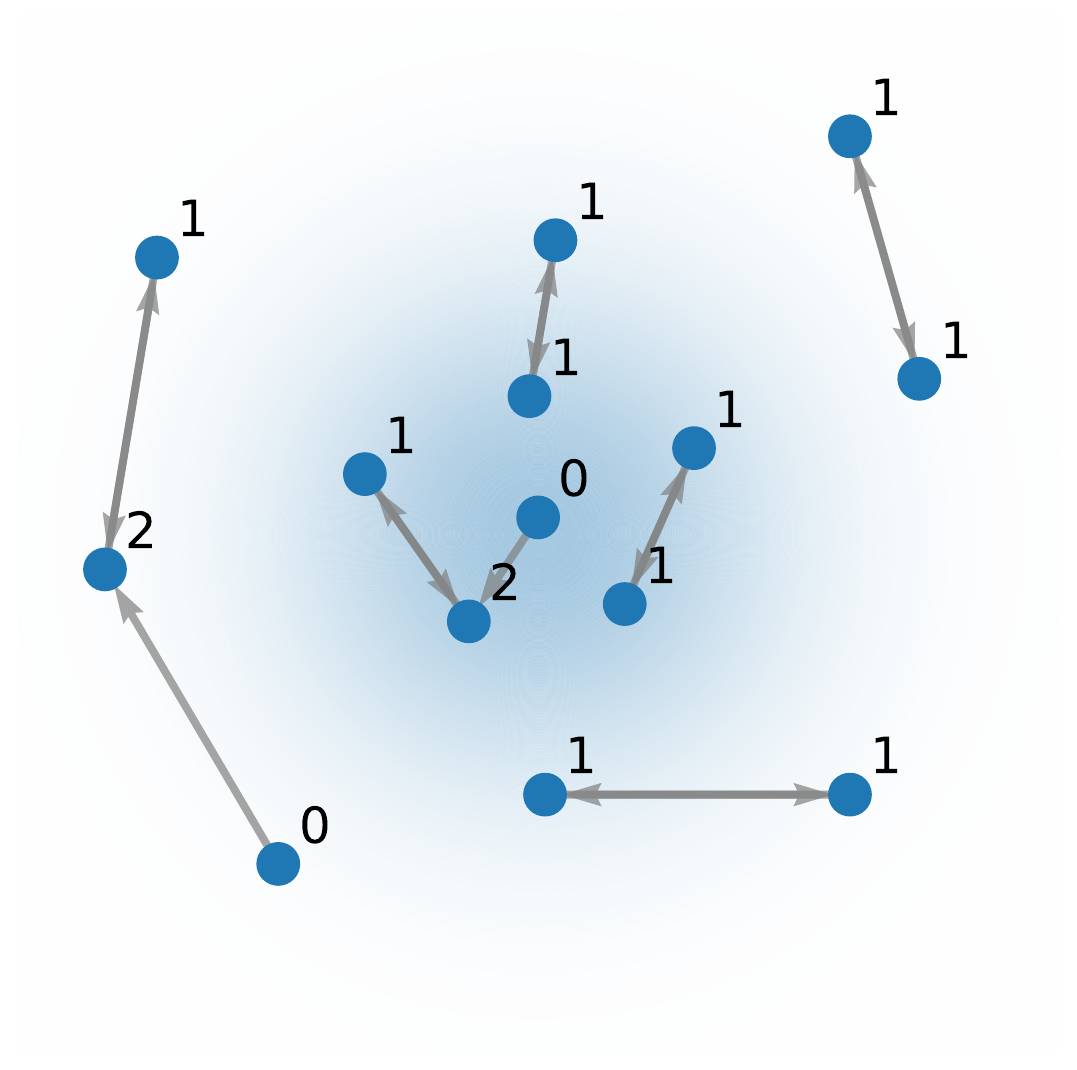}
        \hfill
        \includegraphics[width=0.48\textwidth]{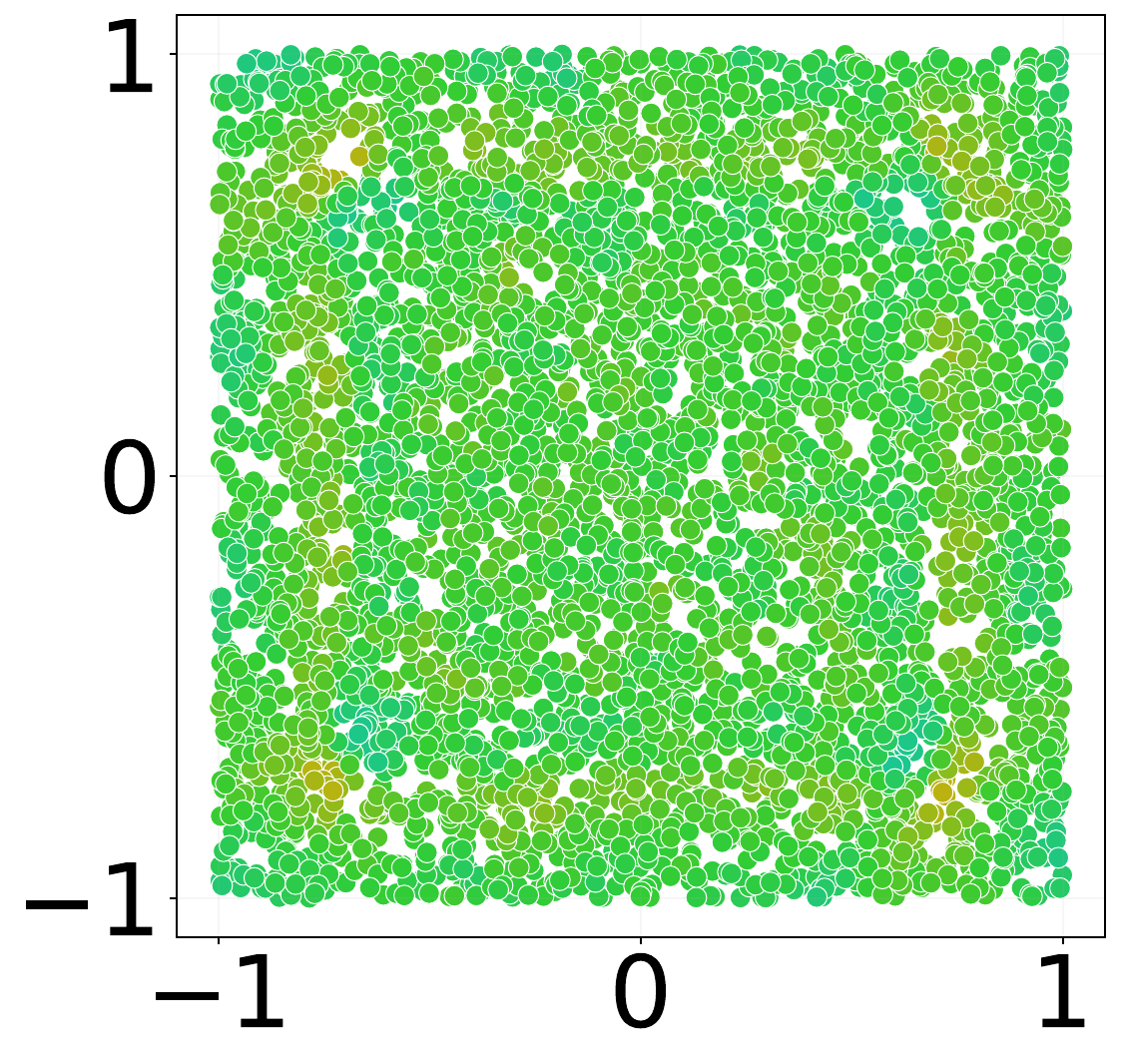}
        \label{fig:intuitive_reduction_icdm}
      }
    \end{minipage}
    \hfill
    \begin{minipage}[c]{0.17\columnwidth}
      \centering
      \includegraphics[width=\textwidth]{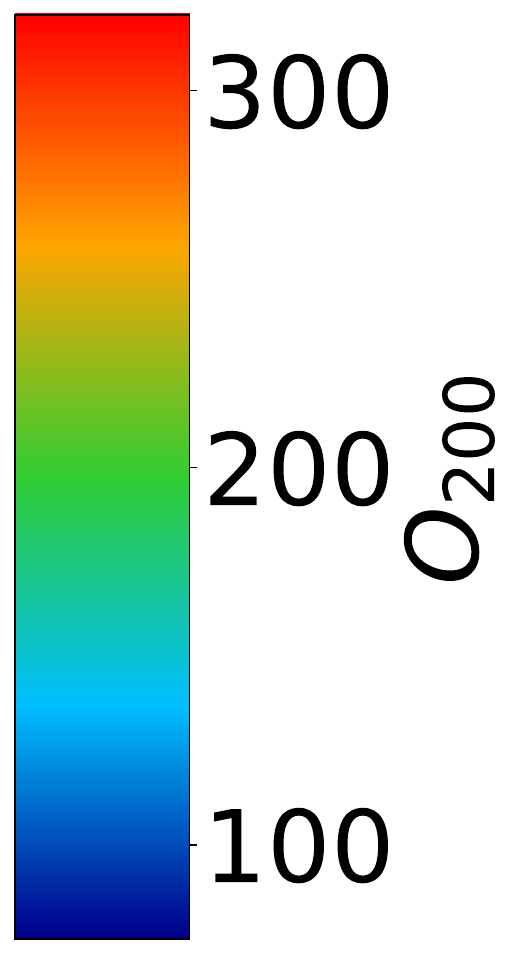}
    \end{minipage}

    \caption{\textbf{Density Gradient}: \textit{(a)} (adapted from \citet{hara2016flattening}, Fig. 2c):
    Illustrative Gaussian points (blue density) with $1$-occurrence counts; arrows indicate nearest-neighbors.
    Occurrences accumulate in denser regions, as arrows follow the gradient of the density.
    \textit{(b)}: $5000$ samples from a uniform square.
    Color indicates $200$-occurrence count. Boundaries force
    neighbors inward, creating density gradients and hubness;
    both effects intensify with increasing dimension.
    After applying \textit{(c)} NICDM and \textit{(d)} ICDM: Left, arrows show updated neighbors. The maximum $O_1$ decreases from $4$ to $3$ (NICDM) and $2$ (ICDM); the number of antihubs drops from $6$ to $5$ and $2$. Right, the occurrence distribution becomes more uniform after NICDM, and even more so after ICDM.
    }
    \label{fig:intuitive}
  \end{center}
\end{figure}

Hubness is often described as a manifestation of the curse of dimensionality
\cite{radovanovic2010hubs}, but for Gaussian data, it is closely linked to centrality. If centrality is removed by sampling uniformly on a sphere, hubness disappears entirely, even in high dimensions (see \Cref{fig:occ_vs_dist_sphere}).

\citet{low2013hubness} argue that hubness is fundamentally caused by
\emph{density gradients}. As illustrated in
\Cref{fig:intuitive_gaussian}, for a Gaussian distribution, samples tend
to have their nearest-neighbors in regions of higher density,
effectively following the gradient of the density function. Consequently,
points in denser regions accumulate more occurrences.

Boundaries also induce density gradients, since the density drops to
zero outside. \Cref{fig:intuitive_square} illustrates how
boundaries affect occurrences for a uniform distribution on a square
in $d=2$: near edges, neighbors can only be found inward.
The boundary-to-volume ratio increases
exponentially with dimension, amplifying this effect: in high dimensions,
almost all points are close to a boundary \cite{low2013hubness}.
A $d$-dimensional sphere is a finite,
boundary-less $(d-1)$-dimensional manifold. 
Sampling uniformly on it avoids hubness, even in high dimensions, 
because there is no density gradient. However, non-uniform sampling introduces 
a density gradient and thus hubness.
In a non-Gaussian setting, simply projecting points 
onto a sphere does not eliminate hubness (see \Cref{tab:hubness_measures}).

The density gradient is intrinsic to the data distribution,
so hubness cannot be eliminated simply by increasing the sample
size \cite{low2013hubness}. Furthermore, hubness
correlates with the intrinsic rather than the ambient dimension, meaning
that reducing the ambient dimension does not alleviate hubness unless information
is lost \cite{radovanovic2010hubs,flexer2015choosing}. This confirms
that hubness is fundamentally tied to the data distribution.

Importantly, hubness is not just a theoretical concern:
it impairs distance-based evaluation (see \Cref{fig:hypersphere_test})
and we observe it in practice within the very embedding spaces
used to evaluate generative models (see \Cref{sec:hubness_measures}, \Cref{tab:hubness_measures}).

\section{Overview of In-sample Hubness Reduction}
\label{sec:prior_work}

Restoring meaningful neighborhoods in high-dimensional spaces
has motivated the development of numerous hubness reduction methods.
They typically compute secondary dissimilarities between data points
to replace original distances for nearest-neighbor search.
Their performance has been compared across various scenarios in the literature
\cite{feldbauer2019comprehensive,amblard2022hubness,obraczka2022fast},
and we consider the union of their recommendations here.

Among these, Local DisSimilarity (DSL) \cite{hara2016flattening} aims to flatten the density directly, while the others are scaling methods that seek to restore symmetry in nearest-neighbor relationships by scaling distances.
By forcing symmetry, these scaling methods consistently increase nearest-neighbor classification accuracy \cite{feldbauer2019comprehensive}, indicating that they improve the semantic correctness of local neighbor relationships by reducing the influence of one-way relationships induced by hubness.
The Gaussian variant of Mutual Proximity ($\text{MP}^{\text{Gauss}}$) \cite{schnitzer2012local}
uses all pairwise distances for scaling, whereas the remaining methods are local scaling
approaches that use the
local neighborhood of each point to scale distances. These include
Local Scaling (LS) \cite{zelnik2004self}, Cross-domain Similarity Local Scaling (CSLS)
\cite{lample2018word}, and the Non-Iterative Contextual Dissimilarity Measure (NICDM)
\cite{jegou2007contextual,jegou2010contextual}, whose formulations are closely related.

Given the pairwise distances $\{d_{x_i,x_j}\}_{i,j=1}^N$ between a set
of points $\{x_i\}_{i=1}^N$, NICDM produces secondary dissimilarities as follows:
\begin{equation*}
\text{NICDM}(d_{x_i,x_j}) = d_{x_i,x_j}\frac{\bar{\mu}}{\sqrt{\mu_i \mu_j}}
\end{equation*}
where
\begin{equation*}
  \mu_i = \frac{1}{K} \sum_{k=1}^K d_{x_i, \NN_k(x_i)},
\end{equation*}
$\NN_k(x)$ is the $k$-th nearest-neighbor of $x$, and $\bar{\mu}$ is the average of all $\mu_i$.
Since $\bar{\mu}$ is a global constant, it can be ignored in practice \cite{schnitzer2012local}.
The resulting dissimilarity is low when both
$\sqrt{\frac{d_{x_i,x_j}}{\mu_i}}$ and $\sqrt{\frac{d_{x_j,x_i}}{\mu_j}}$
are small
(i.e., for reciprocal nearest-neighbors). Local scaling methods have $O(N^2)$ complexity.

Many hubness reduction methods were originally proposed in other contexts and only later adapted for hubness reduction.
The Contextual Dissimilarity Measure (CDM) was first introduced for
image retrieval to improve the symmetry of neighbor relationships
\cite{jegou2007contextual}.
Later, only the simpler Non-Iterative version (NICDM) was adapted for hubness reduction,
while the \emph{Iterative} version (ICDM) was dismissed as too computationally expensive without yielding significant improvement \cite{schnitzer2012local}.
Subsequent works followed this recommendation and used only NICDM \cite{feldbauer2019comprehensive,amblard2022hubness,obraczka2022fast}. ICDM has not been considered for hubness reduction since then.

We compared these methods for hubness reduction (see \Cref{tab:hubness_reduction}), and found that despite being overlooked, ICDM consistently achieves the best performance across all 17 tested pairs of datasets and embeddings.
\Cref{fig:intuitive_reduction_nicdm,fig:intuitive_reduction_icdm} illustrate the effects of NICDM and ICDM.
Both reduce hubness, as evidenced by the decrease in maximum $O_1$ and the number of antihubs (left),
and by the more uniform occurrence distribution (right).

The \emph{Iterative} CDM (ICDM) is obtained by repeatedly applying
the non-iterative version to progressively improve the symmetry of neighbor relationships.
Let $^{(t)}$ denote quantities at iteration $t$. The update rule is:
\begin{equation*}
    d^{(t+1)}_{x_i,x_j} =d^{(t)}_{x_i,x_j} \frac{\bar{\mu}^{(t)}}{\sqrt{\mu^{(t)}_i \mu^{(t)}_j}}
\end{equation*}
These iterations minimize the disparity among the $\mu^{(t)}_i$ values, $S^{(t)} = \sum_{i=1}^N |\mu^{(t)}_i - \bar{\mu}^{(t)}|$, converging to a fixed point where all $\mu^{(t)}_i$ are equal \cite{jegou2010contextual}.
Defining $\delta_i = \prod_{t=1}^{T} \sqrt{\frac{\bar{\mu}^{(t)}}{\mu^{(t)}_i}}$, the final result can
be expressed using the original distances: $d^{(T)}_{x_i,x_j} = d_{x_i,x_j} \delta_i \delta_j$.

In image retrieval, ICDM is first applied to a base dataset. Then, for a query point $x^q$, its nearest-neighbors in the base dataset are found using:
\begin{equation*}
  \NN_k(x^q) = \operatorname{k-argmin}_{i} d_{x^q,x_i} \delta_i
\end{equation*}
Finding the nearest base neighbors of $x^q$ does not require knowing the update term $\delta_q$ for the query point,
as it is constant across all comparisons \cite{jegou2010contextual}.
However, for out-of-sample evaluation, it is not sufficient to simply rank the neighbors of generated points.
For this, \emph{distances} to neighbors are required.
ICDM does not provide these directly, so \Cref{sec:method} presents an algorithmic solution.

To the best of our knowledge, only NICDM, not ICDM, has been considered previously for hubness reduction.

\section{Related Work: Generative Model Metrics}
\label{sec:related_work_metrics}

In this section, we review existing fidelity and coverage metrics for generative model evaluation.

Fidelity metrics usually assess whether generated samples
are realistic by measuring their proximity to real
samples. Improved Precision \cite{kynkaanniemi2019improved} computes the fraction of
generated samples within the $k$-nearest-neighbor radius of real samples: $\text{Precision} = \frac{1}{M} \sum_{j=1}^M \mathbf{1}_{x^g_j \in \cup_{i=1}^N B(x^r_i, \NND_k^r(x^r_i))}$.
Density \cite{naeem2020reliable} extends
this by counting how many real samples each generated sample is close to: $\text{Density} = \frac{1}{kM} \sum_{j=1}^M \sum_{i=1}^N \mathbf{1}_{x^g_j \in B(x^r_i, \NND_k^r(x^r_i))}$.

Conversely, coverage metrics evaluate how well generated samples cover the
real data distribution. Improved Recall \cite{kynkaanniemi2019improved}
measures the fraction of real samples within the $k$-nearest-neighbor
radius of generated samples:
$\text{Recall} = \frac{1}{N} \sum_{i=1}^N \mathbf{1}_{x^r_i \in \cup_{j=1}^M B(x^g_j, \NND_k^g(x^g_j))}$.
Coverage \cite{naeem2020reliable} instead checks if each real sample
is covered by at least one generated point within its real ball: $\text{Coverage} = \frac{1}{N} \sum_{i=1}^N \mathbf{1}_{\exists j, x^g_j \in B(x^r_i, \NND^r_k(x^r_i))}$.

Other metrics use similar constructions. For example, Precision Cover and Recall Cover \cite{cheema2023precision} count balls as covered
if at least $k' > 1$ points lie within them.
P-precision and P-recall
\cite{park2023probabilistic} use kernel density estimates instead of
hard balls.
Rather than using the union of balls as an approximate support, $\alpha$-Precision and $\beta$-Recall \cite{alaa2022faithful} use a one-class approach to estimate supports at varying levels, while Topological Precision and Recall \cite{kim2023topp} use topologically conditioned density kernels.
More closely related, Clipped Density and Clipped Coverage \cite{salvy2026enhanced}
are variants of Density and Coverage that limit the influence of individual
points to aggregate scores. In Clipped Density, for example, real ball radii are clipped to
reduce outlier effects, and the result is normalized by the real set score.

Hubness directly distorts these distance-based metrics: generated
hubs artificially cover many real points, inflating coverage scores,
while antihub regions act as blind
spots where generated samples receive zero fidelity.
This distortion severely biases the evaluation, yielding misleading
conclusions about generative model performance.

The work most closely related to ours identified a flaw in distance-based metrics
and noted that hubness is "very closely related" to their analysis,
though it is not explicitly targeted \cite{khayatkhoei2023emergent}.
Their focus is on the support constructed by Precision and Recall.
For a uniform hypersphere in high dimensions, the support
induced by Precision is biased toward the center, leading to different
outcomes for points at
the same distance, depending on whether they are inwards or outwards from
the sphere. To restore symmetry, they
propose using two alternative supports: one based on real points and one on generated
points, so that when one is biased, the other can be used. symPrecision and
symRecall are then defined as the minimum of Precision or Recall and its complement.
However, for non-unimodal distributions, both supports can be biased
simultaneously. As shown in \Cref{fig:hypersphere_test}, sym metrics fail for a bimodal
hypersphere. In contrast, GICDM addresses hubness more generally and is not tied to a specific metric.

\section{GICDM}
\label{sec:method}

Our goal is to enable reliable distance-based evaluation of
generative models in high-dimensional embedding spaces, where
hubness distorts nearest-neighbor relationships.
To address this, we identify three key desiderata for
effective hubness mitigation:

1. \textbf{Reduce hubness in the real dataset:}
In the transformed distance space, the local density of real data
should be uniform, so that nearest-neighbor relationships
are meaningful and symmetric.

2. \textbf{Preserve the relative positioning of generated points:}
Hubness mitigation should maintain the true local geometric structure between generated points and the real data manifold, removing only hubness effects.

3. \textbf{Conditional independence of generated points given the real set:} Distance-based metrics give sample-wise scores (e.g., a fidelity score to each individual generated sample).
So, the distance scaling of each generated point should
depend only on its relationship to the real set in the original space,
so that a given generated sample falls into the same real balls, receiving the same fidelity score, regardless of the other generated samples evaluated alongside it.

\subsection{Density Estimation from Nearest-Neighbors Distances}

ICDM iteratively applies a popularity penalty: it expands the space around points in relatively
high-density regions (those with small $\mu_i$, which receive large $\delta_i$), making them harder to reach,
and contracts the space around points in relatively low-density regions (i.e. large $\mu_i$, small $\delta_i$), making
them easier to reach. This process repeats until the average neighbor distance is nearly equal
for all points.

\newcommand{\densityEstimatorProp}{
  Let \begin{equation*}
    \hat{p}_{\mu, K}(x_i) \overset{\mathrm{def}}{=} \frac{1}{N V_d} \frac{1}{\mu_i^d} \left(\frac{1}{K} \sum_{k=1}^K k^{1/d}\right)^d
  \end{equation*}
  where $V_d$ is the volume of the unit ball in dimension $d$.
  Then, $\hat{p}_{\mu, K}(x_i)$ is a local density estimator.}
\begin{proposition}
  \label{prop:density_estimator}
  \densityEstimatorProp
\end{proposition}
The proof is provided in \Cref{sec:proof_density_estimator}. $\hat{p}_{\mu, K}$ is closely related to the $K$-NN density estimator \cite{loftsgaarden1965nonparametric}, but uses the average $K$-NN distance
$\mu_i$ instead of the $K$-th nearest-neighbor distance.
\begin{corollary}
  At ICDM convergence, when $\forall i, |\mu_i - \bar{\mu}| < \epsilon$ for an arbitrarily small $\epsilon > 0$, the local density estimates $\hat{p}_{\mu, K}(x_i)$ are equal for all $x_i$.
  \label{cor:uniform_density}
\end{corollary}
Thus, ICDM effectively uniformizes the data density, removing density gradients and thereby reducing hubness.

Applying ICDM directly to the union of real and generated points violates all
of our desiderata: the uniformized density would be that of the combined set, not
just the real data (violating desideratum 1). Secondary dissimilarities would
be influenced by generated points, making each generated point’s evaluation
depend on the others (violating desideratum 3), and dissimilarities between
real points would be affected by generated data (violating all desiderata).

\subsection{Extending ICDM to Generated Points}

Let $\{x^r_i\}_{i=1}^{N}$ denote real data
points and $\{x^g_j\}_{j=1}^{M}$ denote generated data points. Superscripts
$^r$ and $^g$ indicate quantities associated with real
and generated points, respectively.

To uniformize the density of the real points only (desideratum 1),
we first apply ICDM to the real set, yielding scaling factors
$\{\delta^r_i\}_{i=1}^{N}$ such that secondary dissimilarities
between real points are $\text{ICDM}(x^r_i,x^r_j) = d_{x^r_i,x^r_j} \delta^r_i \delta^r_j$.

We now define secondary dissimilarities between generated and real points as $d'_{x^r_i,x^g_j} = d_{x^r_i,x^g_j} \delta^r_i \delta^g_j$, where $\delta^g_j$ is to be determined for each generated point $x^g_j$. To satisfy desiderata 2 and 3, $\delta^g_j$ must be computed independently for each generated point, using only its relationship to the real set.

For real points, considering $K$ neighbors, the average neighbor distance
is $\mu^r(x_i) = \frac{1}{K} \sum_{k=1}^K d_{x_i, \NN_k(x_i)} = \frac{1}{K} \sum_{m=1}^{K+1} D^{r,i}_m$,
where $D^{r,i}_m$ is the $m$-th smallest distance from $x_i$ to the real set,
with $D^{r,i}_1 = 0$ (distance to itself). Thus, for generated points,
the appropriate neighborhood depth in terms of order statistics is $K+1$, not $K$.

Let $^{eq}$ denote quantities at equilibrium after applying ICDM to the real
set. For all $i$, $|\mu^{r,eq}_i - \overline{\mu^{r,eq}}| < \epsilon$ for an arbitrarily small $\epsilon > 0$, i.e., the average
neighbor distance is practically equalized across real points.

Assuming that a generated sample $x^g_j$ is drawn from the same distribution as the
real samples, we require it to exhibit the same spatial properties as the real data. In
particular, its average neighbor distance must match that of the real
samples: $|\mu^{g,eq}_j - \overline{\mu^{r,eq}}| < \epsilon$. By enforcing exact equality at the theoretical equilibrium $\mu^{g,eq}_j = \overline{\mu^{r,eq}}$, we can write for all $i, j$,
\begin{align*}
  \mu^{g,eq}_j&=\frac{1}{K+1}\sum_{k=1}^{K+1} \delta^g_j \delta^r_{\NN^{r,eq}_k(x^g_j)} d_{x^g_j,\NN^{r,eq}_k(x^g_j)}\\
  &=\delta^g_j \frac{1}{K+1}\sum_{k=1}^{K+1} \delta^r_{\NN^{r,eq}_k(x^g_j)} d_{x^g_j,\NN^{r,eq}_k(x^g_j)},
\end{align*}
where $\NN^{r,eq}_k$ denotes the $k$-th nearest real neighbor
after ICDM (known using $\delta^r_i$). This yields:
\begin{equation}
  \delta^g_j=\frac{\overline{\mu^{r,eq}}}{\frac{1}{K+1}\sum_{k=1}^{K+1} \delta^r_{\NN^{r,eq}_k(x^g_j)} d_{x^g_j,\NN^{r,eq}_k(x^g_j)}}
  \label{eq:delta_g}
\end{equation}
Thus, we define the unfiltered secondary dissimilarity as
$\text{GICDM}_\text{unfiltered}(x^r_i, x^g_j) = d_{x^r_i, x^g_j} \delta^r_i \delta^g_j$,
with $\delta^g_j$ as above.

This assumes that generated points are drawn from the same distribution
as the real data. In practice, however, we do not know a priori whether a
generated point comes from the real data distribution, this is precisely
what we aim to assess.

\subsection{Preventing Overcorrections}
If a generated point is not well aligned with the real data manifold and we are outside the crossover regime (discussed in the next subsection), its distances to its real neighbors,
$d_{x^g_j,\NN^{r,eq}_k(x^g_j)}$, will not be consistent
with those observed among real points.
Consequently, the scaling factor $\delta^g_j$ computed
for this generated point will differ substantially
from the $\delta^r_i$ values of its real neighbors.

The $\delta$ values quantify the scaling applied to each point's neighborhood.
A significant discrepancy between $\delta^g_j$ and those
of its real neighbors indicates that
the generated point does not lie within the same region as the real data.

To address this, we compare $\delta^g_j$ to the average $\delta$ of its real neighbors,
$\bar{\delta}^{r|g}_j = \frac{1}{K + 1} \sum_{k=1}^{K+1} \delta^r_{\NN^{r,eq}_k(x^g_j)}$,
and filter out generated points for which $\left|\bar{\delta}^{r|g}_j - \delta^g_j\right|/\bar{\delta}^{r|g}_j$
exceeds the $q$-quantile of $\{|\bar{\delta}^{r|r}_i - \delta^r_i|/\bar{\delta}^{r|r}_i\}_{i=1}^{N}$ computed over real points,
where $\bar{\delta}^{r|r}_i = \frac{1}{K} \sum_{k=1}^{K} \delta^r_{\NN^{r,eq}_k(x^r_i)}$ and $q$ is set to $0.95$ in our experiments.
These filtered points fall outside the real data manifold. As they are not in any real ball, their fidelity score is zero and they do not contribute to coverage.

\subsection{Multi-Scale Filtering}

\begin{figure}[t]
    \begin{center}
        \includegraphics[width=0.7\columnwidth]{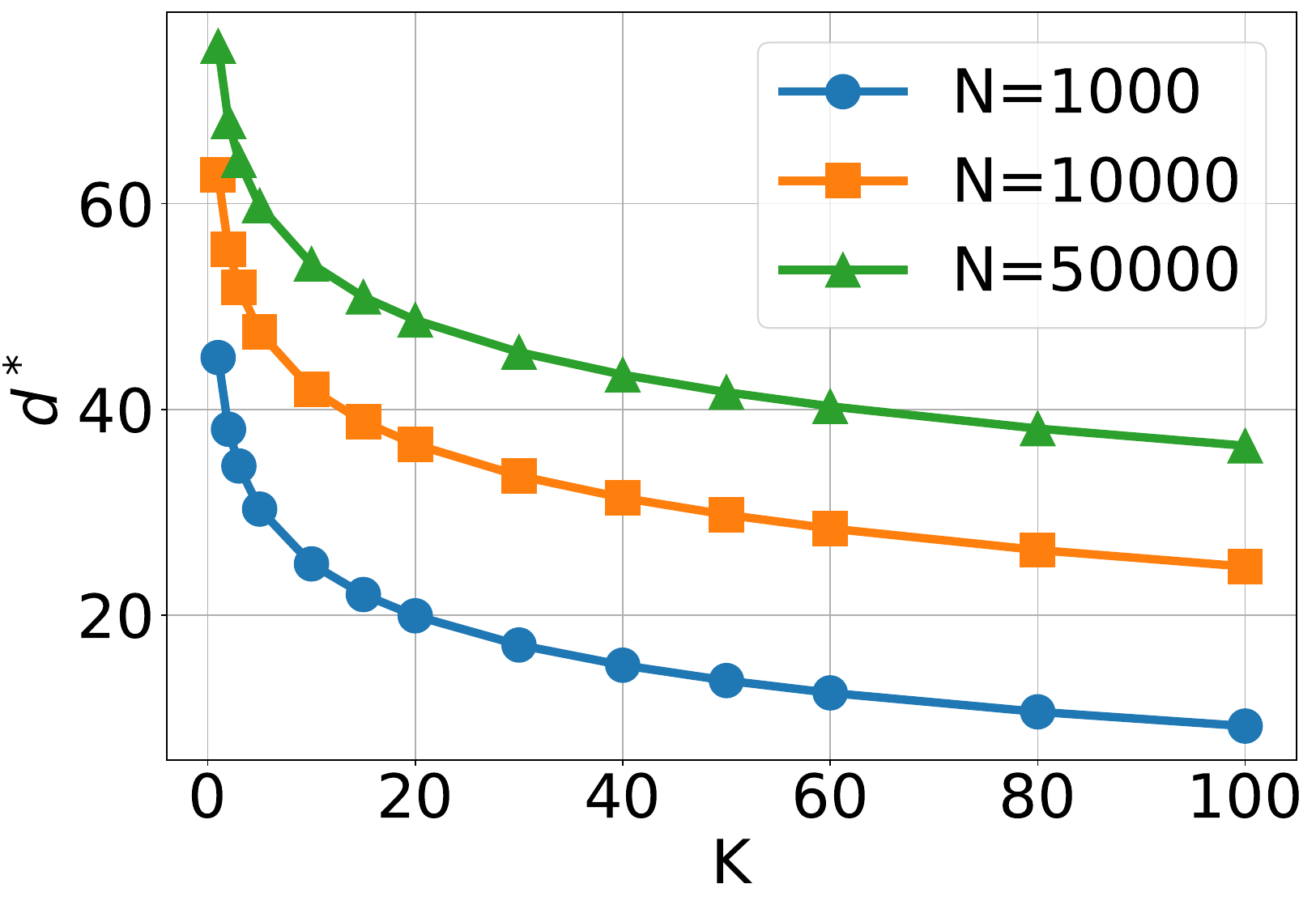}
         \caption{\textbf{Crossover dimension $d^*$}: At $d^*$ for a standard Gaussian,
         the median $K$-NN squared distance and the expected squared distance
         to the center equalize. Points are closer to each other below
         $d^*$ and closer to the center above.
         The plot shows $d^*$ vs $K$ using \Cref{proposition:crossover_dimension},
         for various sample sizes $N$.
         As $N$ increases, points become closer so $d^*$ increases; as $K$ increases, $K$-NN distances increase so $d^*$ decreases.
         Using two sufficiently different $K$ values ensures at least one avoids the crossover regime.}
        \label{fig:crossover_dimension}
    \end{center}
\end{figure}

A generated point not well aligned with the real data manifold can have similar distances to its real neighbors as those observed among real points, thus erroneously avoiding the filtering, if the local density structure falls into the \emph{crossover regime}, which depends on the density, $N$, $d$ and $K$.

To illustrate this, consider samples from a standard Gaussian.
In low dimensions, points are closer to each other than to the center,
while the reverse occurs in high dimensions.
In between is a transition regime where points are as far from
each other as from the center.
In this intermediate regime, a generated point
can be "far" from the real data (e.g., in the empty center),
yet still have a $\delta$ value similar to its real neighbors, because its distances to them
are similar to their distances to their own neighbors.
\newcommand{\crossoverDimensionProp}{
  Let $X_1,\dots,X_N$ be i.i.d. samples from $\mathcal N(0,I_d)$.
  The \emph{crossover dimension} $d^*$ such that the median squared $k$-nearest-neighbor distance equals the expected squared distance to the center, $\operatorname{median}(\NND_k(X_i)^2)=\mathbb E\|X_i\|^2$, is the solution of
  \begin{equation*}
    \int_0^\infty \mathbb{P}\left(\mathrm{Bin}(N-1,F_{\chi^2_{d^*}(\lambda=r)}({d^*}))\geq k\right) f_{\chi^2_{d^*}}(r) dr = \frac{1}{2}
  \end{equation*}
  where $f_{\chi^2_{d^*}}(r)$ is the density of the $\chi^2_{d^*}$ distribution and $F_{\chi^2_{d^*}(\lambda)}$ is the cumulative distribution function of the noncentral $\chi^2_{d^*}$ distribution with noncentrality parameter $\lambda$.
}
\begin{proposition}
  \label{proposition:crossover_dimension}
  \crossoverDimensionProp
\end{proposition}
The proof is in \Cref{sec:proof_crossover_dimension}.
Using \Cref{proposition:crossover_dimension}, we numerically estimate $d^*$.
\Cref{fig:crossover_dimension} plots $d^*$ versus $K$ for various $N$.
As expected,
increasing $N$ brings points closer, raising $d^*$, while increasing $K$
increases $K$-NN distances, lowering $d^*$.
Crucially, $d^*$ depends on both $N$ and $K$.

While this analysis relies on the Gaussian case to quantify the crossover
regime, the underlying mechanism applies generally. For any data
distribution, the $K$-th nearest-neighbor distance monotonically increases with $K$.
Consequently, changing $K$ inherently shifts the scale of local neighborhoods
and, by extension, shifts the targeted crossover regime where
out-of-manifold distances coincidentally match local in-manifold distances.
By choosing two sufficiently different $K$ values for filtering
(e.g., $K_2 = 10 K_1$), we can ensure that at least one
is outside the problematic crossover regime.

\begin{figure}[t]
    \begin{center}
        \includegraphics[width=1\columnwidth]{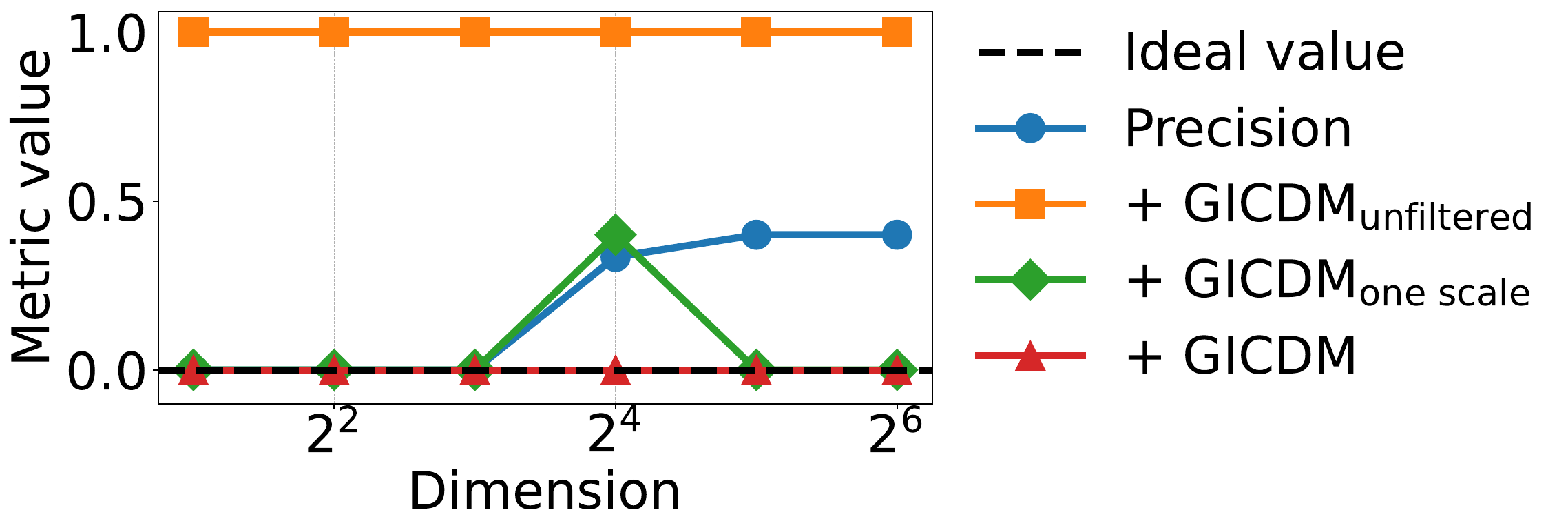}
        \caption{\textbf{Hypersphere Test ablation}: In the hypersphere scenario
            from \Cref{fig:hypersphere_test}, GICDM 
            components are incrementally added to Precision.
            Unfiltered GICDM overcorrects generated points.
            Single-scale filtering reduces this effect, but a spike remains when
            $5$-NN distances between real points and to generated points are
            similar. Finally, multi-scale filtering achieves the
            ideal result.}
        \label{fig:hypersphere_test_ablation}
    \end{center}
\end{figure}

In summary, our proposed Generative Iterative Contextual Dissimilarity
Measure (GICDM) combines the three discussed components to address our
desiderata: (1) applying ICDM to the real set to uniformize its density (desideratum 1),
(2) computing scaling factors for generated points based solely on their relationship to the real set (desiderata 2 and 3),
and (3) filtering generated points at multiple scales to prevent overcorrection (desideratum 2).
The complete algorithm is summarized in \Cref{alg:gicdm}.

\Cref{fig:hypersphere_test_ablation} demonstrates the effect
of each GICDM component in the hypersphere test from \Cref{fig:hypersphere_test},
tracking the Precision metric as components of GICDM are introduced (out-of-sample ICDM: $\text{GICDM}_\text{unfiltered}$, filtering and multi-scale filtering). 
The results indicate that all three components are required to fully resolve hubness-related issues.

GICDM is a dissimilarity rather than a distance metric
as it does not obey the triangle inequality. However,
this is not an issue for distance-based generative model
evaluation, which only requires a measure of similarity capable of ranking neighbors and
determining neighborhood membership, two operations that do not rely on the triangle inequality. This aligns with the common use of other non-metric dissimilarities,
such as cosine similarity, in feature spaces
\cite{oquab2023dinov2,DBLP:journals/corr/abs-2508-10104}.

\begin{algorithm}[t]
\caption{GICDM}
\label{alg:gicdm}
\begin{algorithmic}[1]
\REQUIRE Real points $\{x^r_i\}_{i=1}^{N}$, generated points $\{x^g_j\}_{j=1}^{M}$,
neighborhood sizes $K_1, K_2$, quantile $q$

\FOR{$K \in \{K_1, K_2\}$}
  \STATE Apply ICDM to $\{x^r_i\}_{i=1}^N$ with neighborhood size $K$ to obtain scaling factors $\{\delta^r_{i,K}\}_{i=1}^N$
  \FOR{$i = 1$ to $N$}
    \STATE $\bar{\delta}^{r|r}_{i,K} \gets \frac{1}{K} \sum_{k=1}^K \delta^r_{\NN^{r,eq}_{k}(x^r_i),K}$
    \STATE $r^{r}_{i,K} \gets \frac{|\bar{\delta}^{r|r}_{i,K} - \delta^r_{i,K}|}{\bar{\delta}^{r|r}_{i,K}}$
  \ENDFOR
  \STATE $T_K \gets$ $q$-th quantile of $\{r^{r}_{i,K}\}_{i=1}^N$
\ENDFOR
\FOR{$j = 1$ to $M$}
  \STATE $keep \gets \text{True}$
  \FOR{$K \in \{K_1, K_2\}$}
    \STATE Compute $\delta^g_{j,K}$ using Eq.~\eqref{eq:delta_g} with $\{\delta^r_{i,K}\}$
    \STATE $\bar{\delta}^{r|g}_{j,K} \gets \frac{1}{K+1} \sum_{k=1}^{K+1} \delta^r_{\NN^{r,eq}_{k}(x^g_j),K}$
    \STATE $r^{g}_{j,K} \gets \frac{|\bar{\delta}^{r|g}_{j,K} - \delta^g_{j,K}|}{\bar{\delta}^{r|g}_{j,K}}$
    \IF{$r^{g}_{j,K} > T_K$}
      \STATE $keep \gets \text{False}$
    \ENDIF
  \ENDFOR
  \FOR{$i = 1$ to $N$}
    \IF{$keep$}
      \STATE $\text{GICDM}(x^r_i, x^g_j) \gets d(x^r_i, x^g_j) \cdot \delta^r_{i,K_1} \cdot \delta^g_{j,K_1}$
    \ENDIF
  \ENDFOR
\ENDFOR
\end{algorithmic}
\end{algorithm}

\section{Experiments}
\label{sec:experiments}

\begin{table}[t]
  \caption{\textbf{Benchmark success count.}
    Passed tests count (higher is better) on the synthetic benchmark of \citet{raisa2025position},
    before and after applying GICDM.
    Test categories: \emph{Purpose} (does the score evolution match the intended test objective), \emph{Bounds} (do metrics yield the exact expected values, e.g., 1 in the ideal case), and \emph{Other} (additional checks).
    Results are shown for the best-performing fidelity and coverage metrics:
    Clipped Density and Clipped Coverage, as well as for Precision and Recall.
    }
  \label{table:position_benchmark}
  \centering
  \begin{small}
  \begin{sc}
    \begin{tabular}{l c c c}
  \toprule
  & Purpose & Bounds & Other \\
  \midrule
  Precision & 6/14 & 3/13 & \textbf{2}/3 \\
  Clipped Density & 8/14 & 8/13 & \textbf{2}/3 \\
  \ \ \ \textbf{+ GICDM (ours)} & \textbf{10}/14 & \textbf{11}/13 & \textbf{2}/3 \\
  \midrule
  Recall & 5/14 & 2/13 & 0/3 \\
  Clipped Coverage & 8/14 & 9/13 & \textbf{1}/3 \\
  \ \ \ \textbf{+ GICDM (ours)} & \textbf{10}/14 & \textbf{11}/13 & \textbf{1}/3 \\
  \bottomrule
  \end{tabular}
  \end{sc}
  \end{small}
\end{table}

\begin{table}[t]
  \caption{\textbf{Correlation with human scores.} Pearson
  correlation between metric scores and human error rates.
  Results are not significant for FFHQ or other embeddings.
  GICDM maintains or improves Clipped Density's (the best-performing metric) correlation
  with human scores, especially for DINOv3, which exhibits more
  hubness.}
  \label{tab:correlation_human}
  \centering
  \setlength{\tabcolsep}{3pt}
  \begin{small}
  \begin{sc}
    \begin{tabular}{llccc}
    \toprule
    Embed. & Metric & \makecell{CIFAR\\10} & \makecell{Image\\Net} & \makecell{LSUN\\Bedroom} \\
    \midrule
    \multirow{3}{*}{DINOv2}
      & Precision           & $-0.82$ & $0.76$ & - \\
      & Clipped Density           & $0.93$  & $0.75$ & $0.81$ \\
      & \ \ \ + GICDM   & $0.97$  & $0.80$ & $0.80$ \\
    \midrule
    \multirow{3}{*}{DINOv3}
      & Precision           & $0.86$ & $0.74$ & $0.96$ \\
      & Clipped Density           & $0.82$  & $0.67$ & $0.94$ \\
      & \ \ \ + GICDM   & $0.95$  & $0.82$ & $0.97$ \\
    \bottomrule
  \end{tabular}
  \end{sc}
  \end{small}
\end{table}

\begin{table}[t]
  \caption{\textbf{Effect of guidance.} Metric scores for DINOv2 embeddings of ImageNet-trained DiT-XL-2, with and without classifier-free guidance. Guidance is expected to increase fidelity and decrease coverage. Clipped Density increases both with and without GICDM, but Clipped Coverage decreases only with GICDM.}
  \label{tab:guidance}
  \centering
  \begin{small}
  \begin{sc}
      \begin{tabular}{l c c}
    \toprule
    Metric & DiT-XL-2 & + Guidance \\
    \midrule
    Clipped Density & $0.74$  & $1.62$ $\nearrow$ \\
    \ \ \ + GICDM   & $0.86$  & $1.15$ $\nearrow$ \\
    \midrule
    Clipped Coverage & $0.64$  & $0.83$ $\nearrow$\\
    \ \ \ + GICDM   & $0.79$  & $0.55$ $\searrow$ \\
    \bottomrule
  \end{tabular}

  \end{sc}
  \end{small}
\end{table}
All experiments were conducted on a single NVIDIA H100 GPU with 80GB of memory and use 10 iterations for ICDM (see \Cref{sec:empirical_convergence} for a convergence analysis).

For metrics that use a neighborhood size $k$, we set the GICDM
neighborhood size to $2k$. This ensures the $k$-th nearest-neighbor distance remains comparable across points after hubness
correction. While ICDM equalizes the average distance to the $K$-nearest-neighbors ($\mu_i^{r,eq}$), the $K$-th distance itself
can still vary, especially for small $K$. Using
$K=2k$ places the $k$-th neighbor near the center of the averaged
distances, improving stability.

\textbf{Synthetic Benchmark.}
We evaluated standard metrics and their GICDM-corrected versions on the extensive synthetic benchmark proposed by
\citet{raisa2025position}. This benchmark evaluates metric scores evolution across 14 \emph{Purpose} scenarios, whether metrics yield the expected
values (e.g., 1 for realistic generated data) in 13 \emph{Bounds} cases, and 3 \emph{Other} miscellaneous tests.
Testing over 10 metrics (details in
\Cref{appendix:position_benchmark}), we found that GICDM consistently
improves two Purpose tests.
\Cref{table:position_benchmark} summarizes the results for the best-performing fidelity and coverage metrics:
GICDM substantially improves performance, passing 23 criteria instead of
18 for Clipped Density, and 22 instead of 18 for Clipped Coverage.
This underlines that hubness indeed causes several metric failures
and that GICDM successfully corrects them without altering already-satisfied criteria.

\textbf{Real Data and Human Correlation.}
We further evaluated 42 datasets generated by various models on CIFAR-10 \cite{krizhevsky2009learning}, ImageNet \cite{deng2009imagenet}, LSUN Bedroom \cite{yu2015lsun}, and FFHQ \cite{kazemi2014one}, using the data shared by \citet{stein2024exposing}.
We then computed the Pearson correlation between
fidelity metrics and human error rates (where volunteers discriminated real from generated
images).
As shown in \Cref{tab:correlation_human}, GICDM consistently maintains or improves the correlation of Clipped Density (the best-performing fidelity metric) with human judgments,
particularly for DINOv3 embeddings that are heavily affected by hubness.
This alignment confirms that GICDM preserves the relative positioning of generated samples
(desideratum 2) while mitigating hubness. Since desiderata
1 and 3 are satisfied by construction, GICDM fulfills all desiderata.
See \Cref{sec:human_correlation} for further details; notably, as also demonstrated by an additional benchmark on real CIFAR10 data (\Cref{sec:real_data_benchmark}),
GICDM is only
effective when metrics are robust to real outliers,
further highlighting the importance of using robust metrics.

\textbf{Classifier-Free Guidance.}
\Cref{tab:guidance} shows the effect of classifier-free guidance on DiT-XL-2 for ImageNet
with DINOv2 embeddings. Both prior best metrics, Clipped Density and Clipped Coverage, increase with guidance,
which is counter-intuitive since guidance should improve fidelity but reduce coverage.
With GICDM, Clipped Density increases, but Clipped Coverage decreases as expected.
These results highlight how hubness can distort metrics, and how GICDM corrects it.

\textbf{Qualitative Visualization.}
\Cref{fig:intro_viz} qualitatively shows GICDM's effect on LSUN Bedroom with DINOv2 embeddings: before applying GICDM, the most frequent generated hub appears in $64$ real neighborhoods, while after applying GICDM, it appears in only $4$. These $4$ images are more semantically similar to the hub than the average $64$ were. Thus, GICDM effectively mitigates hubness and restores a more semantically meaningful neighborhood structure.

In summary, by removing hubness, GICDM consistently improves metric results,
confirming that hubness caused the observed evaluation issues and that GICDM resolves
them without removing desirable data structure. This leads to more
trustworthy evaluation of generative models and better alignment with human judgment.

\section{Conclusion}
\label{sec:conclusion}

In this work, we introduced GICDM, a hubness mitigation
method specifically designed for generative model evaluation.
GICDM leverages ICDM to uniformize the density of
the real data manifold, and computes out-of-sample scaling
factors for generated samples based solely on their relationship
to the real set. To prevent overcorrection, a multi-scale filtering
strategy discards points with inconsistent local scaling.

Extensive experiments on both synthetic and real benchmarks demonstrate that GICDM resolves failures of existing metrics,
restores reliable fidelity and coverage scores, and improves alignment with human judgment.
While applicable to any distance-based metric, a limitation is that its effectiveness
relies on the underlying metric being robust to real outliers, such as Clipped Density and Clipped Coverage.

Overall, GICDM enables more reliable and trustworthy evaluation of generative models in high dimensional spaces.

Code to reproduce the experiments and use GICDM is available at \url{https://github.com/nicolassalvy/GICDM}.

\section*{Acknowledgements}

This work benefited from state aid managed by the Agence Nationale de la Recherche under the France 2030 programme, reference ANR-22-PESN-0012 and from the European Union’s Horizon 2020 Research Infrastructures Grant EBRAIN-Health 101058516.
This work was performed using HPC resources from GENCI-IDRIS (Grant 2025-AD011014887R2).

\section*{Impact Statement}

This paper presents work whose goal is to advance the field of Generative Model Evaluation.
By improving the assessment of generative models, our work can help guide the development of more reliable models, which is important for data-scarce domains such as healthcare. At the same time, we recognize that generative models may be misused (e.g., for deepfakes).

\bibliography{Biblio}

@article{DBLP:journals/corr/KomarovDD13,
	author = {Komarov, Ivan and Dashti, Ali and D'Souza, Roshan},
	journal = {CoRR},
	title = {Fast \emph{k}-NNG construction with GPU-based quick multi-select},
	volume = {abs/1309.5478},
	year = {2013}}

@article{toeplitz1911allgemeine,
	author = {Toeplitz, Otto},
	journal = {Prace matematyczno-fizyczne},
	number = {1},
	pages = {113--119},
	publisher = {Towarzystwo Naukowe Warszawskie},
	title = {{\"U}ber allgemeine lineare Mittelbildungen.},
	volume = {22},
	year = {1911}}

@article{tucciarone1973development,
	author = {Tucciarone, John},
	journal = {Archive for history of exact sciences},
	number = {1/2},
	pages = {1--40},
	publisher = {JSTOR},
	title = {The development of the theory of summable divergent series from 1880 to 1925},
	volume = {10},
	year = {1973}}

@book{mood1950introduction,
	author = {Mood, Alexander McFarlane},
	publisher = {McGraw-hill},
	title = {Introduction to the Theory of Statistics.},
	year = {1950}}

@book{Bellman1961,
	author = {Bellman, Richard E.},
	publisher = {Princeton University Press},
	title = {Adaptive Control Processes: A Guided Tour},
	year = {1961}}

@article{loftsgaarden1965nonparametric,
	author = {Loftsgaarden, Don O and Quesenberry, Charles P},
	journal = {The Annals of Mathematical Statistics},
	number = {3},
	pages = {1049--1051},
	title = {A nonparametric estimate of a multivariate density function},
	volume = {36},
	year = {1965}}

@inproceedings{beyer1999nearest,
	author = {Beyer, Kevin and Goldstein, Jonathan and Ramakrishnan, Raghu and Shaft, Uri},
	booktitle = {International conference on database theory},
	organization = {Springer},
	pages = {217--235},
	title = {When is ``nearest neighbor'' meaningful?},
	year = {1999}}

@inproceedings{aggarwal2001surprising,
	author = {Aggarwal, Charu C and Hinneburg, Alexander and Keim, Daniel A},
	booktitle = {International conference on database theory},
	organization = {Springer},
	pages = {420--434},
	title = {On the surprising behavior of distance metrics in high dimensional space},
	year = {2001}}

@article{pachet2004improving,
	author = {Pachet, Francois and Aucouturier, Jean-Julien},
	journal = {Journal of negative results in speech and audio sciences},
	number = {1},
	pages = {1--13},
	title = {Improving timbre similarity: How high is the sky},
	volume = {1},
	year = {2004}}

@article{zelnik2004self,
	author = {Zelnik-Manor, Lihi and Perona, Pietro},
	journal = {Advances in neural information processing systems},
	title = {Self-tuning spectral clustering},
	volume = {17},
	year = {2004}}

@inproceedings{jegou2007contextual,
	author = {J{\'{e}}gou, Herve and Harzallah, Hedi and Schmid, Cordelia},
	booktitle = {2007 IEEE Conference on computer vision and pattern recognition},
	organization = {IEEE},
	pages = {1--8},
	title = {A contextual dissimilarity measure for accurate and efficient image search},
	year = {2007}}

@inproceedings{deng2009imagenet,
	author = {Deng, Jia and Dong, Wei and Socher, Richard and Li, Li-Jia and Li, Kai and Fei-Fei, Li},
	booktitle = {2009 IEEE conference on computer vision and pattern recognition},
	organization = {Ieee},
	pages = {248--255},
	title = {Imagenet: A large-scale hierarchical image database},
	year = {2009}}

@article{krizhevsky2009learning,
	author = {Krizhevsky, Alex and Hinton, Geoffrey and others},
	title = {Learning multiple layers of features from tiny images},
	year = {2009}}

@article{jegou2010contextual,
	author = {J{\'{e}}gou, Herv{\'{e}} and Schmid, Cordelia and Harzallah, Hedi and Verbeek, Jakob},
	journal = {{IEEE} Trans. Pattern Anal. Mach. Intell.},
	number = {1},
	pages = {2--11},
	title = {Accurate Image Search Using the Contextual Dissimilarity Measure},
	volume = {32},
	year = {2010}}

@article{radovanovic2010hubs,
	author = {Radovanovic, Milos and Nanopoulos, Alexandros and Ivanovic, Mirjana},
	journal = {Journal of Machine Learning Research},
	number = {sept},
	pages = {2487--2531},
	title = {Hubs in space: Popular nearest neighbors in high-dimensional data},
	volume = {11},
	year = {2010}}

@inproceedings{DBLP:conf/iccp2/TomasevBMN11,
	author = {Tomasev, Nenad and Brehar, Raluca and Mladenic, Dunja and Nedevschi, Sergiu},
	booktitle = {{IEEE} International Conference on Intelligent Computer Communication and Processing, {ICCP} 2011, Cluj-Napoca, Romania, August 25-27, 2011},
	pages = {367--374},
	publisher = {{IEEE}},
	title = {The influence of hubness on nearest-neighbor methods in object recognition},
	year = {2011}}

@article{schnitzer2012local,
	author = {Schnitzer, Dominik and Flexer, Arthur and Schedl, Markus and Widmer, Gerhard},
	journal = {The Journal of Machine Learning Research},
	number = {1},
	pages = {2871--2902},
	title = {Local and global scaling reduce hubs in space},
	volume = {13},
	year = {2012}}

@incollection{low2013hubness,
	author = {Low, Thomas and Borgelt, Christian and Stober, Sebastian and N{\"u}rnberger, Andreas},
	booktitle = {Towards Advanced Data Analysis by Combining Soft Computing and Statistics},
	pages = {267--278},
	publisher = {Springer},
	title = {The hubness phenomenon: Fact or artifact?},
	year = {2013}}

@inproceedings{kazemi2014one,
	author = {Kazemi, Vahid and Sullivan, Josephine},
	booktitle = {Proceedings of the IEEE conference on computer vision and pattern recognition},
	pages = {1867--1874},
	title = {One millisecond face alignment with an ensemble of regression trees},
	year = {2014}}

@article{yu2015lsun,
	author = {Yu, Fisher and Seff, Ari and Zhang, Yinda and Song, Shuran and Funkhouser, Thomas and Xiao, Jianxiong},
	journal = {arXiv preprint arXiv:1506.03365},
	title = {Lsun: Construction of a large-scale image dataset using deep learning with humans in the loop},
	year = {2015}}

@inproceedings{DBLP:journals/corr/SimonyanZ14a,
	author = {Simonyan, Karen and Zisserman, Andrew},
	booktitle = {3rd International Conference on Learning Representations, {ICLR} 2015, San Diego, CA, USA, May 7-9, 2015, Conference Track Proceedings},
	editor = {Bengio, Yoshua and LeCun, Yann},
	title = {Very Deep Convolutional Networks for Large-Scale Image Recognition},
	year = {2015}}

@inproceedings{hara2015reducing,
	author = {Hara, Kazuo and Suzuki, Ikumi and Kobayashi, Kei and Fukumizu, Kenji},
	booktitle = {Proceedings of the 38th international ACM SIGIR conference on research and development in information retrieval},
	pages = {815--818},
	title = {Reducing hubness: A cause of vulnerability in recommender systems},
	year = {2015}}

@article{flexer2015choosing,
	author = {Flexer, Arthur and Schnitzer, Dominik},
	journal = {Neurocomputing},
	pages = {281--287},
	title = {Choosing lp norms in high-dimensional spaces based on hub analysis},
	volume = {169},
	year = {2015}}

@inproceedings{DBLP:conf/cvpr/SzegedyVISW16,
	author = {Szegedy, Christian and Vanhoucke, Vincent and Ioffe, Sergey and Shlens, Jonathon and Wojna, Zbigniew},
	booktitle = {2016 {IEEE} Conference on Computer Vision and Pattern Recognition, {CVPR} 2016, Las Vegas, NV, USA, June 27-30, 2016},
	pages = {2818--2826},
	publisher = {{IEEE} Computer Society},
	title = {Rethinking the Inception Architecture for Computer Vision},
	year = {2016}}

@inproceedings{hara2016flattening,
	author = {Hara, Kazuo and Suzuki, Ikumi and Kobayashi, Kei and Fukumizu, Kenji and Radovanovic, Milos},
	booktitle = {Proceedings of the AAAI Conference on Artificial Intelligence},
	number = {1},
	title = {Flattening the density gradient for eliminating spatial centrality to reduce hubness},
	volume = {30},
	year = {2016}}

@inproceedings{DBLP:conf/ismir/DefferrardBVB17,
	author = {Defferrard, Micha{\"{e}}l and Benzi, Kirell and Vandergheynst, Pierre and Bresson, Xavier},
	booktitle = {Proceedings of the 18th International Society for Music Information Retrieval Conference, {ISMIR} 2017, Suzhou, China, October 23-27, 2017},
	editor = {Cunningham, Sally Jo and Duan, Zhiyao and Hu, Xiao and Turnbull, Douglas},
	pages = {316--323},
	title = {{FMA:} {A} Dataset for Music Analysis},
	year = {2017}}

@article{gulrajani2017improved,
	author = {Gulrajani, Ishaan and Ahmed, Faruk and Arjovsky, Martin and Dumoulin, Vincent and Courville, Aaron C},
	journal = {Advances in neural information processing systems},
	title = {Improved training of wasserstein gans},
	volume = {30},
	year = {2017}}

@inproceedings{odena2017conditional,
	author = {Odena, Augustus and Olah, Christopher and Shlens, Jonathon},
	booktitle = {International conference on machine learning},
	organization = {PMLR},
	pages = {2642--2651},
	title = {Conditional image synthesis with auxiliary classifier gans},
	year = {2017}}

@article{heusel2017gans,
	author = {Heusel, Martin and Ramsauer, Hubert and Unterthiner, Thomas and Nessler, Bernhard and Hochreiter, Sepp},
	journal = {Advances in neural information processing systems},
	title = {Gans trained by a two time-scale update rule converge to a local nash equilibrium},
	volume = {30},
	year = {2017}}

@inproceedings{lample2018word,
	author = {Lample, Guillaume and Conneau, Alexis and Ranzato, Marc'Aurelio and Denoyer, Ludovic and J{\'e}gou, Herv{\'e}},
	booktitle = {International conference on learning representations},
	title = {Word translation without parallel data},
	year = {2018}}

@inproceedings{feldbauer2018fast,
	author = {Feldbauer, Roman and Leodolter, Maximilian and Plant, Claudia and Flexer, Arthur},
	booktitle = {2018 IEEE International Conference on Big Knowledge (ICBK)},
	organization = {IEEE},
	pages = {358--367},
	title = {Fast approximate hubness reduction for large high-dimensional data},
	year = {2018}}

@article{sajjadi2018assessing,
	author = {Sajjadi, Mehdi SM and Bachem, Olivier and Lucic, Mario and Bousquet, Olivier and Gelly, Sylvain},
	journal = {Advances in neural information processing systems},
	title = {Assessing generative models via precision and recall},
	volume = {31},
	year = {2018}}

@book{vershynin2018high,
	author = {Vershynin, Roman},
	publisher = {Cambridge university press},
	title = {High-dimensional probability: An introduction with applications in data science},
	volume = {47},
	year = {2018}}

@inproceedings{karras2019style,
	author = {Karras, Tero and Laine, Samuli and Aila, Timo},
	booktitle = {Proceedings of the IEEE/CVF conference on computer vision and pattern recognition},
	pages = {4401--4410},
	title = {A style-based generator architecture for generative adversarial networks},
	year = {2019}}

@article{chen2019residual,
	author = {Chen, Ricky TQ and Behrmann, Jens and Duvenaud, David K and Jacobsen, J{\"o}rn-Henrik},
	journal = {Advances in Neural Information Processing Systems},
	title = {Residual flows for invertible generative modeling},
	volume = {32},
	year = {2019}}

@inproceedings{turner2019metropolis,
	author = {Turner, Ryan and Hung, Jane and Frank, Eric and Saatchi, Yunus and Yosinski, Jason},
	booktitle = {International Conference on Machine Learning},
	organization = {PMLR},
	pages = {6345--6353},
	title = {Metropolis-hastings generative adversarial networks},
	year = {2019}}

@article{wu2019logan,
	author = {Wu, Yan and Donahue, Jeff and Balduzzi, David and Simonyan, Karen and Lillicrap, Timothy},
	journal = {arXiv preprint arXiv:1912.00953},
	title = {Logan: Latent optimisation for generative adversarial networks},
	year = {2019}}

@inproceedings{brock2018large,
	author = {Brock, Andrew and Donahue, Jeff and Simonyan, Karen},
	booktitle = {7th International Conference on Learning Representations, {ICLR} 2019, New Orleans, LA, USA, May 6-9, 2019},
	title = {Large scale GAN training for high fidelity natural image synthesis},
	year = {2019}}

@article{feldbauer2019comprehensive,
	author = {Feldbauer, Roman and Flexer, Arthur},
	journal = {Knowledge and Information Systems},
	number = {1},
	pages = {137--166},
	title = {A comprehensive empirical comparison of hubness reduction in high-dimensional spaces},
	volume = {59},
	year = {2019}}

@article{kynkaanniemi2019improved,
	author = {Kynk{\"a}{\"a}nniemi, Tuomas and Karras, Tero and Laine, Samuli and Lehtinen, Jaakko and Aila, Timo},
	journal = {Advances in neural information processing systems},
	title = {Improved precision and recall metric for assessing generative models},
	volume = {32},
	year = {2019}}

@article{karras2020training,
	author = {Karras, Tero and Aittala, Miika and Hellsten, Janne and Laine, Samuli and Lehtinen, Jaakko and Aila, Timo},
	journal = {Advances in neural information processing systems},
	pages = {12104--12114},
	title = {Training generative adversarial networks with limited data},
	volume = {33},
	year = {2020}}

@article{vahdat2020nvae,
	author = {Vahdat, Arash and Kautz, Jan},
	journal = {Advances in neural information processing systems},
	pages = {19667--19679},
	title = {NVAE: A deep hierarchical variational autoencoder},
	volume = {33},
	year = {2020}}

@article{ho2020denoising,
	author = {Ho, Jonathan and Jain, Ajay and Abbeel, Pieter},
	journal = {Advances in neural information processing systems},
	pages = {6840--6851},
	title = {Denoising diffusion probabilistic models},
	volume = {33},
	year = {2020}}

@inproceedings{naeem2020reliable,
	author = {Naeem, Muhammad Ferjad and Oh, Seong Joon and Uh, Youngjung and Choi, Yunjey and Yoo, Jaejun},
	booktitle = {International Conference on Machine Learning},
	organization = {PMLR},
	pages = {7176--7185},
	title = {Reliable fidelity and diversity metrics for generative models},
	year = {2020}}

@article{yang2021data,
	author = {Yang, Ceyuan and Shen, Yujun and Xu, Yinghao and Zhou, Bolei},
	journal = {Advances in Neural Information Processing Systems},
	pages = {9378--9390},
	title = {Data-efficient instance generation from instance discrimination},
	volume = {34},
	year = {2021}}

@article{sauer2021projected,
	author = {Sauer, Axel and Chitta, Kashyap and M{\"u}ller, Jens and Geiger, Andreas},
	journal = {Advances in Neural Information Processing Systems},
	pages = {17480--17492},
	title = {Projected gans converge faster},
	volume = {34},
	year = {2021}}

@article{dhariwal2021diffusion,
	author = {Dhariwal, Prafulla and Nichol, Alexander},
	journal = {Advances in neural information processing systems},
	pages = {8780--8794},
	title = {Diffusion models beat gans on image synthesis},
	volume = {34},
	year = {2021}}

@article{vahdat2021score,
	author = {Vahdat, Arash and Kreis, Karsten and Kautz, Jan},
	journal = {Advances in neural information processing systems},
	pages = {11287--11302},
	title = {Score-based generative modeling in latent space},
	volume = {34},
	year = {2021}}

@article{kang2021rebooting,
	author = {Kang, Minguk and Shim, Woohyeon and Cho, Minsu and Park, Jaesik},
	journal = {Advances in neural information processing systems},
	pages = {23505--23518},
	title = {Rebooting acgan: Auxiliary classifier gans with stable training},
	volume = {34},
	year = {2021}}

@inproceedings{nichol2021improved,
	author = {Nichol, Alexander Quinn and Dhariwal, Prafulla},
	booktitle = {International conference on machine learning},
	organization = {PMLR},
	pages = {8162--8171},
	title = {Improved denoising diffusion probabilistic models},
	year = {2021}}

@inproceedings{zhang2022styleswin,
	author = {Zhang, Bowen and Gu, Shuyang and Zhang, Bo and Bao, Jianmin and Chen, Dong and Wen, Fang and Wang, Yong and Guo, Baining},
	booktitle = {Proceedings of the IEEE/CVF conference on computer vision and pattern recognition},
	pages = {11304--11314},
	title = {Styleswin: Transformer-based gan for high-resolution image generation},
	year = {2022}}

@inproceedings{walton2022stylenat,
	author = {Walton, Steven and Hassani, Ali and Xu, Xingqian and Wang, Zhangyang and Shi, Humphrey},
	booktitle = {Proceedings of the Computer Vision and Pattern Recognition Conference (CVPR) Workshops},
	title = {Efficient Image Generation with Variadic Attention Heads},
	year = {2025}}

@article{hazami2022efficientvdvae,
	author = {Hazami, Louay and Mama, Rayhane and Thurairatnam, Ragavan},
	journal = {arXiv preprint arXiv:2203.13751},
	title = {Efficient-vdvae: Less is more},
	year = {2022}}

@inproceedings{bond2022unleashing,
	author = {Bond-Taylor, Sam and Hessey, Peter and Sasaki, Hiroshi and Breckon, Toby P and Willcocks, Chris G},
	booktitle = {European Conference on Computer Vision},
	organization = {Springer},
	pages = {170--188},
	title = {Unleashing transformers: Parallel token prediction with discrete absorbing diffusion for fast high-resolution image generation from vector-quantized codes},
	year = {2022}}

@inproceedings{lee2022autoregressive,
	author = {Lee, Doyup and Kim, Chiheon and Kim, Saehoon and Cho, Minsu and Han, Wook-Shin},
	booktitle = {Proceedings of the IEEE/CVF Conference on Computer Vision and Pattern Recognition},
	pages = {11523--11532},
	title = {Autoregressive image generation using residual quantization},
	year = {2022}}

@inproceedings{chang2022maskgit,
	author = {Chang, Huiwen and Zhang, Han and Jiang, Lu and Liu, Ce and Freeman, William T},
	booktitle = {Proceedings of the IEEE/CVF conference on computer vision and pattern recognition},
	pages = {11315--11325},
	title = {Maskgit: Masked generative image transformer},
	year = {2022}}

@inproceedings{rombach2022high,
	author = {Rombach, Robin and Blattmann, Andreas and Lorenz, Dominik and Esser, Patrick and Ommer, Bj{\"o}rn},
	booktitle = {Proceedings of the IEEE/CVF conference on computer vision and pattern recognition},
	pages = {10684--10695},
	title = {High-resolution image synthesis with latent diffusion models},
	year = {2022}}

@inproceedings{sauer2022stylegan,
	author = {Sauer, Axel and Schwarz, Katja and Geiger, Andreas},
	booktitle = {ACM SIGGRAPH 2022 conference proceedings},
	pages = {1--10},
	title = {Stylegan-xl: Scaling stylegan to large diverse datasets},
	year = {2022}}

@inproceedings{pinaya2022brain,
	author = {Pinaya, Walter HL and Tudosiu, Petru-Daniel and Dafflon, Jessica and Da Costa, Pedro F and Fernandez, Virginia and Nachev, Parashkev and Ourselin, Sebastien and Cardoso, M Jorge},
	booktitle = {MICCAI Workshop on Deep Generative Models},
	organization = {Springer},
	pages = {117--126},
	title = {Brain imaging generation with latent diffusion models},
	year = {2022}}

@article{amblard2022hubness,
	author = {Amblard, Elise and Bac, Jonathan and Chervov, Alexander and Soumelis, Vassili and Zinovyev, Andrei},
	journal = {Bioinformatics},
	number = {4},
	pages = {1045--1051},
	title = {Hubness reduction improves clustering and trajectory inference in single-cell transcriptomic data},
	volume = {38},
	year = {2022}}

@article{obraczka2022fast,
	author = {Obraczka, Daniel and Rahm, Erhard},
	journal = {SN Computer Science},
	number = {6},
	pages = {501},
	title = {Fast hubness-reduced nearest neighbor search for entity alignment in knowledge graphs},
	volume = {3},
	year = {2022}}

@inproceedings{alaa2022faithful,
	author = {Alaa, Ahmed and Van Breugel, Boris and Saveliev, Evgeny S and van der Schaar, Mihaela},
	booktitle = {International Conference on Machine Learning},
	organization = {PMLR},
	pages = {290--306},
	title = {How faithful is your synthetic data? sample-level metrics for evaluating and auditing generative models},
	year = {2022}}

@inproceedings{wang2022diffusion,
	author = {Wang, Zhendong and Zheng, Huangjie and He, Pengcheng and Chen, Weizhu and Zhou, Mingyuan},
	booktitle = {The Eleventh International Conference on Learning Representations, {ICLR} 2023, Kigali, Rwanda, May 1-5, 2023},
	title = {Diffusion-gan: Training gans with diffusion},
	year = {2023}}

@article{song2023consistency,
	author = {Song, Yang and Dhariwal, Prafulla and Chen, Mark and Sutskever, Ilya},
	journal = {International Conference on Machine Learning},
	title = {Consistency models},
	year = {2023}}

@inproceedings{kang2023scaling,
	author = {Kang, Minguk and Zhu, Jun-Yan and Zhang, Richard and Park, Jaesik and Shechtman, Eli and Paris, Sylvain and Park, Taesung},
	booktitle = {Proceedings of the IEEE/CVF conference on computer vision and pattern recognition},
	pages = {10124--10134},
	title = {Scaling up gans for text-to-image synthesis},
	year = {2023}}

@inproceedings{peebles2023scalable,
	author = {Peebles, William and Xie, Saining},
	booktitle = {Proceedings of the IEEE/CVF international conference on computer vision},
	pages = {4195--4205},
	title = {Scalable diffusion models with transformers},
	year = {2023}}

@inproceedings{xu2023pfgm++,
	author = {Xu, Yilun and Liu, Ziming and Tian, Yonglong and Tong, Shangyuan and Tegmark, Max and Jaakkola, Tommi},
	booktitle = {International Conference on Machine Learning},
	organization = {PMLR},
	pages = {38566--38591},
	title = {Pfgm++: Unlocking the potential of physics-inspired generative models},
	year = {2023}}

@article{kang2023studiogan,
	author = {Kang, Minguk and Shin, Joonghyuk and Park, Jaesik},
	journal = {IEEE Transactions on Pattern Analysis and Machine Intelligence},
	number = {12},
	pages = {15725--15742},
	title = {StudioGAN: a taxonomy and benchmark of GANs for image synthesis},
	volume = {45},
	year = {2023}}

@inproceedings{DBLP:conf/icassp/ElizaldeDIW23,
	author = {Elizalde, Benjamin and Deshmukh, Soham and Ismail, Mahmoud Al and Wang, Huaming},
	booktitle = {{IEEE} International Conference on Acoustics, Speech and Signal Processing {ICASSP} 2023, Rhodes Island, Greece, June 4-10, 2023},
	pages = {1--5},
	publisher = {{IEEE}},
	title = {{CLAP} Learning Audio Concepts from Natural Language Supervision},
	year = {2023}}

@article{oquab2023dinov2,
	author = {Oquab, Maxime and Darcet, Timoth{\'e}e and Moutakanni, Th{\'e}o and Vo, Huy and Szafraniec, Marc and Khalidov, Vasil and Fernandez, Pierre and Haziza, Daniel and Massa, Francisco and El-Nouby, Alaaeldin and others},
	journal = {arXiv preprint arXiv:2304.07193},
	title = {Dinov2: Learning robust visual features without supervision},
	year = {2023}}

@article{kim2023topp,
	author = {Kim, Pum Jun and Jang, Yoojin and Kim, Jisu and Yoo, Jaejun},
	journal = {Advances in Neural Information Processing Systems},
	pages = {7831--7866},
	title = {Topp\&r: Robust support estimation approach for evaluating fidelity and diversity in generative models},
	volume = {36},
	year = {2023}}

@inproceedings{park2023probabilistic,
	author = {Park, Dogyun and Kim, Suhyun},
	booktitle = {Proceedings of the IEEE/CVF international conference on computer vision},
	pages = {20099--20109},
	title = {Probabilistic precision and recall towards reliable evaluation of generative models},
	year = {2023}}

@inproceedings{cheema2023precision,
	author = {Cheema, Fasil and Urner, Ruth},
	booktitle = {International Conference on Artificial Intelligence and Statistics},
	organization = {PMLR},
	pages = {6571--6594},
	title = {Precision recall cover: A method for assessing generative models},
	year = {2023}}

@inproceedings{khayatkhoei2023emergent,
	author = {Khayatkhoei, Mahyar and AbdAlmageed, Wael},
	booktitle = {International Conference on Machine Learning},
	organization = {PMLR},
	pages = {16326--16343},
	title = {Emergent asymmetry of precision and recall for measuring fidelity and diversity of generative models in high dimensions},
	year = {2023}}

@article{stein2024exposing,
	author = {Stein, George and Cresswell, Jesse and Hosseinzadeh, Rasa and Sui, Yi and Ross, Brendan and Villecroze, Valentin and Liu, Zhaoyan and Caterini, Anthony L and Taylor, Eric and Loaiza-Ganem, Gabriel},
	journal = {Advances in Neural Information Processing Systems},
	title = {Exposing flaws of generative model evaluation metrics and their unfair treatment of diffusion models},
	volume = {36},
	year = {2023}}

@article{bluethgen2024vision,
	author = {Bluethgen, Christian and Chambon, Pierre and Delbrouck, Jean-Benoit and van der Sluijs, Rogier and Po{\l}acin, Ma{\l}gorzata and Zambrano Chaves, Juan Manuel and Abraham, Tanishq Mathew and Purohit, Shivanshu and Langlotz, Curtis P and Chaudhari, Akshay S},
	journal = {Nature Biomedical Engineering},
	pages = {1--13},
	title = {A vision--language foundation model for the generation of realistic chest X-ray images},
	year = {2024}}

@article{koetzier2024generating,
	author = {Koetzier, Lennart R and Wu, Jie and Mastrodicasa, Domenico and Lutz, Aline and Chung, Matthew and Koszek, W Adam and Pratap, Jayanth and Chaudhari, Akshay S and Rajpurkar, Pranav and Lungren, Matthew P and others},
	journal = {Radiology},
	number = {3},
	pages = {e232471},
	title = {Generating synthetic data for medical imaging},
	volume = {312},
	year = {2024}}

@article{DBLP:journals/corr/abs-2508-10104,
	author = {Sim{\'{e}}oni, Oriane and Vo, Huy V. and Seitzer, Maximilian and Baldassarre, Federico and Oquab, Maxime and Jose, Cijo and Khalidov, Vasil and Szafraniec, Marc and Yi, Seung Eun and Ramamonjisoa, Micha{\"{e}}l and Massa, Francisco and Haziza, Daniel and Wehrstedt, Luca and Wang, Jianyuan and Darcet, Timoth{\'{e}}e and Moutakanni, Th{\'{e}}o and Sentana, Leonel and Roberts, Claire and Vedaldi, Andrea and Tolan, Jamie and Brandt, John and Couprie, Camille and Mairal, Julien and J{\'{e}}gou, Herv{\'{e}} and Labatut, Patrick and Bojanowski, Piotr},
	journal = {arXiv preprint arXiv:2508.10104},
	title = {DINOv3},
	year = {2025}}

@inproceedings{raisa2025position,
	author = {R{\"a}is{\"a}, Ossi and van Breugel, Boris and van der Schaar, Mihaela},
	booktitle = {Forty-second International Conference on Machine Learning Position Paper Track},
	title = {Position: All Current Generative Fidelity and Diversity Metrics are Flawed},
	year = {2025}}

@inproceedings{salvy2026enhanced,
	author = {Salvy, Nicolas and Talbot, Hugues and Thirion, Bertrand},
	booktitle = {The Fourteenth International Conference on Learning Representations},
	title = {Enhanced Generative Model Evaluation with Clipped Density and Coverage},
	year = {2026}}
\bibliographystyle{icml2026}

\newpage
\appendix
\onecolumn

\section{Hubness Measures}
\label{sec:hubness_measures}

Hubness skews the distribution of $k$-occurrences, $O_k$.
To quantify hubness, \citet{low2013hubness} proposed examining the right tail of the distribution by measuring the largest $k$-occurrence values
and averaging them for robustness:
\begin{equation*}
    h_1^k(q) = \frac{1}{k |\mathcal{O}_k(q)|} \sum_{x \in \mathcal{O}_k(q)} O_k(x),
\end{equation*}
where $\mathcal{O}_k(q)$ is the set of the $\lfloor qn \rfloor$ points with the highest $k$-occurrences,
and typically $q=0.01$ is used for a small proportion.
$h_1^k(q)$ measures, on average, how much more often the top $q$-fraction of points appear among $k$-nearest-neighbors compared to the average point.

On the opposite side of the distribution, \citet{flexer2015choosing} proposed measuring the proportion of antihubs,
$A^k$, defined as points that never appear in the $k$-nearest-neighbors of any other point:
\begin{equation*}
    A^k = \frac{|\{x \in D \mid O_k(x) = 0\}|}{|D|}.
\end{equation*}

\Cref{fig:hubness_evolution} shows that both measures indicate increasing hubness for standard Gaussian data as dimensionality grows.

\begin{figure}[b!]
  \begin{center}
    \includegraphics[width=0.55\textwidth]{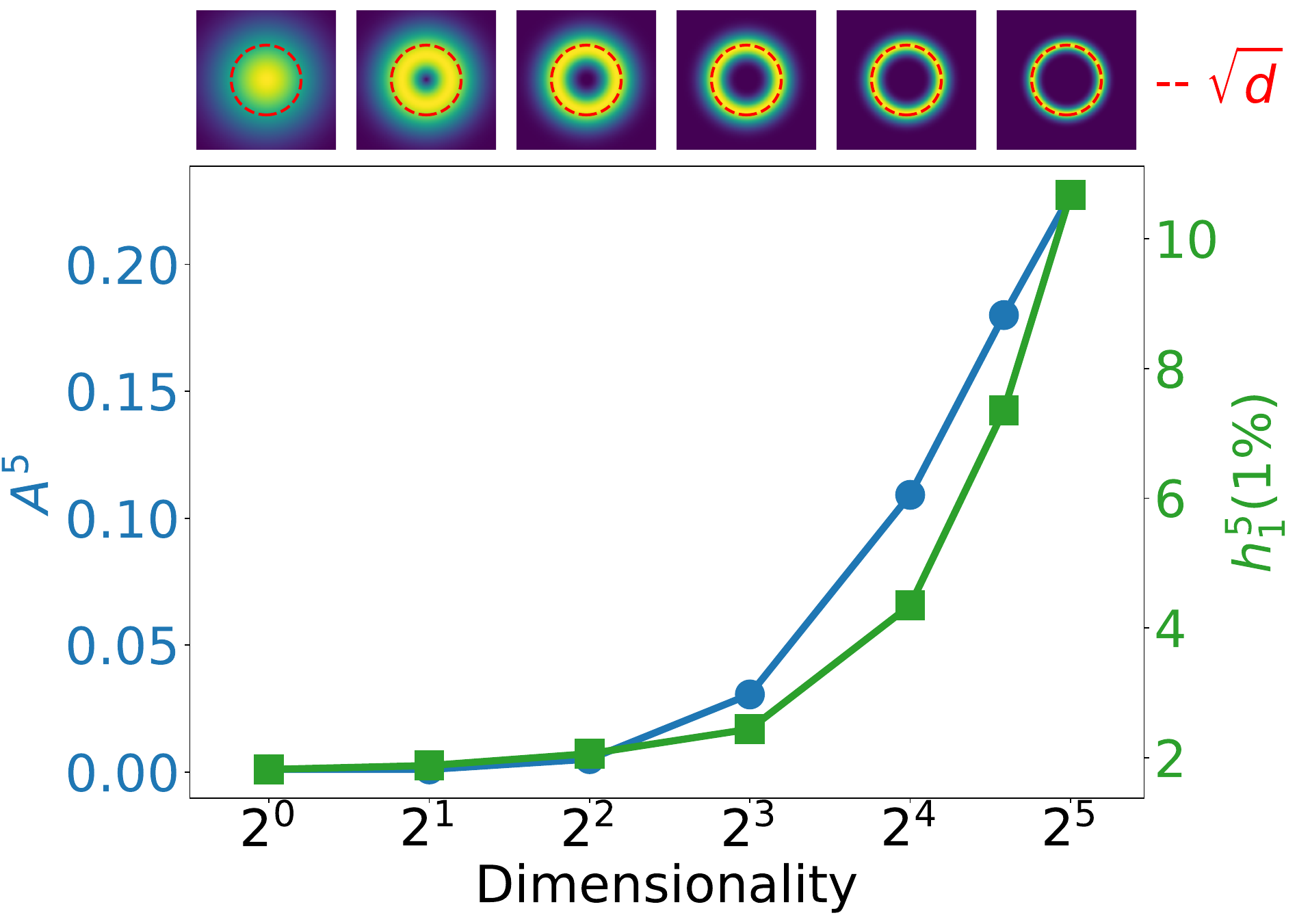}
    \caption{\textbf{Hubness Evolution} for standard Gaussian data ($N=20000$) as dimensionality increases.
    \textit{Top}: Intuitive 2D visualization of a Gaussian using a wrapped chi-squared distribution with $d$ degrees of freedom (lighter colors indicate higher density). As the dimension increases, samples concentrate on a thin spherical shell.
    \textit{Bottom}: The y-axes are not shared. Both hubness measures ($A^5$ and $h_1^5(1\%)$) indicate that hubness increases with dimensionality.
    }
    \label{fig:hubness_evolution}
  \end{center}
\end{figure}

\subsection{Hubness in Common Embedding Spaces}

\Cref{tab:hubness_measures} reports results for Cifar10 \cite{krizhevsky2009learning}, FFHQ \cite{kazemi2014one},
ImageNet \cite{deng2009imagenet}, and LSUN Bedroom \cite{yu2015lsun} datasets embedded with
VGG16 \cite{DBLP:journals/corr/SimonyanZ14a}, Inceptionv3 \cite{DBLP:conf/cvpr/SzegedyVISW16},
DINOv2 \cite{oquab2023dinov2}, and DINOv3 \cite{DBLP:journals/corr/abs-2508-10104} image encoders,
as well as the Free Music Archive \cite{DBLP:conf/ismir/DefferrardBVB17} dataset
embedded with the CLAP audio encoder \cite{DBLP:conf/icassp/ElizaldeDIW23}.
All exhibit hubness.
DINOv2, which has the lowest dimensionality (1024) among the image encoders (VGG16: 4096, Inceptionv3: 2048, DINOv3: 4096),
is the least affected, but still shows hubness.

For the least affected case, ImageNet with DINOv2 embeddings, the top 1\% most frequent points appear more than 4 times as often as the average point among the 5-nearest-neighbors of other points, and 10\% of points are antihubs, never appearing in any 5-nearest-neighbors list.
This hubness affects the reliability of nearest-neighbor relationships, which are crucial for distance-based generative model evaluation metrics.

\begin{table}[t]
  \caption{\textbf{Hubness measures for various datasets and embedding spaces}.
  $A^5$ is the proportion of antihubs ($O_5(x) = 0$), and $h_1^5(1\%)$ quantifies how
  much more frequently the top 1\% of points appear among $k$-nearest-neighbors compared to the average.
  In the absence of hubness, the $k$-occurrence distribution would be nearly symmetric
  around its mean
  (as in \Cref{fig:low_dim_occurrences}), yielding $h_1^5(1\%)$ values slightly above $2$
  and $A^5$ values below $0.01$.
  All reported values are at least twice as large, indicating pronounced hubness.}
  \label{tab:hubness_measures}
  \begin{center}
    \begin{small}
      \begin{sc}
        \begin{tabular}{lcrrr}
        \toprule
        Embedding & Dimension & Dataset & $h_1^5(1\%)$ & $A^5$ \\
        \midrule
        \multirow{4}{*}{VGG16} & \multirow{4}{*}{4096} & Cifar10 & 10.4 & 0.21 \\
         & & FFHQ & 9.8 & 0.23 \\
         & & ImageNet & 12.6 & 0.20 \\
         & & LSUN Bedroom & 16.6 & 0.31 \\
        \midrule
        \multirow{4}{*}{Inceptionv3} & \multirow{4}{*}{2048} & Cifar10 & 9.0 & 0.22 \\
         & & FFHQ & 8.3 & 0.20 \\
         & & ImageNet & 5.7 & 0.20 \\
         & & LSUN Bedroom & 12.9 & 0.28 \\
        \midrule
        \multirow{4}{*}{DINOv2} & \multirow{4}{*}{1024} & Cifar10 & 6.5 & 0.14 \\
         & & FFHQ & 5.8 & 0.12 \\
         & & ImageNet & 4.4 & 0.10 \\
         & & LSUN Bedroom & 5.8 & 0.13 \\
        \midrule
        \multirow{4}{*}{DINOv3} & \multirow{4}{*}{4096} & Cifar10 & 9.0 & 0.24 \\
         & & FFHQ & 10.1 & 0.26 \\
         & & ImageNet & 8.0 & 0.20 \\
         & & LSUN Bedroom & 9.0 & 0.23 \\
        \midrule
        CLAP & 1024 & Free Music Archive & 4.9 & 0.09 \\
        \bottomrule
        \end{tabular}
      \end{sc}
    \end{small}
  \end{center}
  \vskip -0.1in
\end{table}

\subsection{Hubness Reduction Comparison}
\label{sec:hubness_reduction_comparison}

\Cref{tab:hubness_reduction} compares various hubness reduction methods on the datasets and embeddings from \Cref{tab:hubness_measures}, showing that ICDM constently outperforms other methods. See
\Cref{sec:hubness_reduction_full_results} for full results.

\begin{table}[h]
  \caption{\textbf{Hubness reduction comparison.} For each method, we report its rank and value for $A^5$ and $h_1^5(1\%)$, averaged over five datasets and embeddings from \Cref{tab:hubness_measures}.
  "Sphere" is projection onto the unit sphere. Lower values are better. ICDM is best in all $17$ cases. For methods needing a neighborhood size $K$, we took the $K$ in $\{5,10,20,50\}$ with lowest $h_1^5(1\%)$. With $h_1^5(1\%)$ below $2$ and $A^5$ near $0$, ICDM eliminates hubness effectively.
  }
  \label{tab:hubness_reduction}
  \begin{center}
    \begin{small}
      \begin{sc}
        \begin{tabular}{lrrrr}
        \toprule
          \multirow{2}{*}{Method} & \multicolumn{2}{c}{$h_1^5(1\%)$} & \multicolumn{2}{c}{$A^5$} \\
          \cmidrule(lr){2-3} \cmidrule(lr){4-5}
           & Rank & Value & Rank & Value \\
        \midrule
        Original & 7.8 & 8.8 & 7.9 & 0.20 \\
        Sphere & 6.9 & 6.0 & 7.1 & 0.12 \\
        $\text{MP}^\text{Gauss}$ & 5.7 & 4.5 & 5.9 & 0.08 \\
        LS & 5.2 & 4.1 & 5.1 & 0.05 \\
        CSLS & 3.7 & 3.2 & 4.1 & 0.04 \\
        DSL & 3.2 & 3.3 & 2.0 & 0.01 \\
        NICDM & 2.6 & 3.1 & 3.0 & 0.03 \\
        ICDM & 1.0 & 1.7 & 1.0 & 0.00 \\
        \bottomrule
        \end{tabular}
      \end{sc}
    \end{small}
  \end{center}
  \vskip -0.1in
\end{table}

\newpage
\section{Proofs}
\subsection{Proof of \Cref{prop:density_estimator}}
\label{sec:proof_density_estimator}

\textbf{\Cref{prop:density_estimator}.} \textit{\densityEstimatorProp}

\begin{proof}
  Let $\hat{p}_{\text{K-NN}}$ denote the $K$-nearest-neighbors density estimator, by definition \cite{loftsgaarden1965nonparametric}:
\begin{equation*}
  \hat{p}_{\text{K-NN}}(x_i) = \frac{K}{N V_d (\NND_K(x_i))^d}
\end{equation*}
where $\NND_K(x_i)$ is the distance from $x_i$ to its $K$-th nearest-neighbor, $N$ is the number of samples, and $V_d$ is the volume of the unit ball in $d$ dimensions. We have:
\begin{align*}
  \NND_k(x_i) &= \left(\frac{k}{N V_d \hat{p}_{\text{K-NN}}(x_i)}\right)^{1/d}\\
  \mu_i &= \frac{1}{K} \sum_{k=1}^K \left(\frac{k}{N V_d \hat{p}_{\text{k-NN}}(x_i)}\right)^{1/d}\\
  \mu_i &= \frac{1}{(N V_d)^{1/d}} \frac{1}{K} \sum_{k=1}^K k^{1/d} \hat{p}_{\text{k-NN}}(x_i)^{-1/d}
\end{align*}
Substituting this expression for $\mu_i$ into the definition of $\hat{p}_{\mu, K}(x_i)$ yields:
\begin{align*}
  \hat{p}_{\mu, K}(x_i) &= \frac{1}{N V_d} \left(\frac{1}{(N V_d)^{1/d}} \frac{1}{K} \sum_{k=1}^K k^{1/d} \hat{p}_{\text{k-NN}}(x_i)^{-1/d}\right)^{-d} \left(\frac{1}{K} \sum_{k=1}^K k^{1/d}\right)^d\\
  &= \left(\frac{\sum_{k=1}^K k^{1/d}}{\sum_{k=1}^K k^{1/d} \hat{p}_{\text{k-NN}}(x_i)^{-1/d}}\right)^{d}\\
  &= \left(\sum_{k=1}^K \underbrace{\frac{k^{1/d}}{\sum_{j=1}^K j^{1/d}}}_{w_{K, k}} \hat{p}_{\text{k-NN}}(x_i)^{-1/d}\right)^{-d}
\end{align*}

The weights $w_{K, k}$ satisfy the conditions of a regular summability method \cite{toeplitz1911allgemeine,tucciarone1973development}:
\begin{itemize}
  \item $w_{K, k} \xrightarrow{K \to \infty} 0$ for any fixed $k \in \mathbb{N}$, 
  \item are positive and $\sum_{k=1}^K w_{K,k} = 1$, so $\sum_{k=1}^K |w_{K,k}| = \sum_{k=1}^K w_{K,k} = 1$, and the absolute sum of the weights is bounded by a constant independent of $K$.
  \item $\sum_{k=1}^K w_{K,k}=1$, so $\sum_{k=1}^K w_{K,k} \xrightarrow{K \to \infty} 1$.
\end{itemize}
With $k(N)$ such that $k(N) \xrightarrow{N \rightarrow \infty} \infty$, and $k(N)/N \xrightarrow{N \rightarrow \infty} 0$, $\hat{p}_{\text{k-NN}}(x_i) \xrightarrow{P} p(x_i)$.

By the Silverman-Toeplitz theorem \cite{toeplitz1911allgemeine,tucciarone1973development}, which guarantees that regular summability methods preserve limits, we have $\sum_{k=1}^K \frac{k^{1/d}}{\sum_{j=1}^K j^{1/d}} \hat{p}_{\text{k-NN}}(x_i)^{-1/d} \xrightarrow{P} p(x_i)^{-1/d}$ as $K \to \infty$ and $K/N \to 0$.

Finally, because the power function is continuous, $\hat{p}_{\mu, K}(x_i) \xrightarrow{P} p(x_i)$.
\end{proof}

\subsection{Proof of \Cref{proposition:crossover_dimension}}
\label{sec:proof_crossover_dimension}

\textbf{\Cref{proposition:crossover_dimension}.} \textit{\crossoverDimensionProp}

\begin{proof}
  Let $X_1,\dots,X_N$ be i.i.d. samples from $\mathcal N(0,I_d)$. The squared distance to the center satisfies $\|X_1\|^2 \sim \chi^2_d$, so $\mathbb E\|X_1\|^2=d$.

  For a fixed $X_1$, $\forall j = 2, \dots, N$, $X_j - X_1 \mid X_1 \sim \mathcal{N}(-X_1, I_d)$.
  Conditionally on $\|X_1\|^2 = r$, the random variables $\{\|X_j - X_1\|^2\}_{j=2}^N$ are i.i.d., and we have $\|X_j - X_1\|^2 \mid \|X_1\|^2 = r \sim \chi^2_d(\lambda = r)$, a noncentral chi-square distribution with $d$ degrees of freedom and noncentrality parameter $\lambda$, with cumulative distribution function $F_{\chi^2_d(\lambda=r)}$.

  $\NND_k(X_1)^2$ is the $k$-th order statistic of these $N-1$ variables. Its c.d.f. is such that:
  \begin{equation*}
    \mathbb{P}\left(\NND_k(X_1)^2\leq t \mid \|X_1\|^2 = r\right) = \mathbb{P}\left(\mathrm{Bin}(N-1,F_{\chi^2_d(\lambda=r)}(t))\geq k\right)
  \end{equation*} (see e.g. \citet{mood1950introduction} section VI 5.1).

  So, after integration to remove the conditioning on $\|X_1\|^2 = r$, we have:
  \begin{equation*}
    \mathbb{P}\left(\NND_k(X_1)^2\leq t\right) = \int_0^\infty \mathbb{P}\left(\mathrm{Bin}(N-1,F_{\chi^2_d(\lambda=r)}(t))\geq k\right) f_{\chi^2_d}(r) dr
  \end{equation*}
  By definition, the median of $\NND_k(X_1)^2$ is the value $t$ for which $\Pr\left(\NND_k(X_1)^2\leq t\right) = \frac{1}{2}$. When $t=d$, this gives the equation in the proposition.
\end{proof}

\newpage
\section{Empirical Convergence of ICDM}
\label{sec:empirical_convergence}

GICDM relies on ICDM uniformizing the density of the real data manifold, i.e., $|\mu_i^r - \bar{\mu}^r| < \epsilon$ for all $i$ for a small $\epsilon$ after a finite number of iterations. In this section, we empirically validate this pointwise convergence across various datasets and embeddings.

We used 16 combinations of embedders (DINOv2, DINOv3, Inceptionv3, VGG16) and datasets (CIFAR-10, ImageNet, FFHQ, LSUN-bedroom) and monitored the relative difference between individual distance averages $\mu_i^r$ and their overall mean $\bar{\mu}^r$, using $K=10$ and $K=100$ (the two values used when evaluating metrics with $k=5$).

The results, shown in \Cref{fig:icdm_convergence}, demonstrate fast pointwise convergence.
Across all tested scenarios, the relative deviation values $\frac{|\mu_i^r - \bar{\mu}^r|}{\bar{\mu}^r}$ fall within the following ranges:
\begin{itemize}
    \item $[0.84, 1.51]$ after 1 iteration (equivalent to NICDM, the non-iterative version),
    \item $[0.985, 1.022]$ after 5 iterations,
    \item $[0.9989, 1.0017]$ after 10 iterations (the stopping criterion used in our experiments),
    \item $[0.99983, 1.00018]$ after 15 iterations,
    \item $[0.999966, 1.000036]$ after 20 iterations.
\end{itemize}

After 10 iterations, the maximum deviation of any individual $\mu_i^r$ from the mean is less than $0.17\%$ across all 16 diverse scenarios. Consequently, the pointwise convergence assumption holds robustly in practice, and stopping at $10$ ensures convergence for GICDM.

\begin{figure}[h]
  \begin{center}
    \includegraphics[width=0.6\textwidth]{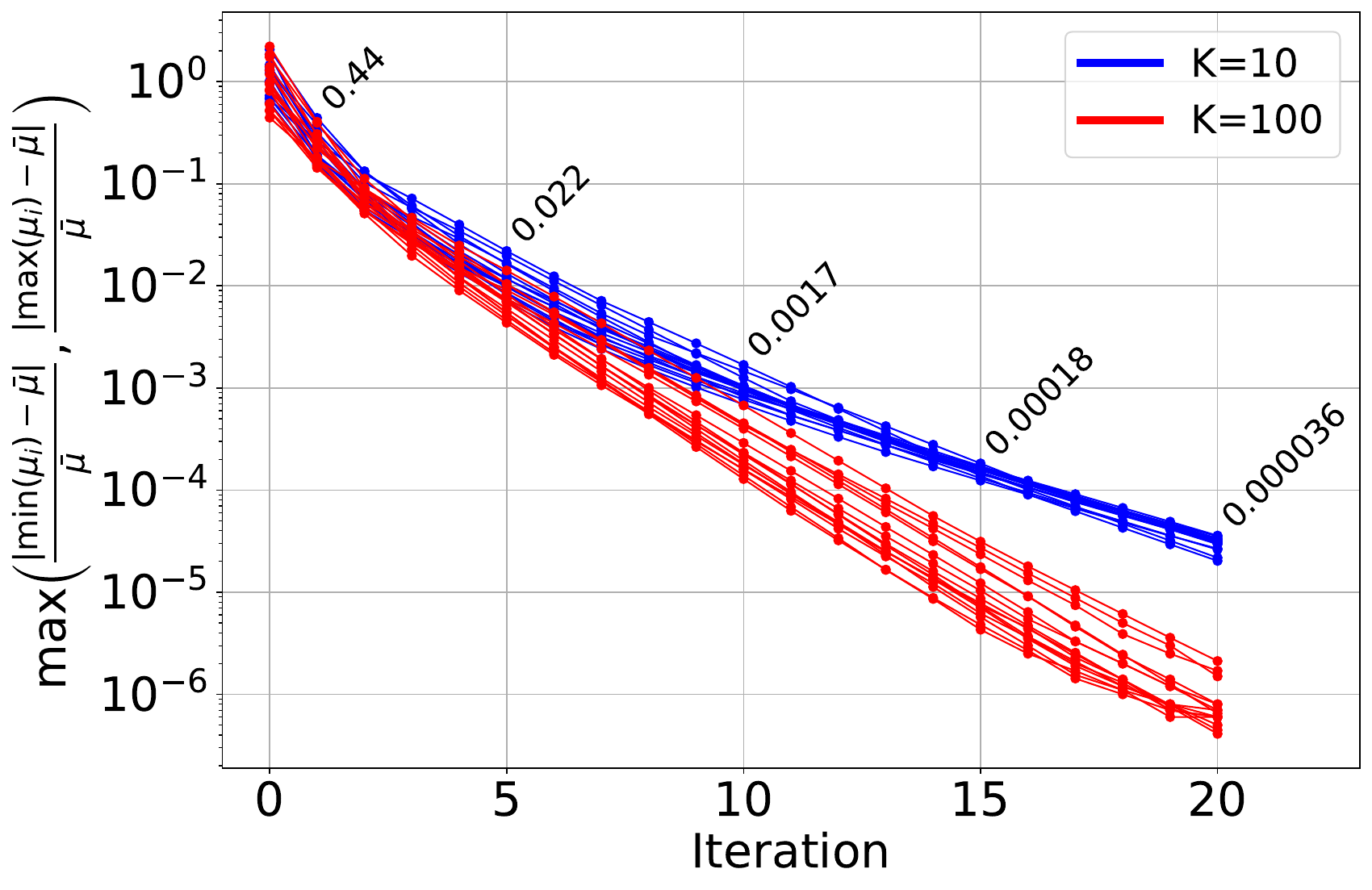}
    \caption{\textbf{Empirical convergence of ICDM.} Maximum relative deviation of individual neighborhood distance averages $\mu_i^r$ from their overall mean $\bar{\mu}^r$ across iterations, for 16 combinations of datasets and embeddings, with $K=10$ and $K=100$ (the two values used when evaluating metrics with $k=5$). Values from the above results at iterations 1, 5, 10, 15 and 20 show the biggest relative difference. After 10 iterations, the maximum deviation is less than $0.17\%$ across all scenarios, confirming fast pointwise convergence.}
    \label{fig:icdm_convergence}
  \end{center}
\end{figure}

\newpage
\section{Synthetic Benchmark}
\label{appendix:position_benchmark}

\begin{table}[b!]
  \caption{\textbf{Fidelity metrics benchmark with and without GICDM}. Tests are grouped by desiderata type: \emph{Purpose} checks whether the curve shape of the evaluated metric matches the intended test objective, while \emph{Bounds} evaluates whether the metric yields the expected values (e.g., 1 for perfect generated data). For each metric, T indicates a passed test, and \color{red}F \color{black} indicates failure. Adding GICDM improves Purpose results on several tests, such as "Gaussian Std. Deviation Difference" and "Hypersphere Surface", where the real data is fixed and the generated data varies in radius (variance for a Gaussian or radius for a sphere). GICDM improves Bounds results for the normalized metric Clipped Density.}
  \label{tab:pos_fidelity}
  \centering
  \begin{small}
  \begin{sc}
  \begin{tabular}{ll c| c|c c|c c|c c|c c|c c}
  \toprule
  Desiderata & Sanity Check & \makebox[6pt]{\rotatebox{90}{$\alpha$-Precision}} & \makebox[6pt]{\rotatebox{90}{Precision Cover}} & \makebox[6pt]{\rotatebox{90}{Precision}} & \makebox[6pt]{\rotatebox{90}{+ GICDM}} & \makebox[6pt]{\rotatebox{90}{Density}} & \makebox[6pt]{\rotatebox{90}{+ GICDM}} & \makebox[6pt]{\rotatebox{90}{symPrecision}} & \makebox[6pt]{\rotatebox{90}{+ GICDM}} & \makebox[6pt]{\rotatebox{90}{P-precision}} & \makebox[6pt]{\rotatebox{90}{+ GICDM}} & \makebox[6pt]{\rotatebox{90}{Clipped Density}} & \makebox[6pt]{\rotatebox{90}{+ GICDM}} \\
  \midrule
  \multirow[c]{14}{*}{Purpose} & Discrete Num. vs. Continuous Num. & \color{red} F & \color{red} F & \color{red} F & \color{red} F & \color{red} F & \color{red} F & \color{red} F & \color{red} F & \color{red} F & \color{red} F & \color{red} F & \color{red} F \\
  & Gaussian Mean Difference  & T & T & T & T & T & T & T & T & T & T & T & T \\
  & Gaussian Mean Difference + Outlier  & T & T & \color{red} F & T & T & T & \color{red} F & T & T & T & T & T \\
  & Gaussian Mean Difference + Pareto & T & T & T & T & T & T & T & T & T & T & T & T \\
  & Gaussian Std. Deviation Difference & \color{red} F & \color{red} F & \color{red} F & T & \color{red} F & T & \color{red} F & \color{red} F & \color{red} F & T & \color{red} F & T \\
  & Hypercube, Varying Sample Size & \color{red} F & \color{red} F & \color{red} F & \color{red} F & \color{red} F & \color{red} F & \color{red} F & \color{red} F & \color{red} F & \color{red} F & \color{red} F & \color{red} F \\
  & Hypercube, Varying Syn. Size & \color{red} F & \color{red} F & \color{red} F & \color{red} F & \color{red} F & \color{red} F & \color{red} F & \color{red} F & \color{red} F & \color{red} F & \color{red} F & \color{red} F \\
  & Hypersphere Surface & T & \color{red} F & \color{red} F & T & \color{red} F & T & T & T & \color{red} F & T & \color{red} F & T \\
  & Mode Collapse  & T & \color{red} F & \color{red} F & T & T & T & T & T & T & T & T & T \\
  & Mode Dropping + Invention & \color{red} F & \color{red} F & T & T & T & T & \color{red} F & \color{red} F & T & T & T & T \\
  & One Disjoint Dim. + Many Identical Dim.  & \color{red} F & \color{red} F & \color{red} F & \color{red} F & \color{red} F & \color{red} F & \color{red} F & \color{red} F & \color{red} F & \color{red} F & \color{red} F & \color{red} F \\
  & Sequential Mode Dropping  & \color{red} F & \color{red} F & T & T & T & T & \color{red} F & \color{red} F & T & T & T & T \\
  & Simultaneous Mode Dropping  & \color{red} F & \color{red} F & T & T & T & T & \color{red} F & \color{red} F & T & T & T & T \\
  & Sphere vs. Torus  & T & \color{red} F & T & T & T & T & T & T & T & T & T & T \\
  \midrule
  Hyperparam. & Hypercube, Varying Syn. Size  & T & \color{red} F & T & T & T & T & \color{red} F & \color{red} F & T & T & T & T \\
  \midrule
  Data & Hypercube, Varying Sample Size  & \color{red} F & \color{red} F & \color{red} F & \color{red} F & \color{red} F & \color{red} F & \color{red} F & \color{red} F & \color{red} F & \color{red} F & \color{red} F & \color{red} F \\
  \midrule
  \multirow[c]{13}{*}{Bounds} & Discrete Num. vs. Continuous Num.  & \color{red} F & \color{red} F & \color{red} F & \color{red} F & \color{red} F & \color{red} F & \color{red} F & \color{red} F & \color{red} F & \color{red} F & \color{red} F & \color{red} F \\
  & Gaussian Mean Difference & \color{red} F & T & \color{red} F & \color{red} F & T & \color{red} F & \color{red} F & \color{red} F & \color{red} F & \color{red} F & T & T \\
  & Gaussian Mean Difference + Outlier & \color{red} F & T & \color{red} F & \color{red} F & \color{red} F & \color{red} F & \color{red} F & \color{red} F & \color{red} F & \color{red} F & \color{red} F & T \\
  & Gaussian Mean Difference + Pareto & T & T & \color{red} F & \color{red} F & T & \color{red} F & T & \color{red} F & T & \color{red} F & T & T \\
  & Gaussian Std. Deviation Difference  & \color{red} F & \color{red} F & \color{red} F & \color{red} F & \color{red} F & \color{red} F & \color{red} F & \color{red} F & \color{red} F & \color{red} F & \color{red} F & T \\
  & Hypersphere Surface & \color{red} F & \color{red} F & \color{red} F & \color{red} F & \color{red} F & \color{red} F & T & \color{red} F & \color{red} F & \color{red} F & \color{red} F & T \\
  & Mode Collapse & \color{red} F & T & \color{red} F & \color{red} F & T & \color{red} F & \color{red} F & \color{red} F & \color{red} F & \color{red} F & T & T \\
  & Mode Dropping + Invention  & \color{red} F & \color{red} F & T & \color{red} F & T & \color{red} F & \color{red} F & \color{red} F & \color{red} F & \color{red} F & T & T \\
  & One Disjoint Dim. + Many Identical Dim. & \color{red} F & \color{red} F & \color{red} F & \color{red} F & \color{red} F & \color{red} F & \color{red} F & \color{red} F & \color{red} F & \color{red} F & \color{red} F & \color{red} F \\
  & Scaling One Dimension & T & T & T & T & T & T & T & T & T & T & T & T \\
  & Sequential Mode Dropping  & \color{red} F & \color{red} F & \color{red} F & \color{red} F & T & \color{red} F & \color{red} F & \color{red} F & \color{red} F & \color{red} F & T & T \\
  & Simultaneous Mode Dropping & \color{red} F & \color{red} F & \color{red} F & \color{red} F & T & \color{red} F & \color{red} F & \color{red} F & \color{red} F & \color{red} F & T & T \\
  & Sphere vs. Torus & \color{red} F & T & T & T & T & T & T & T & T & T & T & T \\
  \midrule
  Invariance & Scaling One Dimension & T & T & T & T & T & T & T & T & T & T & T & T \\
  \midrule
  \bottomrule
  \end{tabular}
  \end{sc}
  \end{small}
\end{table}

\begin{table}[t!]
    \caption{\textbf{Coverage metrics benchmark with and without GICDM}. This table, analogous to \Cref{tab:pos_fidelity}, reports results for coverage metrics. T indicates a passed test, while \color{red}F \color{black} indicates failure. For coverage metrics, L and H denote that the metric correctly identifies low (L) or high (H) coverage scenarios; either outcome is considered a pass. As with fidelity metrics, GICDM improves Purpose results on several tests, such as "Gaussian Std. Deviation Difference" and "Hypersphere Surface." GICDM also improves Bounds results for the normalized metric Clipped Coverage.}
    \label{tab:pos_coverage}
    \centering
    \begin{small}
    \begin{sc}
\begin{tabular}{ll c | c|c c|c c|c c|c c|c c}
  \toprule
  Desiderata & Sanity Check & \makebox[6pt]{\rotatebox{90}{$\beta$-Recall}} & \makebox[6pt]{\rotatebox{90}{Recall Cover}} & \makebox[6pt]{\rotatebox{90}{Recall}} & \makebox[6pt]{\rotatebox{90}{+ GICDM}} & \makebox[6pt]{\rotatebox{90}{Coverage}} & \makebox[6pt]{\rotatebox{90}{+ GICDM}} & \makebox[6pt]{\rotatebox{90}{symRecall}} & \makebox[6pt]{\rotatebox{90}{+ GICDM}} & \makebox[6pt]{\rotatebox{90}{P-recall}} & \makebox[6pt]{\rotatebox{90}{+ GICDM}} & \makebox[6pt]{\rotatebox{90}{Clipped Coverage}} & \makebox[6pt]{\rotatebox{90}{+ GICDM}} \\
  \midrule
  \multirow[c]{14}{*}{Purpose} & Discrete Num. vs. Continuous Num. & \color{red} F & \color{red} F & \color{red} F & \color{red} F & \color{red} F & \color{red} F & \color{red} F & \color{red} F & \color{red} F & \color{red} F & \color{red} F & \color{red} F \\
  & Gaussian Mean Difference  & T & T & T & T & T & T & T & T & T & T & T & T \\
  & Gaussian Mean Difference + Outlier    & T & T & \color{red} F & T & T & T & T & T & T & T & T & T \\
  & Gaussian Mean Difference + Pareto & T & T & T & T & T & T & T & T & T & T & T & T \\
  & Gaussian Std. Deviation Difference    & L & \color{red} F & \color{red} F & H & \color{red} F & L & L & L & \color{red} F & H & \color{red} F & L \\
  & Hypercube, Varying Sample Size    & \color{red} F & \color{red} F & \color{red} F & \color{red} F & \color{red} F & \color{red} F & \color{red} F & \color{red} F & \color{red} F & \color{red} F & \color{red} F & \color{red} F \\
  & Hypercube, Varying Syn. Size    & \color{red} F & \color{red} F & \color{red} F & \color{red} F & \color{red} F & \color{red} F & \color{red} F & \color{red} F & \color{red} F & \color{red} F & \color{red} F & \color{red} F \\
  & Hypersphere Surface    & T & \color{red} F & \color{red} F & T & \color{red} F & T & T & T & \color{red} F & T & \color{red} F & T \\
  & Mode Collapse   & \color{red} F & L & \color{red} F & \color{red} F & \color{red} F & \color{red} F & \color{red} F & \color{red} F & \color{red} F & \color{red} F & L & L \\
  & Mode Dropping + Invention    & \color{red} F & \color{red} F & H & H & \color{red} F & \color{red} F & \color{red} F & H & H & H & L & L \\
  & One Disjoint Dim. + Many Identical Dim.    & \color{red} F & \color{red} F & \color{red} F & \color{red} F & \color{red} F & \color{red} F & \color{red} F & \color{red} F & \color{red} F & \color{red} F & \color{red} F & \color{red} F \\
  & Sequential Mode Dropping    & \color{red} F & T & T & T & T & T & T & T & T & T & T & T \\
  & Simultaneous Mode Dropping    & \color{red} F & T & T & T & T & T & T & T & T & T & T & T \\
  & Sphere vs. Torus    & T & \color{red} F & \color{red} F & T & \color{red} F & T & \color{red} F & T & \color{red} F & T & T & T \\
  \midrule
  Hyperparam. & Hypercube, Varying Syn. Size    & \color{red} F & \color{red} F & \color{red} F & \color{red} F & \color{red} F & \color{red} F & \color{red} F & \color{red} F & \color{red} F & \color{red} F & \color{red} F & \color{red} F \\
  \midrule
  Data & Hypercube, Varying Sample Size   & \color{red} F & \color{red} F & \color{red} F & \color{red} F & \color{red} F & \color{red} F & \color{red} F & \color{red} F & \color{red} F & \color{red} F & \color{red} F & \color{red} F \\
  \midrule
  \multirow[c]{13}{*}{Bounds} & Discrete Num. vs. Continuous Num. & \color{red} F & \color{red} F & \color{red} F & \color{red} F & \color{red} F & \color{red} F & \color{red} F & \color{red} F & \color{red} F & \color{red} F & \color{red} F & \color{red} F \\
  & Gaussian Mean Difference   & \color{red} F & T & \color{red} F & \color{red} F & T & \color{red} F & \color{red} F & \color{red} F & \color{red} F & \color{red} F & T & T \\
  & Gaussian Mean Difference + Outlier    & \color{red} F & T & \color{red} F & \color{red} F & T & \color{red} F & \color{red} F & \color{red} F & \color{red} F & \color{red} F & T & T \\
  & Gaussian Mean Difference + Pareto & \color{red} F & T & \color{red} F & \color{red} F & T & \color{red} F & T & \color{red} F & T & \color{red} F & T & T \\
  & Gaussian Std. Deviation Difference    & \color{red} F & \color{red} F & \color{red} F & \color{red} F & \color{red} F & \color{red} F & \color{red} F & \color{red} F & \color{red} F & \color{red} F & \color{red} F & L \\
  & Hypersphere Surface    & \color{red} F & \color{red} F & \color{red} F & \color{red} F & \color{red} F & \color{red} F & T & \color{red} F & \color{red} F & \color{red} F & \color{red} F & T \\
  & Mode Collapse    & \color{red} F & T & \color{red} F & \color{red} F & T & \color{red} F & \color{red} F & \color{red} F & \color{red} F & \color{red} F & T & T \\
  & Mode Dropping + Invention    & \color{red} F & T & T & \color{red} F & T & \color{red} F & T & \color{red} F & \color{red} F & \color{red} F & T & T \\
  & One Disjoint Dim. + Many Identical Dim.    & T & \color{red} F & \color{red} F & \color{red} F & \color{red} F & \color{red} F & \color{red} F & \color{red} F & \color{red} F & \color{red} F & \color{red} F & \color{red} F \\
  & Scaling One Dimension    & T & T & \color{red} F & T & T & T & T & T & T & T & T & T \\
  & Sequential Mode Dropping    & \color{red} F & T & \color{red} F & \color{red} F & T & \color{red} F & \color{red} F & \color{red} F & \color{red} F & \color{red} F & T & T \\
  & Simultaneous Mode Dropping   & \color{red} F & T & \color{red} F & \color{red} F & T & \color{red} F & \color{red} F & \color{red} F & \color{red} F & \color{red} F & T & T \\
  & Sphere vs. Torus    & T & \color{red} F & T & T & \color{red} F & \color{red} F & \color{red} F & T & T & T & T & T \\
  \midrule
  Invariance & Scaling One Dimension    & T & T & \color{red} F & T & T & T & T & T & T & T & T & T \\
  \midrule
  \bottomrule
  \end{tabular}
  \end{sc}
  \end{small}
\end{table}

\subsection{Setup}

\citet{raisa2025position} recently proposed a synthetic benchmark to evaluate fidelity and coverage metrics for generative models, from which they concluded that no existing fidelity or coverage metric was fully satisfactory. In this section, we evaluate the impact of GICDM on this benchmark.

We apply several changes to the benchmark to tackle issues pointed out recently \cite{salvy2026enhanced}:
\begin{itemize}
  \item We used \texttt{export PYTHONHASHSEED=42} to ensure reproducibility across runs.
  \item For tests with constant sample size, we used $10000$ samples instead of $1000$ to remove instability issues. To mitigate partially the increased computational cost, we used $10$ or $11$ values per tests instead of $20$ or $51$. For Precision Cover and Recall Cover, we used the implementation from \citet{salvy2026enhanced} for faster computation.
  \item We corrected the success criterion for coverage metrics in the "Mode Dropping + Invention" test: "L" instead of failure for a decrease in coverage as invented modes are added.
\end{itemize}

We made one additional change to the success criteria, specifically for the "Gaussian Std. Deviation Difference" test. In this test, the real data is sampled from a standard Gaussian distribution, while the generated data is sampled from a Gaussian with the same mean but varying standard deviation. The test is conducted in dimensions $1$, $8$, and $64$.

In the original benchmark, the success criterion for fidelity metrics was to observe a high fidelity score for low standard deviation values (up to $1$), followed by a decrease in fidelity as the standard deviation of the generated data increased beyond $1$. This is intuitive in low dimensions, where a generated point from a distribution with lower variance than the real data will fall into denser regions of the real data distribution, resulting in high fidelity.

However, in high dimensions, Gaussian distributions resemble spheres with an empty center. As a result, a generated point from a Gaussian with lower standard deviation than the real data will be located in the empty center region, far from any real points, and should therefore have low fidelity. Consequently, we modified the success criterion for this test in dimension $64$ to match that of the "Hypersphere Surface" test, which is the analogous test using spheres instead of Gaussians. For simplicity, we removed the criterion for dimension $8$, as the ideal behavior in this intermediate dimension is not clear.

\subsection{Results}
We ran the benchmark on standard metrics as well as on distance-based metrics with GICDM correction. Results are shown in \Cref{tab:pos_fidelity,tab:pos_coverage}. T indicates a passed test, while \color{red}F \color{black} indicates failure. For coverage metrics, the benchmark distinguishes between high (H, support-based) diversity metrics and low (L, density-based) coverage metrics. Either is considered a success as long as the metric consistently exhibits one behavior: L or H.

GICDM consistently improves Purpose results (shape of the curve) on both the "Gaussian Std. Deviation Difference" and "Hypersphere Surface" tests.

The hypersphere test was already passed by symmetric metrics without GICDM, as these metrics were specifically designed for this scenario. However, symPrecision fails both with and without GICDM on the Gaussian std test, as it exhibits the behavior expected for spheres in both high and low dimensions (i.e., low fidelity for low std values, even in dimension 1). The symmetric approach does not work as soon as there is more than one mode, as shown in \Cref{fig:hypersphere_test}.

For the Bounds criteria, which evaluate whether metrics yield the exact expected values, we observe clear improvements for Clipped metrics. These metrics are normalized, either relative to the score of the real set or via a closed-form formula, and GICDM preserves this normalization effectively. In contrast, other metrics exhibit less robust normalization after GICDM correction, sometimes resulting in ideal values slightly below 1 (e.g., around 0.9). This suggests that normalization by the real set score, as implemented in Clipped Density, could also be advantageous for other metrics after applying GICDM.

\subsection{Remaining Failures}

Even for the top metrics with GICDM, some failures remain. We discuss these below.

\textit{Discrete Num. vs. Continuous Num.}: This test aims to detect when generated data is discrete rather than continuous, or vice versa. None of the metrics is designed to detect this failure mode.

\textit{One Disjoint Dim. + Many Identical Dim.}: The real data is sampled from a standard Gaussian distribution in $d$ dimensions, while the generated data is sampled from a Gaussian that matches the real one in $d-1$ dimensions but has variance $6$ times larger in one dimension. As noted by \citet{raisa2025position}, for metrics based on Euclidean distances, this difference is averaged out as $d$ increases.

\textit{Hypercube tests: Varying Sample Size and Varying Syn. Size}: In this test from \citet{cheema2023precision}, the real and synthetic distributions are uniform on $d$-dimensional hypercubes of side $1$, with overlapping volume $0.2$, and the sample sizes vary. To maintain a fixed overlapping volume of $0.2$, the distance $h$ between the corners of the two hypercubes must satisfy $(1 - h)^{d}=0.2$, i.e., $h = 1 - 0.2^{1/d}$. Thus, as observed by \citet{raisa2025position}, the distance between points in the real and synthetic hypercubes decreases as the dimension increases.

These failures arise from limitations of Euclidean distances. The selection of embedding space and distance metric is critical to representing data in a way that captures its underlying structure. For example, if variations along one dimension are semantically more important than others, using Euclidean distance without adjustment may yield suboptimal results.

\newpage
\section{Evaluation on Real Datasets}

\begin{figure*}[b!]
    \centering
    \includegraphics[width=0.95\textwidth]{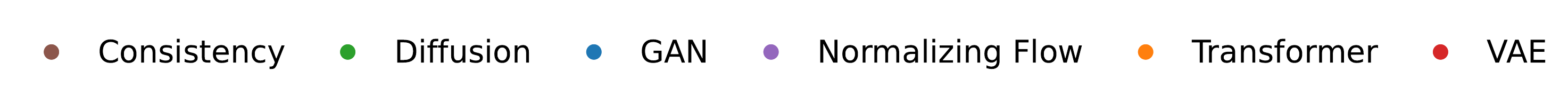}
    \vspace{-0.1cm}

    \centering
    \begin{subfigure}[b]{0.1\textwidth}
    \end{subfigure}
    \hfill
    \begin{subfigure}[b]{0.2\textwidth}
        \centering
        {CIFAR-10}\\
    \end{subfigure}
    \hfill
    \begin{subfigure}[b]{0.2\textwidth}
        \centering
        {ImageNet}\\
    \end{subfigure}
    \hfill
    \begin{subfigure}[b]{0.2\textwidth}
        \centering
        {LSUN Bedroom}\\
    \end{subfigure}
    \hfill
    \begin{subfigure}[b]{0.2\textwidth}
        \centering
        {FFHQ}\\
    \end{subfigure}

    \centering
    \begin{subfigure}[b]{0.01\textwidth}
    \end{subfigure}
    \hfill
    \begin{subfigure}[b]{0.234\textwidth}
        \centering
        \includegraphics[width=1.0\textwidth]{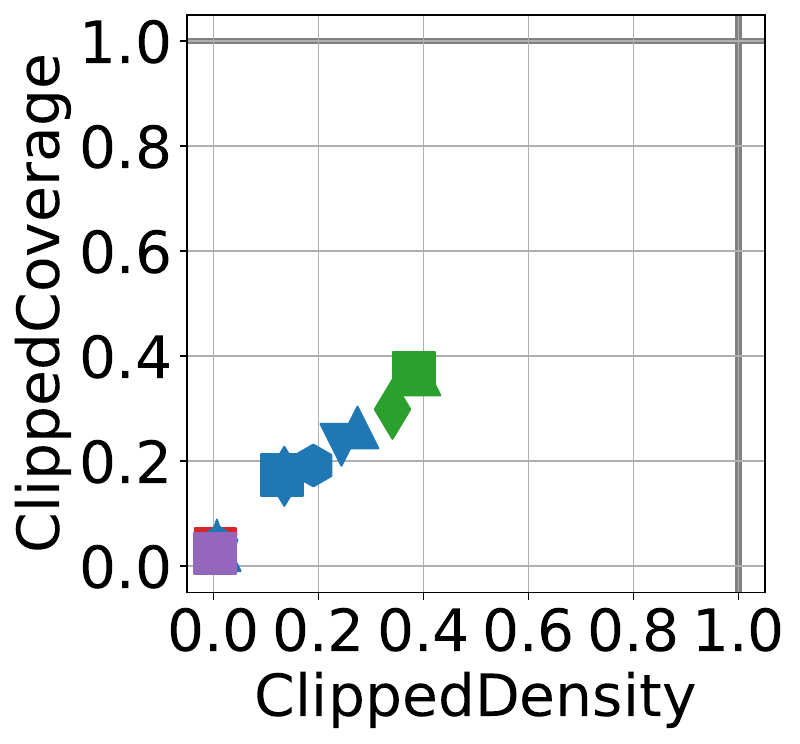}
    \end{subfigure}
    \hfill
    \begin{subfigure}[b]{0.223\textwidth}
        \centering
        \includegraphics[width=\textwidth]{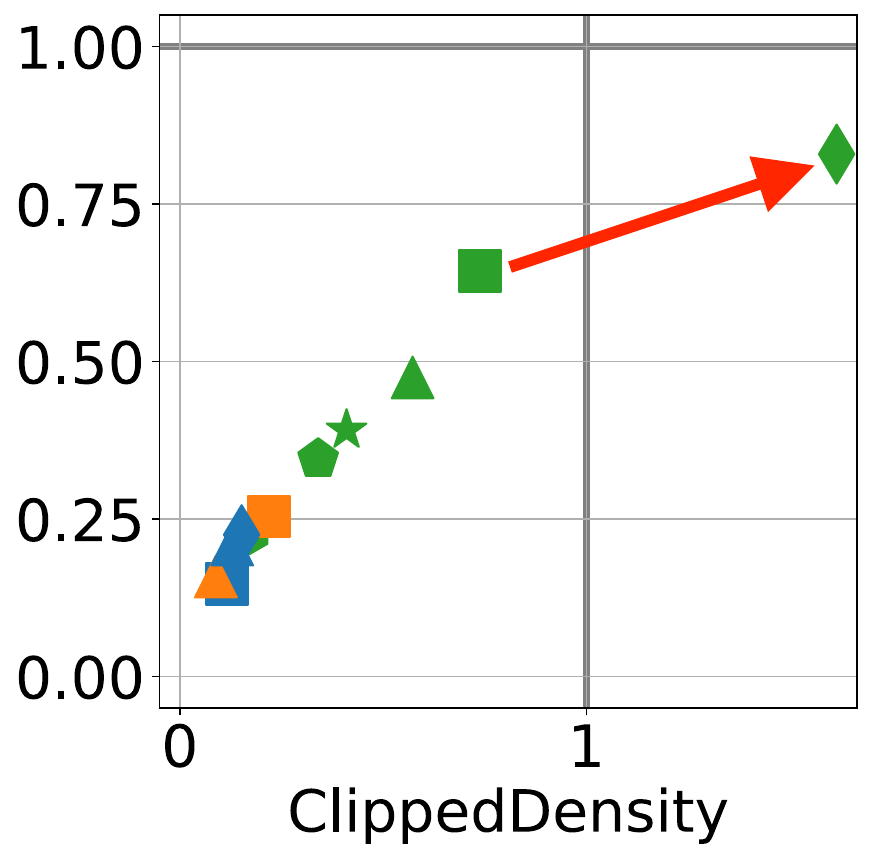}
    \end{subfigure}
    \hfill
    \begin{subfigure}[b]{0.217\textwidth}
        \centering
        \includegraphics[width=\textwidth]{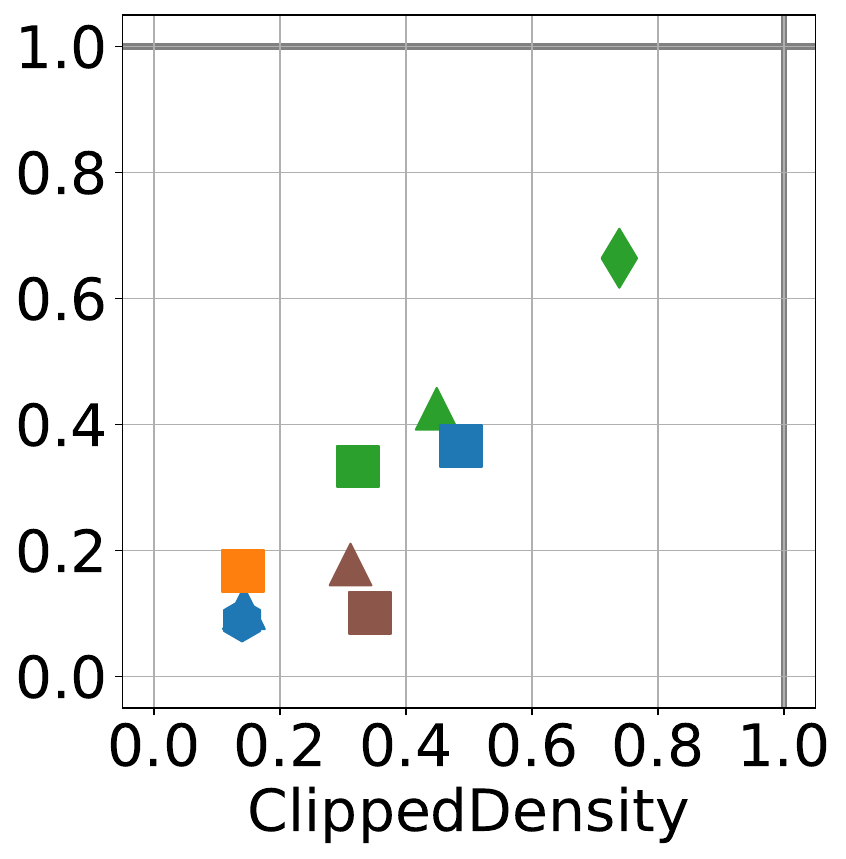}
    \end{subfigure}
    \hfill
    \begin{subfigure}[b]{0.218\textwidth}
        \centering
        \includegraphics[width=\textwidth]{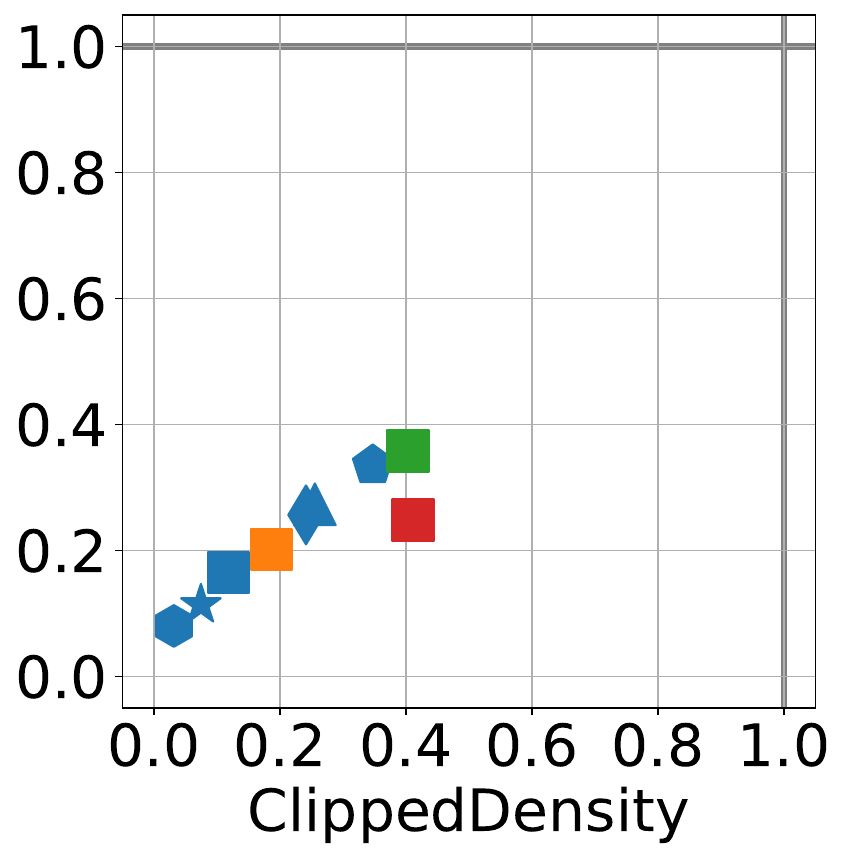}
    \end{subfigure}
        
    \centering
    \begin{subfigure}[b]{0.25\textwidth}
        \centering
        \includegraphics[width=\textwidth]{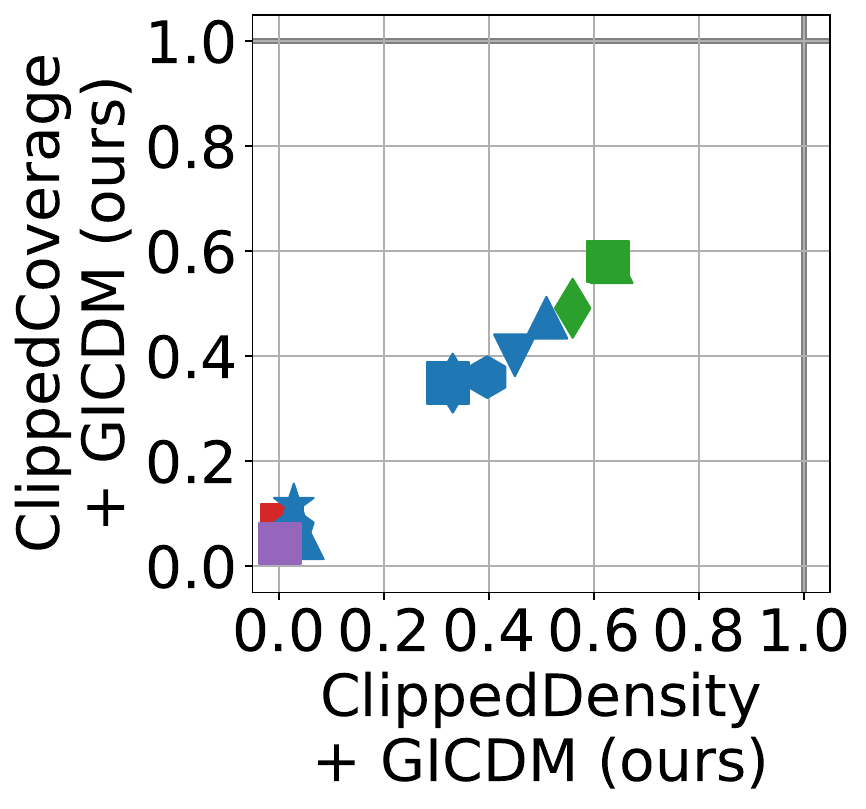}
    \end{subfigure}
    \hfill
    \begin{subfigure}[b]{0.23\textwidth}
        \centering
        \includegraphics[width=\textwidth]{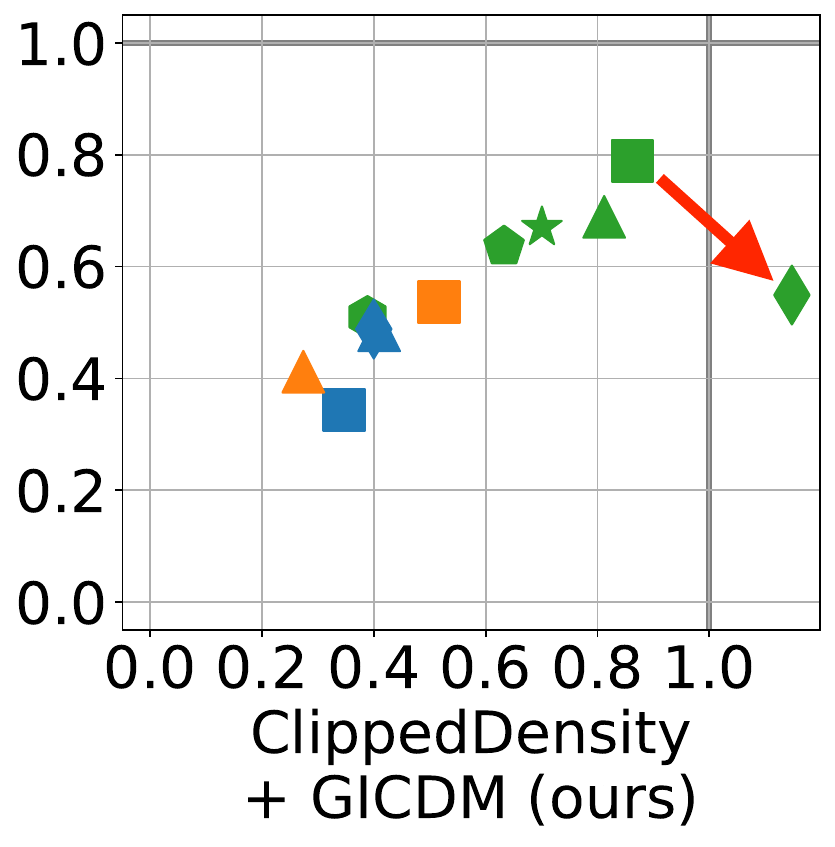}
    \end{subfigure}
    \hfill
    \begin{subfigure}[b]{0.213\textwidth}
        \centering
        \includegraphics[width=\textwidth]{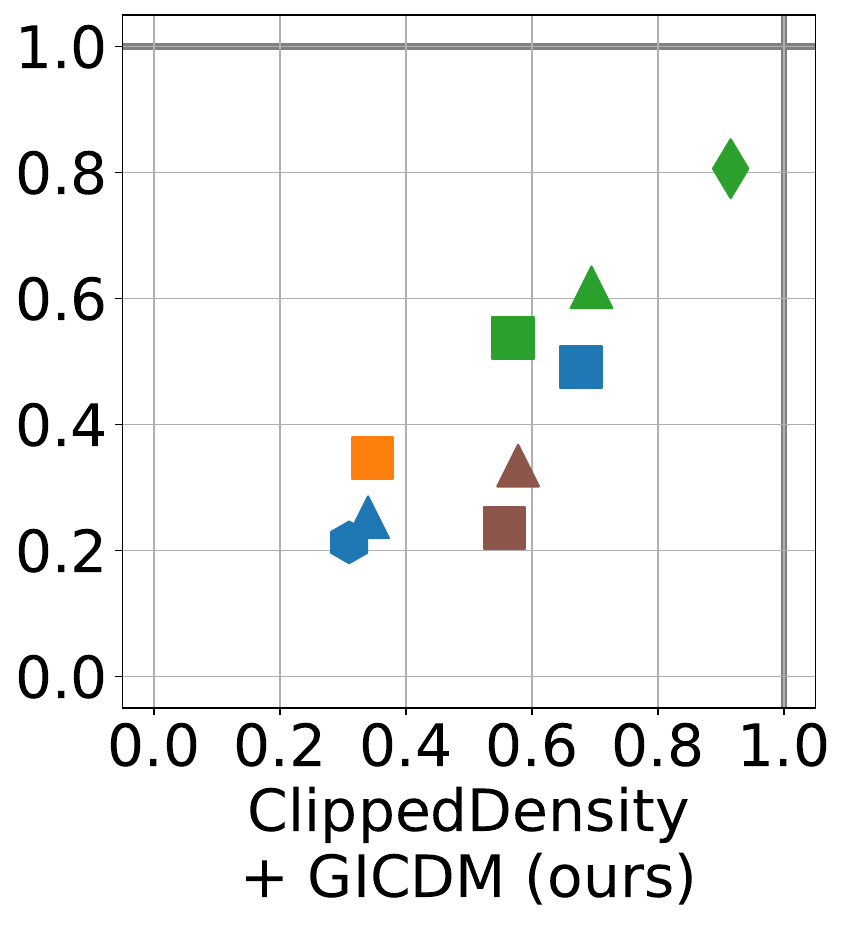}
    \end{subfigure}
    \hfill
    \begin{subfigure}[b]{0.214\textwidth}
        \centering
        \includegraphics[width=\textwidth]{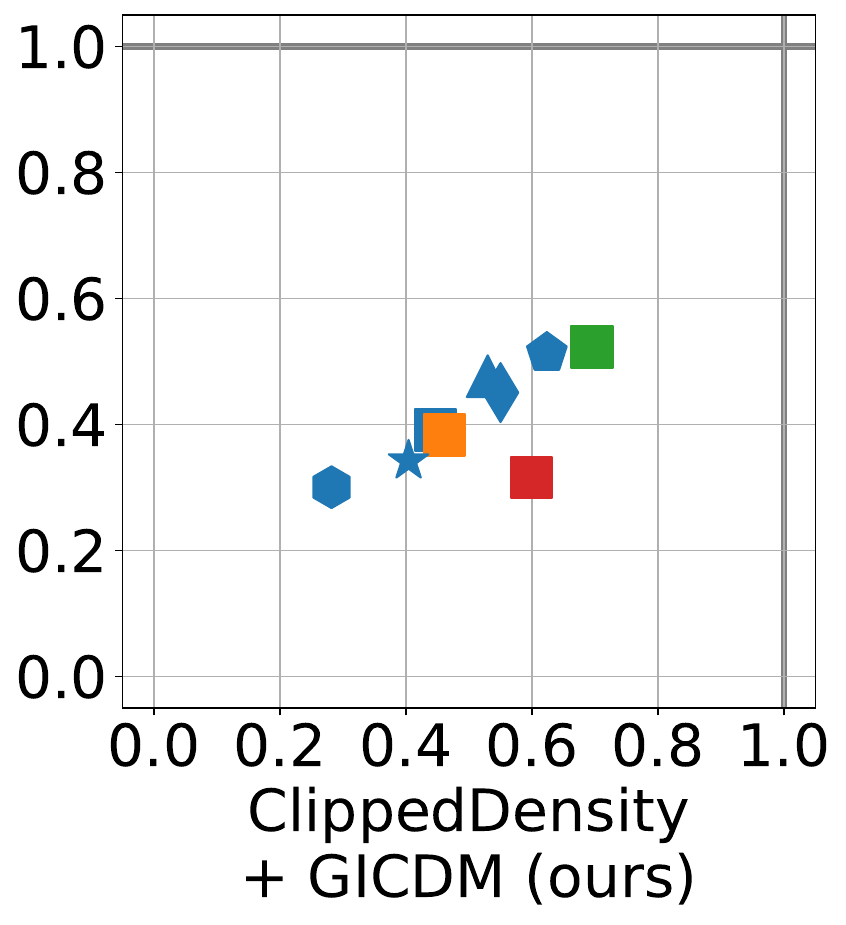}
    \end{subfigure}

    \begin{subfigure}[t]{0.01\textwidth}
    \end{subfigure}
    \hfill
    \begin{subfigure}[t]{0.229\textwidth}
        \centering
        \includegraphics[width=0.8\textwidth]{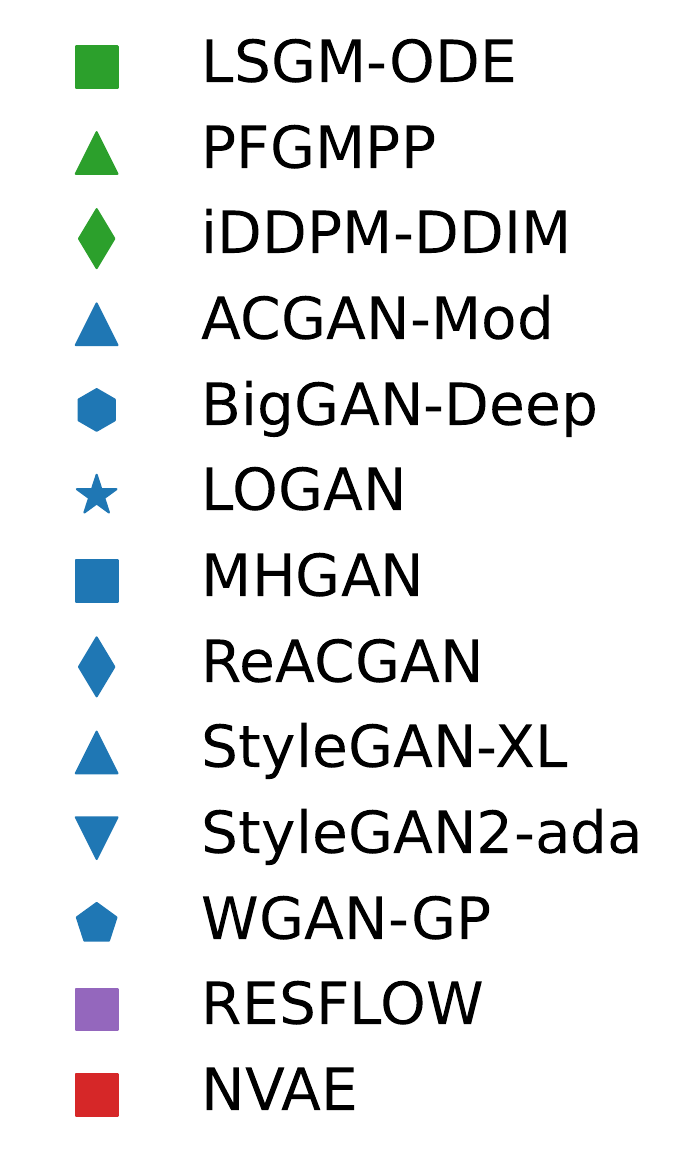}
    \end{subfigure}
    \hfill
    \begin{subfigure}[t]{0.24\textwidth}
       \centering
       \raisebox{15pt}[0pt][0pt]{\includegraphics[width=0.82\textwidth, trim={0 0 0 15}, clip]{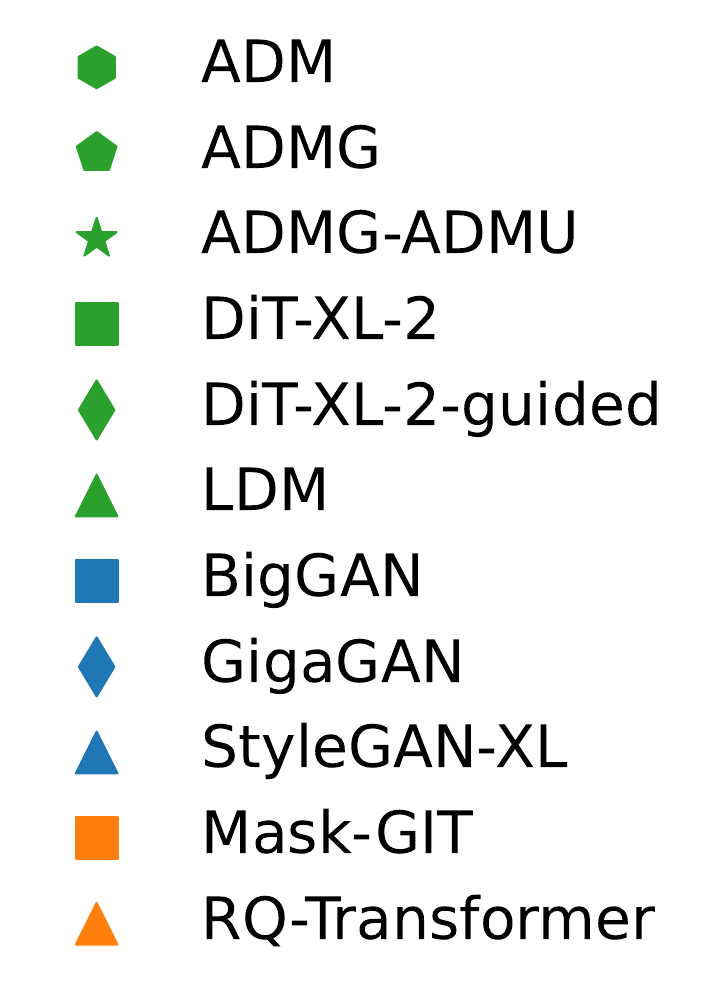}}
    \end{subfigure}   
    \hfill
    \begin{subfigure}[t]{0.24\textwidth}
        \centering
        \raisebox{29pt}[0pt][0pt]{\includegraphics[width=0.87\textwidth, trim={0 0 0 30}, clip]{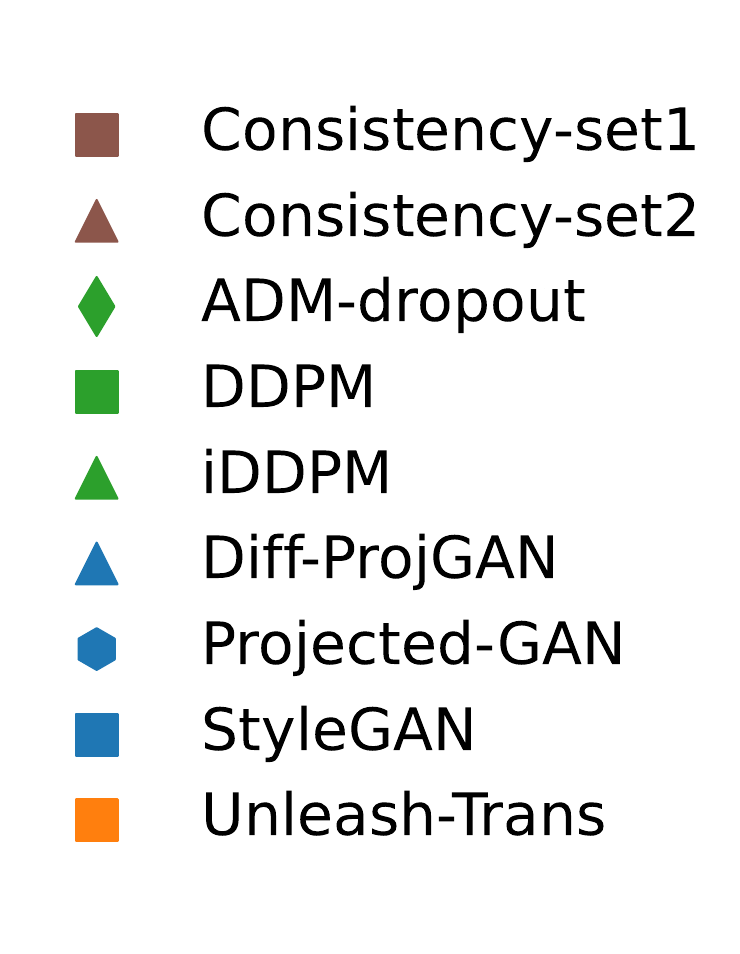}}
    \end{subfigure}
    \hfill
    \begin{subfigure}[t]{0.24\textwidth}
        \centering
        \raisebox{30pt}[0pt][0pt]{\includegraphics[width=0.80\textwidth, trim={0 0 0 30}, clip]{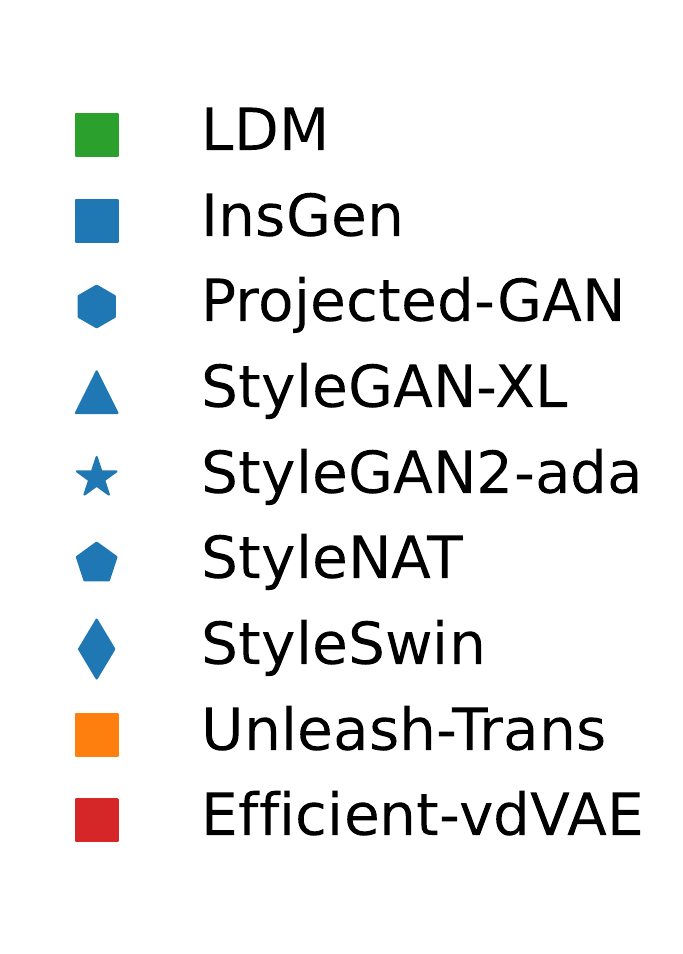}}
    \end{subfigure} 
    \caption{\textbf{Fidelity vs Coverage in DINOv2 embeddings}: Clipped Density vs Clipped Coverage (top row) and Clipped Density + GICDM vs Clipped Coverage + GICDM (bottom row) for CIFAR-10, ImageNet, LSUN Bedroom, and FFHQ datasets. Each point represents a generative model. For DiT-XL-2 on ImageNet, in the top row, classifier-free guidance increases both fidelity and coverage, whereas with GICDM (bottom row), it increases fidelity but decreases coverage, as expected (red arrows).
    }
    \label{fig:dino_v2}
\end{figure*}

\begin{figure*}[t!]
    \centering
    \includegraphics[width=0.95\textwidth]{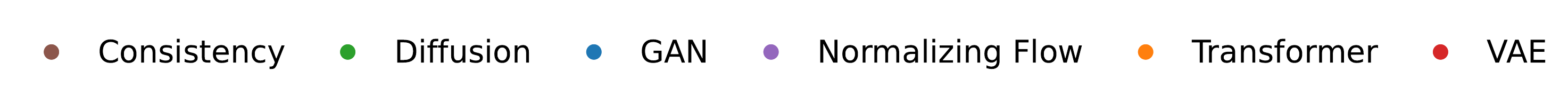}
    \vspace{-0.1cm}

    \centering
    \begin{subfigure}[b]{0.1\textwidth}
    \end{subfigure}
    \hfill
    \begin{subfigure}[b]{0.2\textwidth}
        \centering
        {CIFAR-10}\\
    \end{subfigure}
    \hfill
    \begin{subfigure}[b]{0.2\textwidth}
        \centering
        {ImageNet}\\
    \end{subfigure}
    \hfill
    \begin{subfigure}[b]{0.2\textwidth}
        \centering
        {LSUN Bedroom}\\
    \end{subfigure}
    \hfill
    \begin{subfigure}[b]{0.2\textwidth}
        \centering
        {FFHQ}\\
    \end{subfigure}
        
    \centering
    \begin{subfigure}[b]{0.01\textwidth}
    \end{subfigure}
    \hfill
    \begin{subfigure}[b]{0.234\textwidth}
        \centering
        \includegraphics[width=1.0\textwidth]{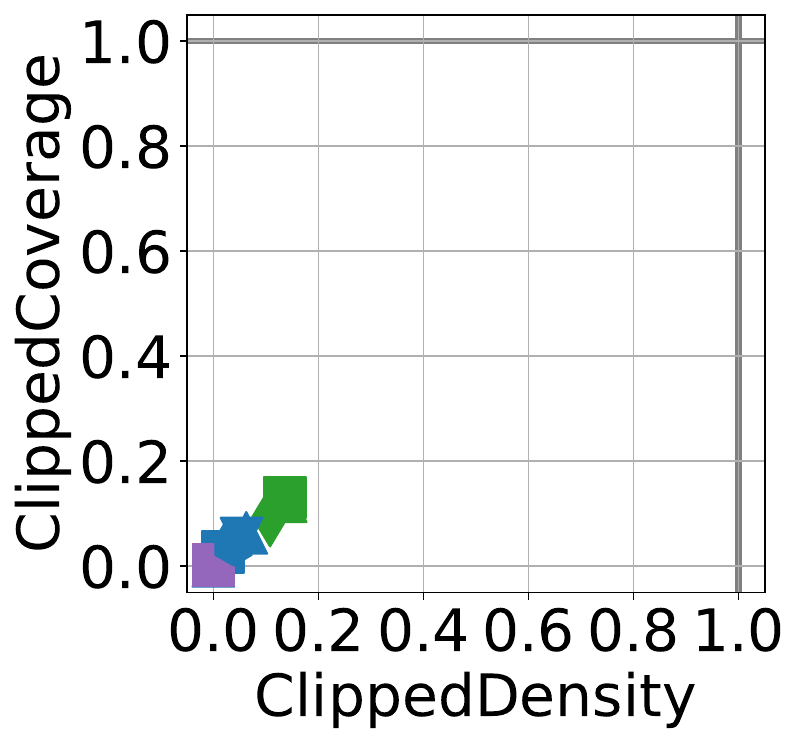}
    \end{subfigure}
    \hfill
    \begin{subfigure}[b]{0.214\textwidth}
        \centering
        \includegraphics[width=\textwidth]{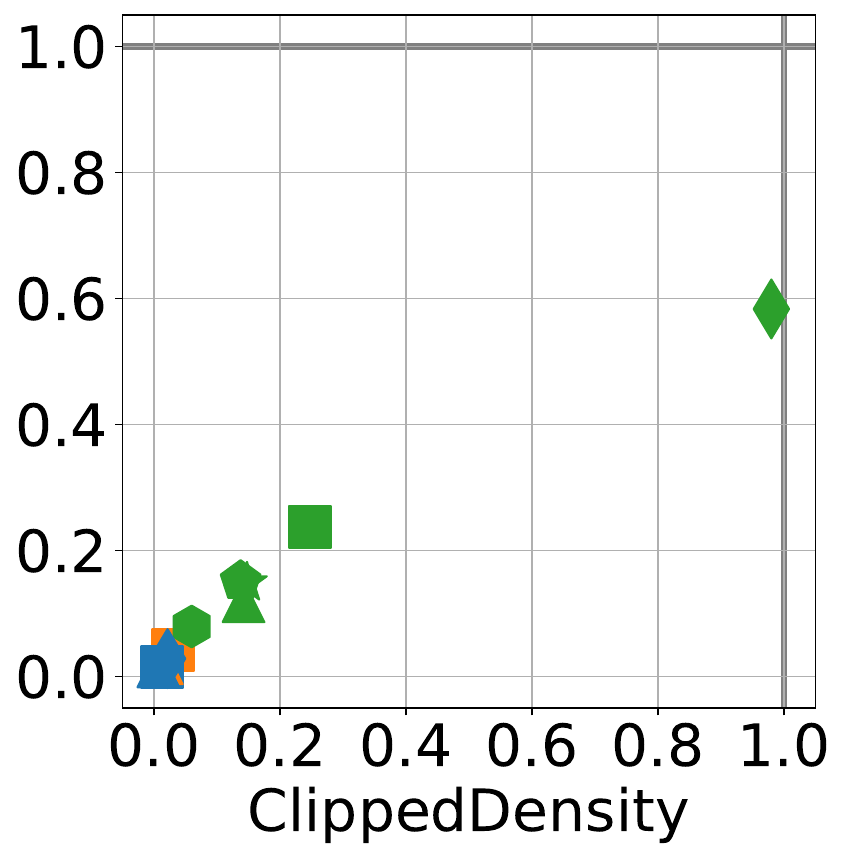}
    \end{subfigure}
    \hfill
    \begin{subfigure}[b]{0.215\textwidth}
        \centering
        \includegraphics[width=\textwidth]{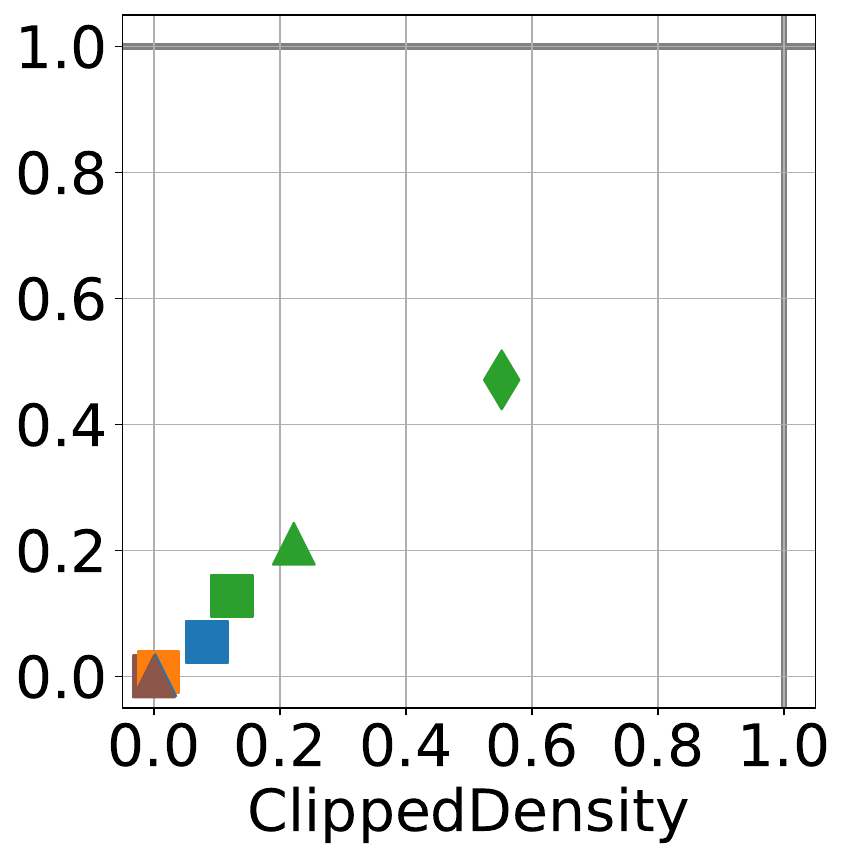}
    \end{subfigure}
    \hfill
    \begin{subfigure}[b]{0.218\textwidth}
        \centering
        \includegraphics[width=\textwidth]{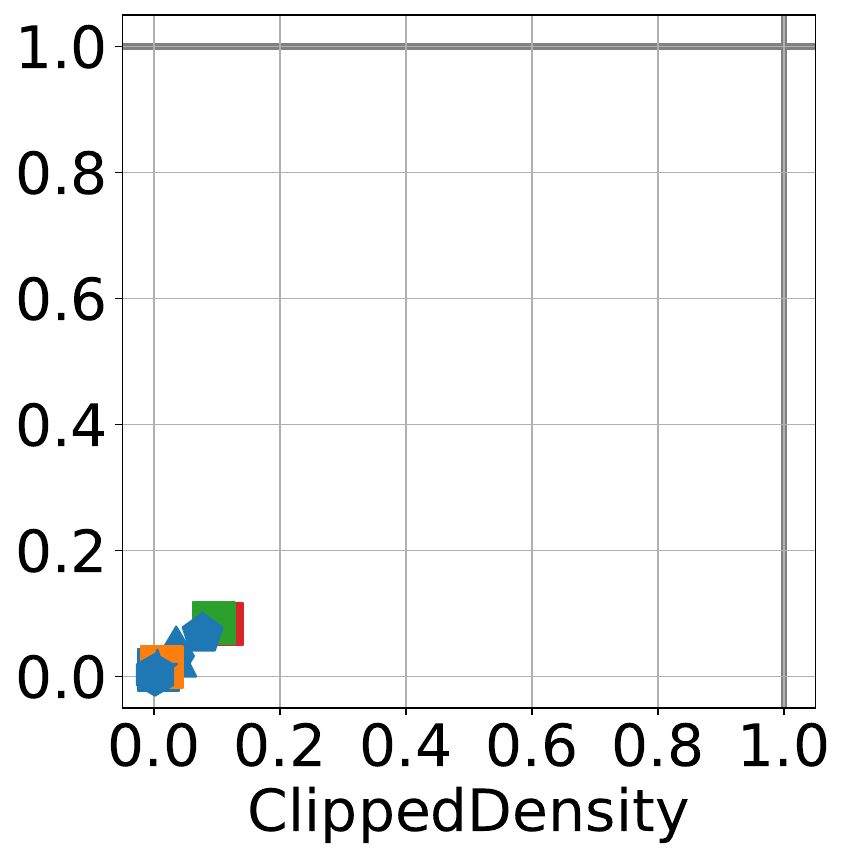}
    \end{subfigure}
        
    \centering
    \begin{subfigure}[b]{0.25\textwidth}
        \centering
        \includegraphics[width=\textwidth]{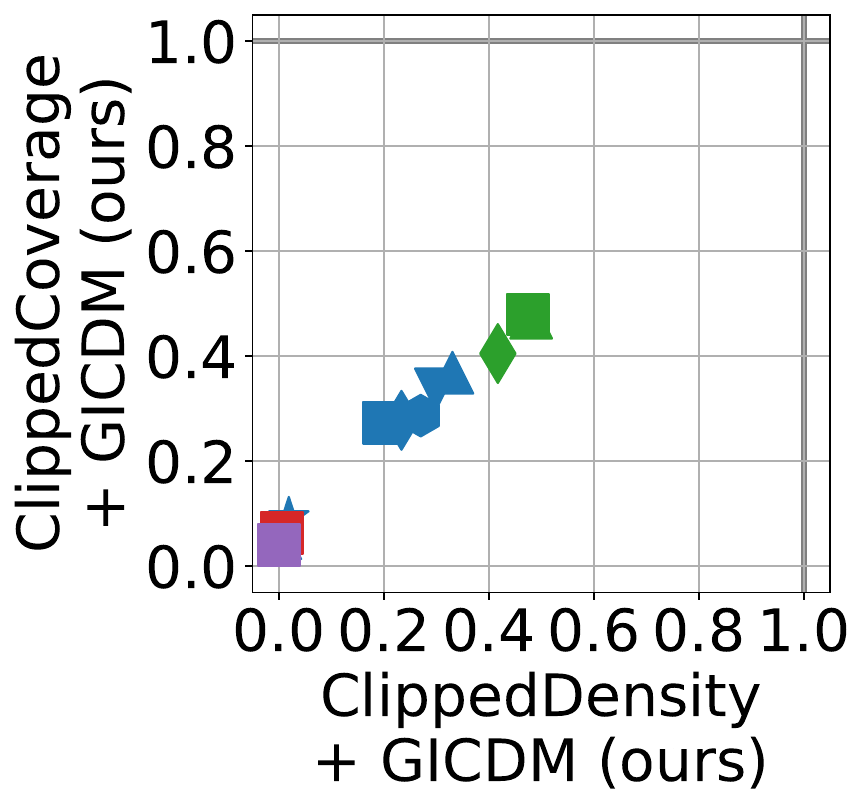}
    \end{subfigure}
    \hfill
    \begin{subfigure}[b]{0.213\textwidth}
        \centering
        \includegraphics[width=\textwidth]{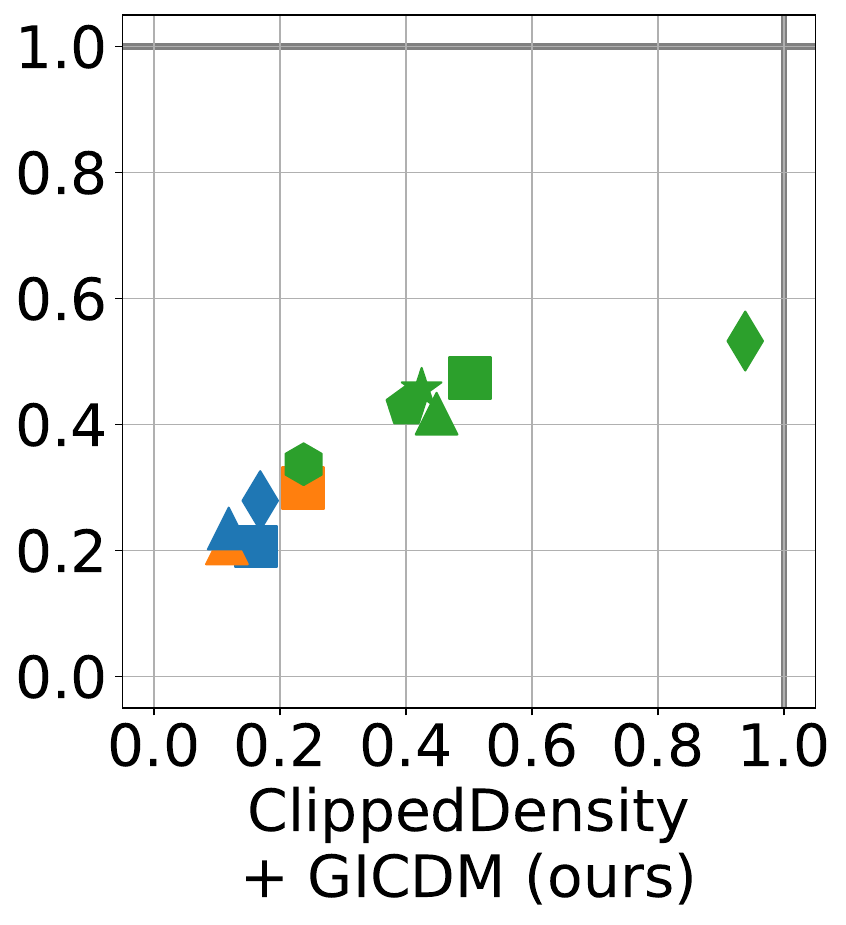}
    \end{subfigure}
    \hfill
    \begin{subfigure}[b]{0.213\textwidth}
        \centering
        \includegraphics[width=\textwidth]{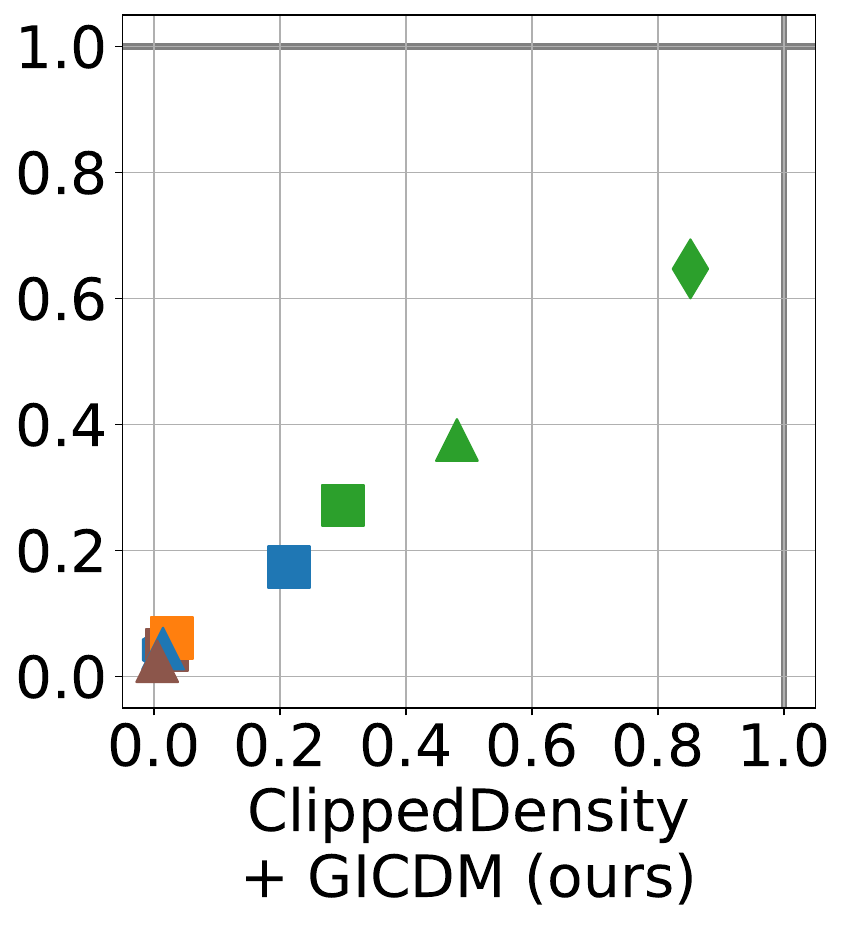}
    \end{subfigure}
    \hfill
    \begin{subfigure}[b]{0.214\textwidth}
        \centering
        \includegraphics[width=\textwidth]{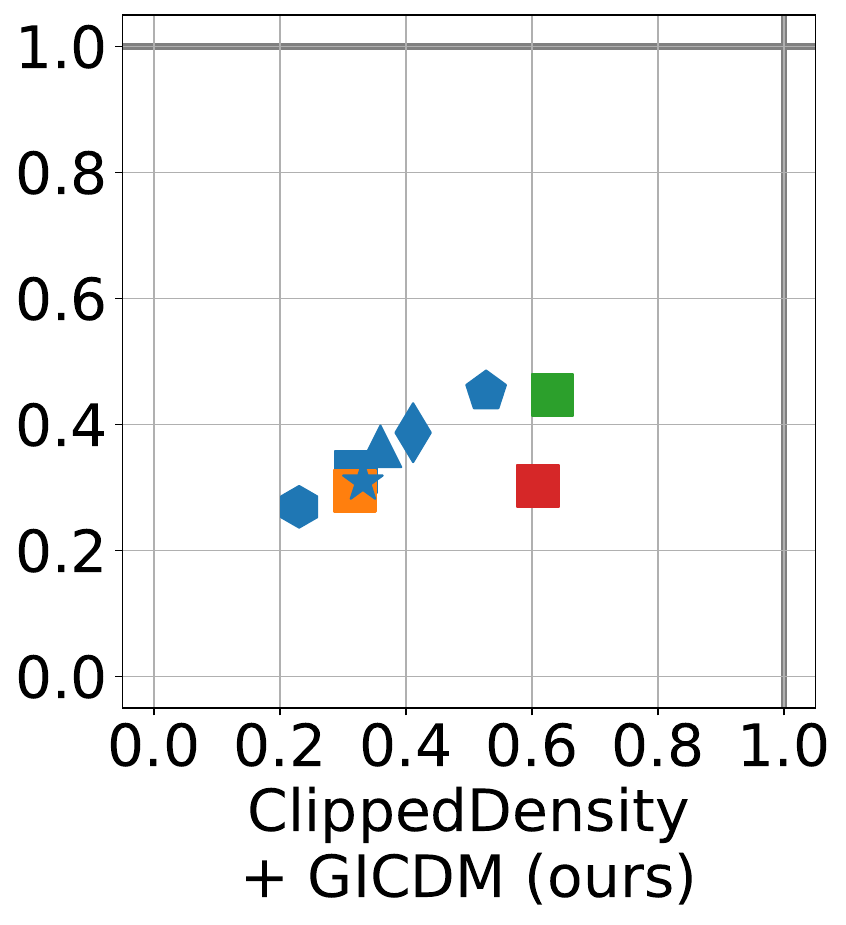}
    \end{subfigure}

    \begin{subfigure}[t]{0.01\textwidth}
    \end{subfigure}
    \hfill
    \begin{subfigure}[t]{0.229\textwidth}
        \centering
        \includegraphics[width=0.8\textwidth]{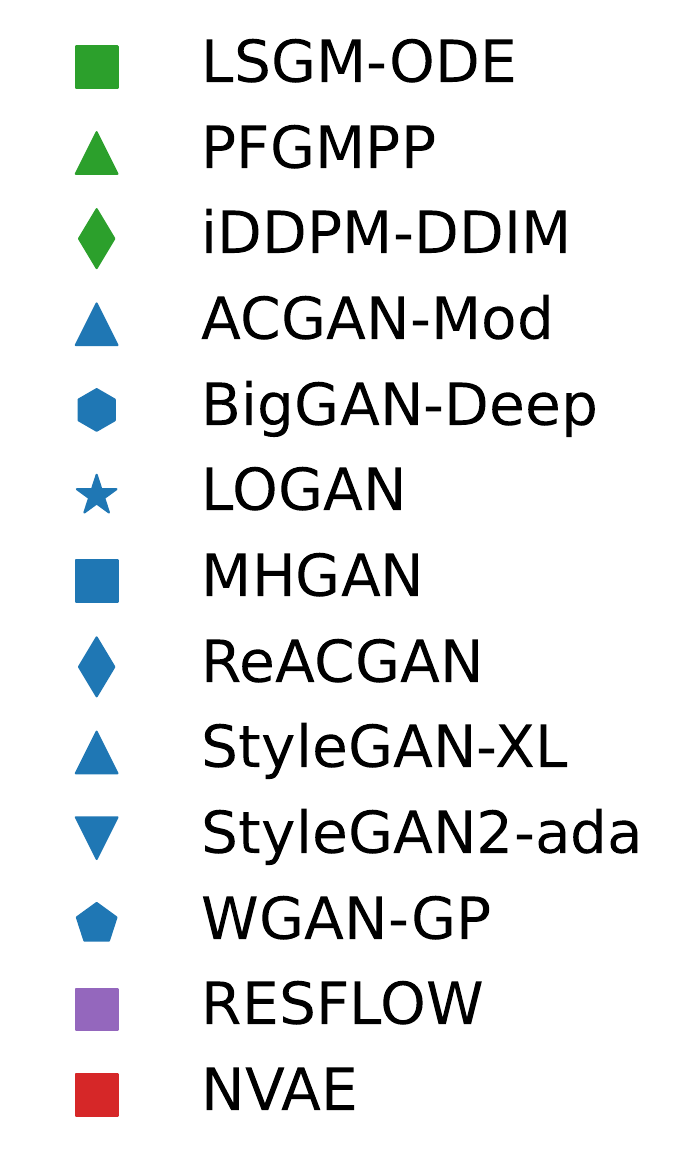}
    \end{subfigure}
    \hfill
    \begin{subfigure}[t]{0.24\textwidth}
       \centering
       \raisebox{15pt}[0pt][0pt]{\includegraphics[width=0.82\textwidth, trim={0 0 0 15}, clip]{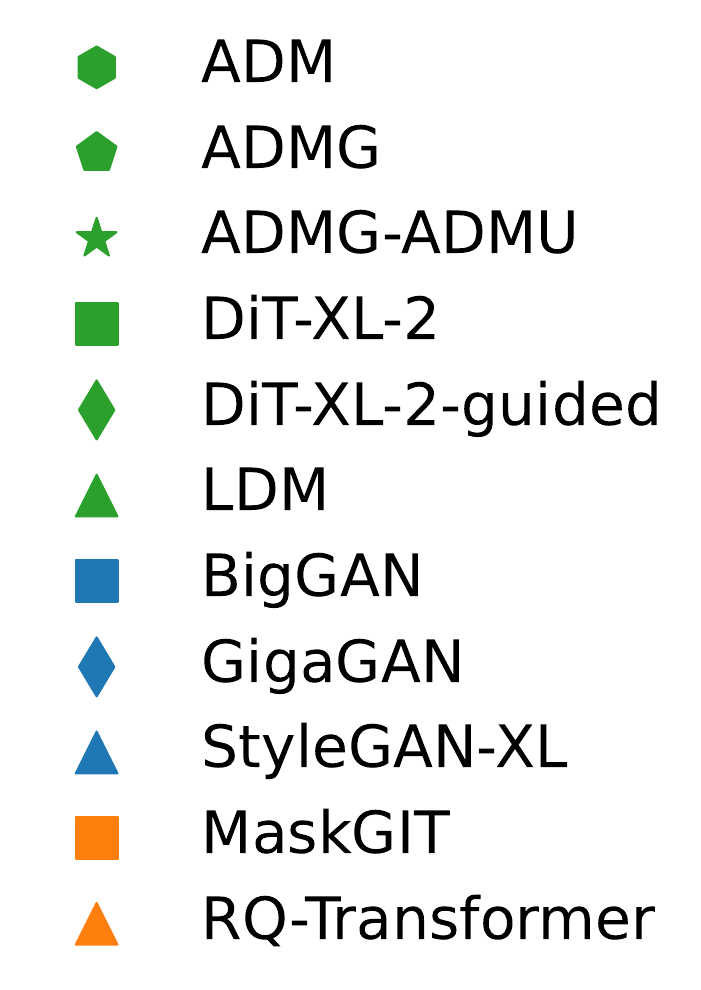}}
    \end{subfigure}   
    \hfill
    \begin{subfigure}[t]{0.24\textwidth}
        \centering
        \raisebox{29pt}[0pt][0pt]{\includegraphics[width=0.87\textwidth, trim={0 0 0 30}, clip]{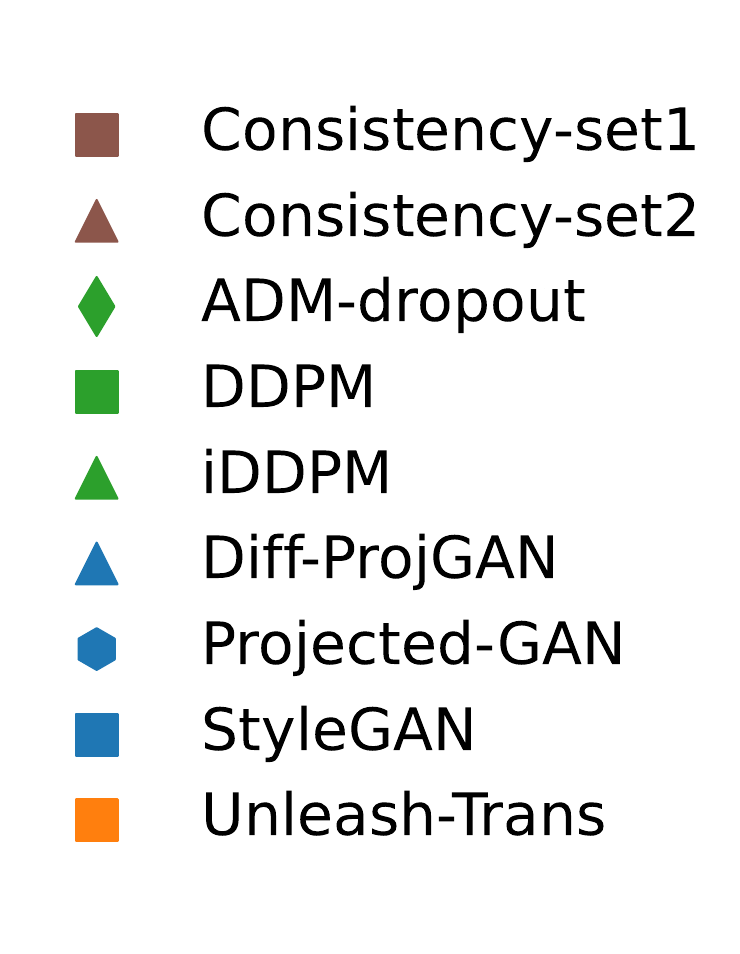}}
    \end{subfigure}
    \hfill
    \begin{subfigure}[t]{0.24\textwidth}
        \centering
        \raisebox{30pt}[0pt][0pt]{\includegraphics[width=0.80\textwidth, trim={0 0 0 30}, clip]{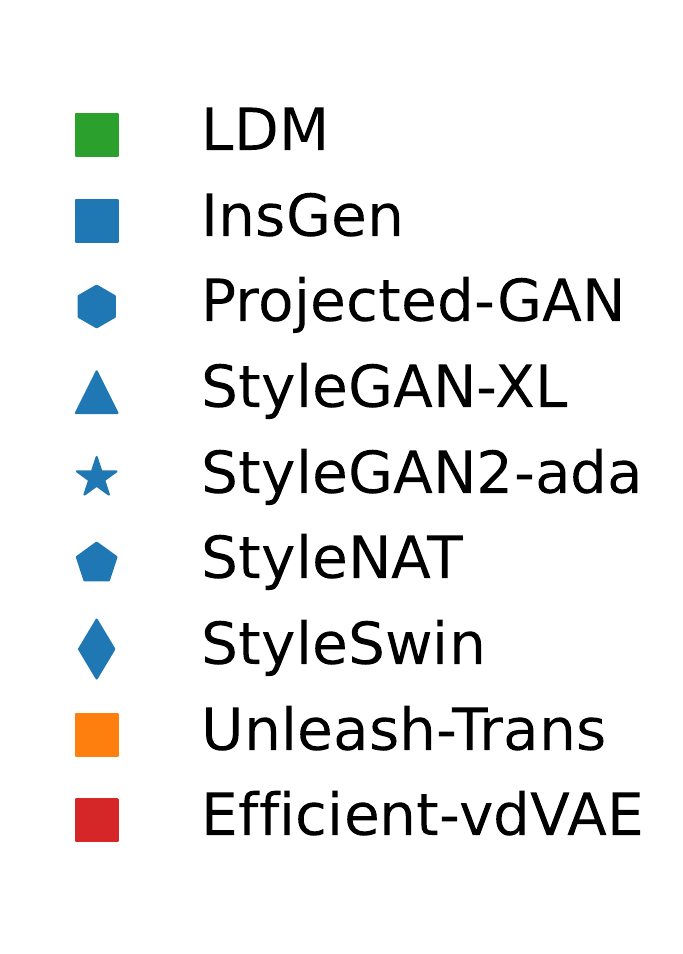}}
    \end{subfigure} 
    \caption{\textbf{Fidelity vs Coverage in DINOv3 embeddings}: Same setup as \Cref{fig:dino_v2}, but using DINOv3 embeddings. GICDM tends to increase the spread of points.}
    \label{fig:dino_v3}
\end{figure*}

\subsection{Data and Setup}

We evaluated metrics on generated datasets publicly released by \citet{stein2024exposing}, using $50000$ real and $50000$ generated samples for each evaluation, balanced across classes when applicable.

For CIFAR-10, generated data includes samples from: LSGM-ODE \cite{vahdat2021score}, PFGM++ (PFGMPP) \cite{xu2023pfgm++}, iDDPM-DDIM \cite{nichol2021improved}, StudioGAN models \cite{kang2023studiogan} (ACGAN-Mod \cite{odena2017conditional}, BigGAN \cite{brock2018large}, LOGAN \cite{wu2019logan}, MHGAN \cite{turner2019metropolis}, ReACGAN \cite{kang2021rebooting}, WGAN-GP \cite{gulrajani2017improved}), StyleGAN-XL \cite{sauer2022stylegan}, StyleGAN2-ada \cite{karras2020training}, RESFLOW \cite{chen2019residual}, and NVAE \cite{vahdat2020nvae}.

For ImageNet, $256\times256$ rescaled images were provided by \citet{stein2024exposing}, generated by: ADM \cite{dhariwal2021diffusion}, ADMG \cite{dhariwal2021diffusion}, ADMG-ADMU \cite{dhariwal2021diffusion}, BigGAN \cite{brock2018large}, DiT-XL-2 \cite{peebles2023scalable}, DiT-XL-2-guided \cite{peebles2023scalable}, LDM \cite{rombach2022high}, GigaGAN \cite{kang2023scaling}, StyleGAN-XL \cite{sauer2022stylegan}, Mask-GIT \cite{chang2022maskgit}, and RQ-Transformer \cite{lee2022autoregressive}.

For LSUN Bedroom \cite{yu2015lsun}, generated data includes: ADM-dropout \cite{dhariwal2021diffusion}, DDPM \cite{ho2020denoising}, iDDPM \cite{nichol2021improved}, StyleGAN \cite{karras2019style}, Diffusion-Projected GAN \cite{wang2022diffusion}, Projected GAN \cite{sauer2021projected}, Unleashing Transformers \cite{bond2022unleashing}, and two Consistency sets \cite{song2023consistency}.

For FFHQ \cite{kazemi2014one}, $256 \times 256$ downsampled images were provided by \citet{stein2024exposing}, generated by: LDM \cite{rombach2022high}, InsGen \cite{yang2021data}, Projected-GAN \cite{sauer2021projected}, StyleGAN-XL \cite{sauer2022stylegan}, StyleGAN2-ada \cite{karras2020training}, StyleNAT \cite{walton2022stylenat}, StyleSwin \cite{zhang2022styleswin}, Unleashing Transformers \cite{bond2022unleashing}, and Efficient-vdVAE \cite{hazami2022efficientvdvae}.

For consistency with prior work, we set $k=5$ for all metrics,
with default parameters elsewhere.

\subsection{Results}

\Cref{fig:dino_v2,fig:dino_v3} present fidelity versus coverage plots using Clipped Density and Clipped Coverage, both with and without GICDM, for DINOv2 and DINOv3 embeddings, respectively.

In both embeddings, applying GICDM tends to increase the spread of points. This effect is especially pronounced in DINOv3 embeddings, which are more affected by hubness.

For ImageNet with DINOv2 embeddings, adding guidance to DiT-XL-2 increases both fidelity and coverage in the Clipped measures without GICDM, which is counter-intuitive. In contrast, with GICDM, adding guidance increases fidelity but decreases coverage, as expected.

\subsection{Correlation with human scores}
\label{sec:human_correlation}

\begin{figure*}[b!]
    \centering
    \includegraphics[width=0.95\textwidth]{figures/fig6/DINOv2/FFHQ256/legend_categories.pdf}
    \vspace{-0.1cm}

    \centering
    \begin{subfigure}[b]{0.01\textwidth}
    \end{subfigure}
    \hfill
    \begin{subfigure}[b]{0.24\textwidth}
        \centering
        {CIFAR-10}\\
    \end{subfigure}
    \hfill
    \begin{subfigure}[b]{0.24\textwidth}
        \centering
        {ImageNet}\\
    \end{subfigure}
    \hfill
    \begin{subfigure}[b]{0.24\textwidth}
        \centering
        {LSUN Bedroom}\\
    \end{subfigure}
    \hfill
    \begin{subfigure}[b]{0.24\textwidth}
        \centering
        {FFHQ}\\
    \end{subfigure}

    \centering
    \begin{subfigure}[b]{0.23\textwidth}
        \centering
        \includegraphics[width=1.0\textwidth]{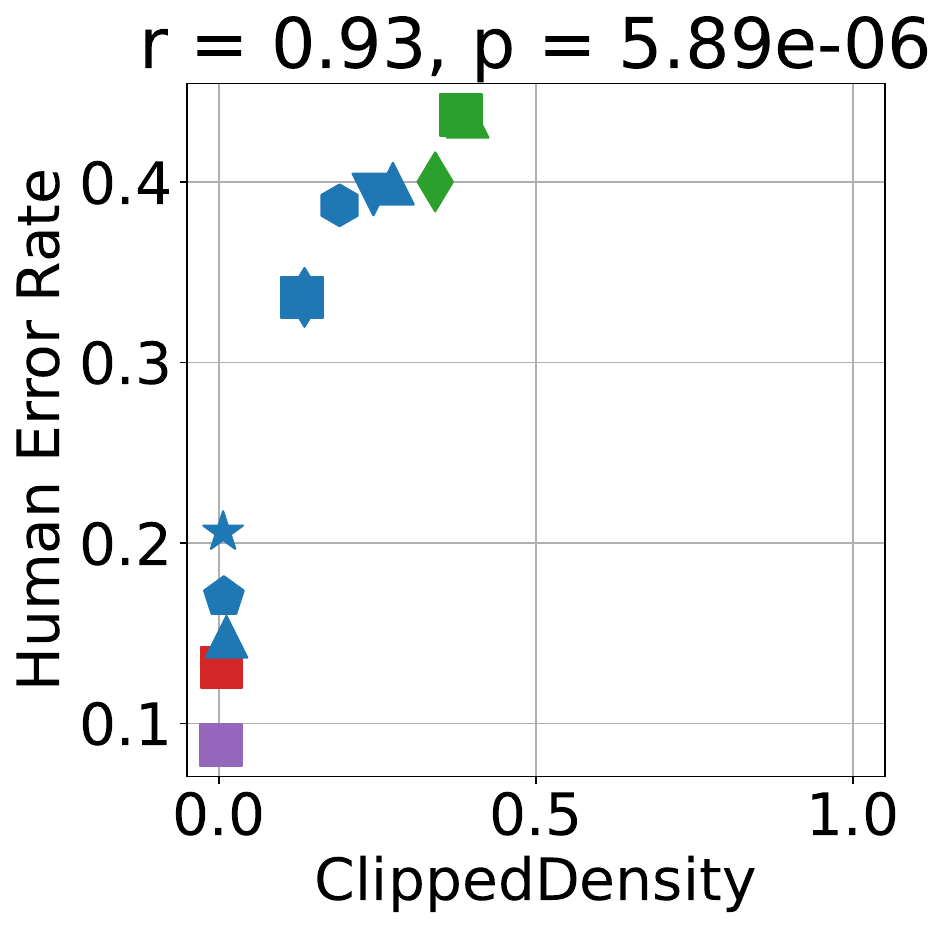}
    \end{subfigure}
    \hfill
    \begin{subfigure}[b]{0.24\textwidth}
        \centering
        \includegraphics[width=\textwidth]{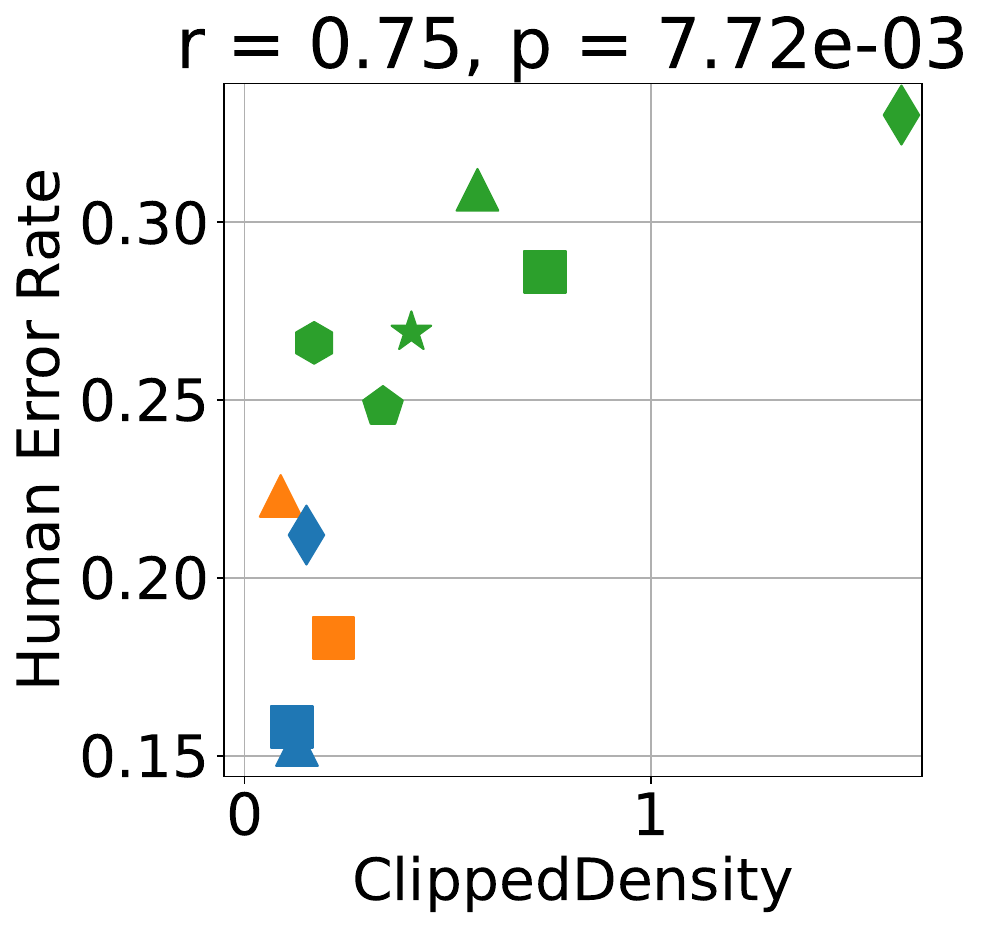}
    \end{subfigure}
    \hfill
    \begin{subfigure}[b]{0.231\textwidth}
        \centering
        \includegraphics[width=\textwidth]{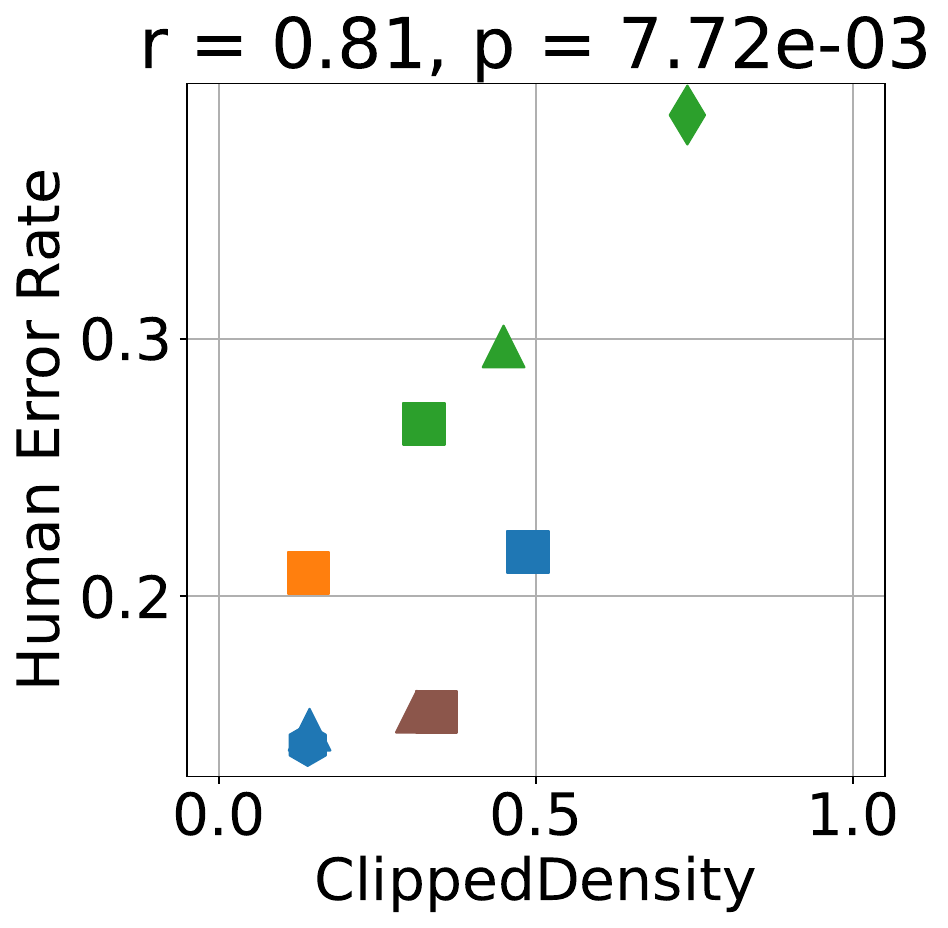}
    \end{subfigure}
    \hfill
    \begin{subfigure}[b]{0.243\textwidth}
        \centering
        \includegraphics[width=\textwidth]{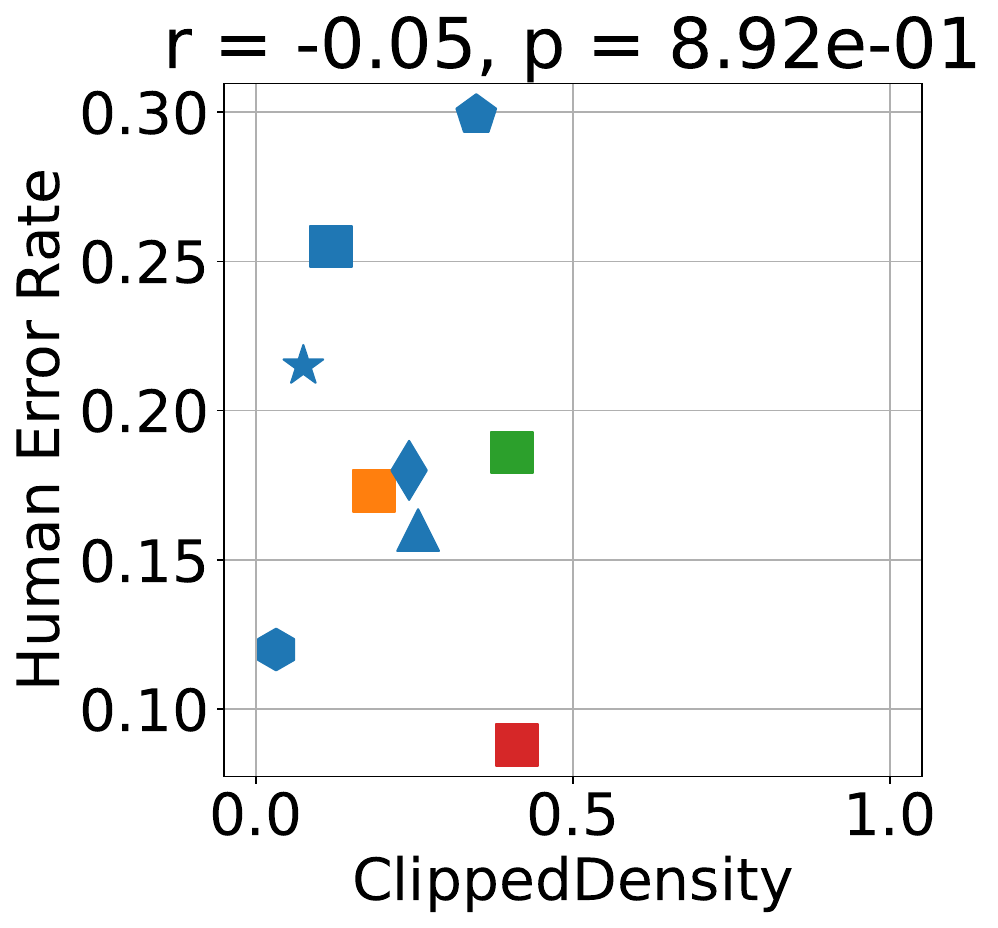}
    \end{subfigure}
        
    \centering
    \begin{subfigure}[b]{0.23\textwidth}
        \centering
        \includegraphics[width=\textwidth]{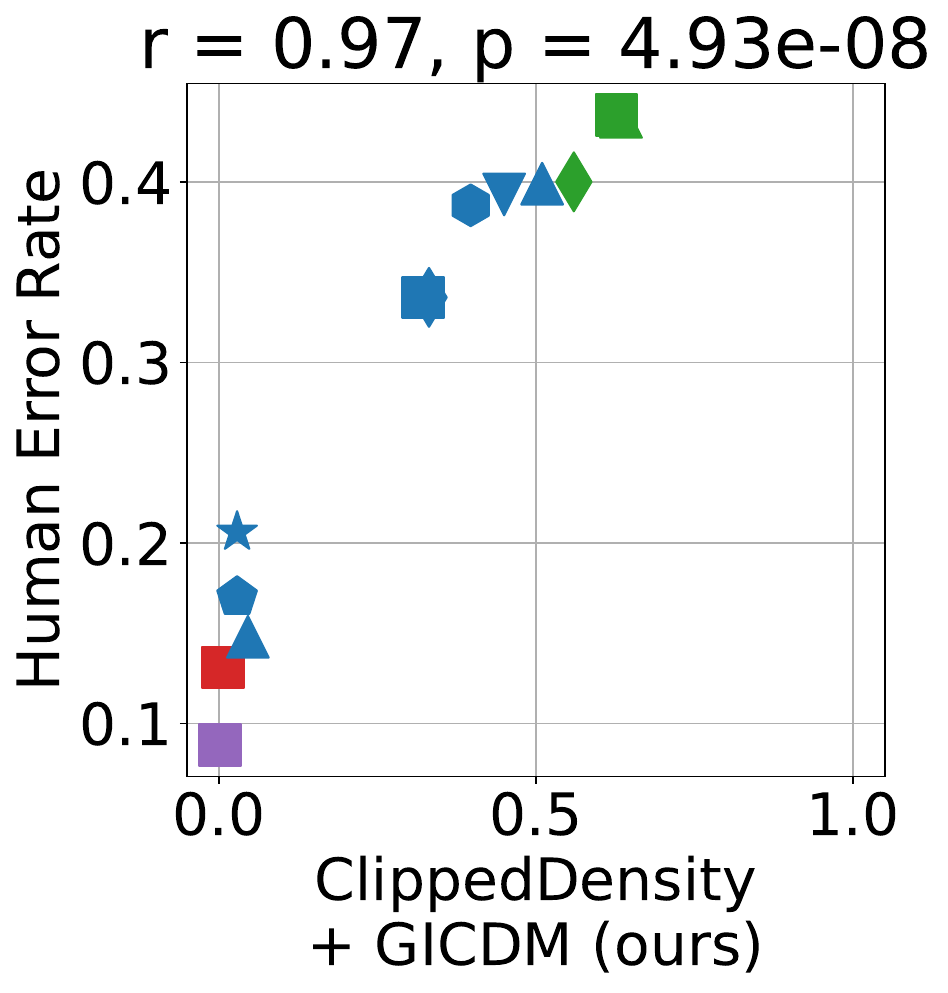}
    \end{subfigure}
    \hfill
    \begin{subfigure}[b]{0.239\textwidth}
        \centering
        \includegraphics[width=\textwidth]{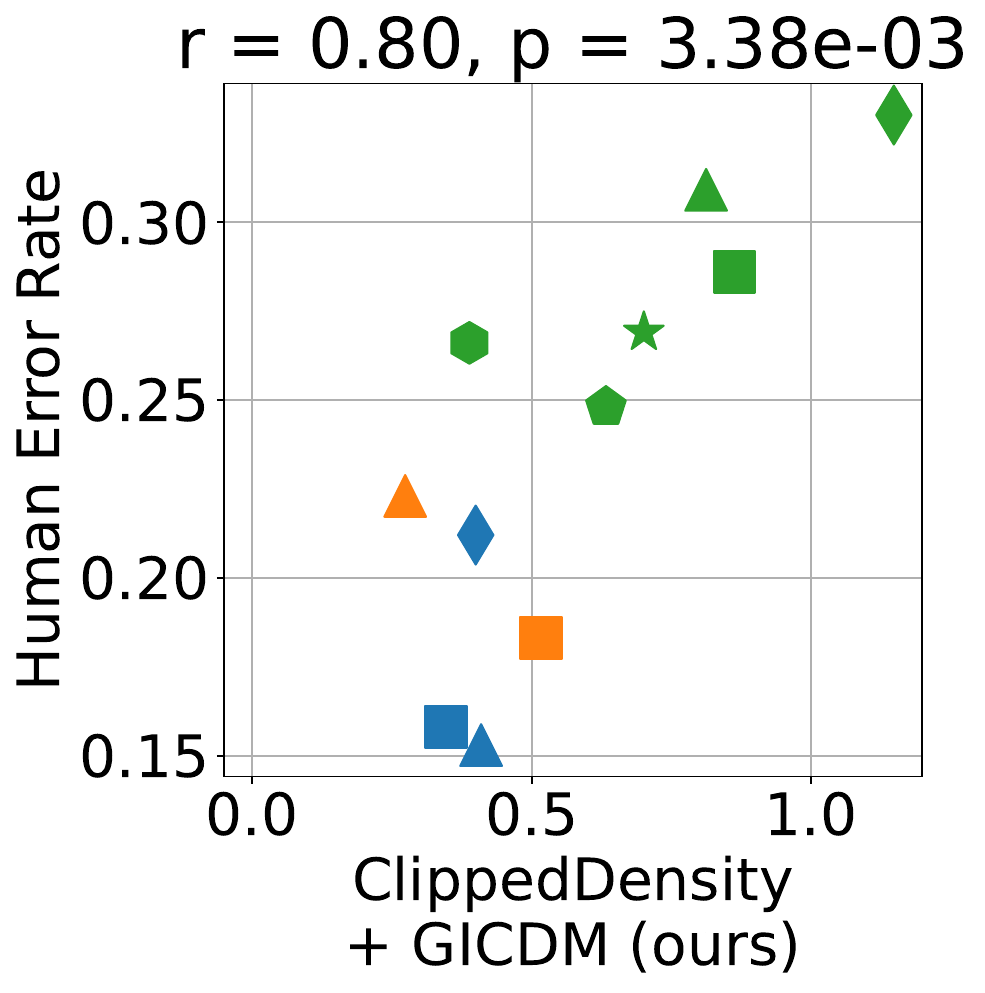}
    \end{subfigure}
    \hfill
    \begin{subfigure}[b]{0.23\textwidth}
        \centering
        \includegraphics[width=\textwidth]{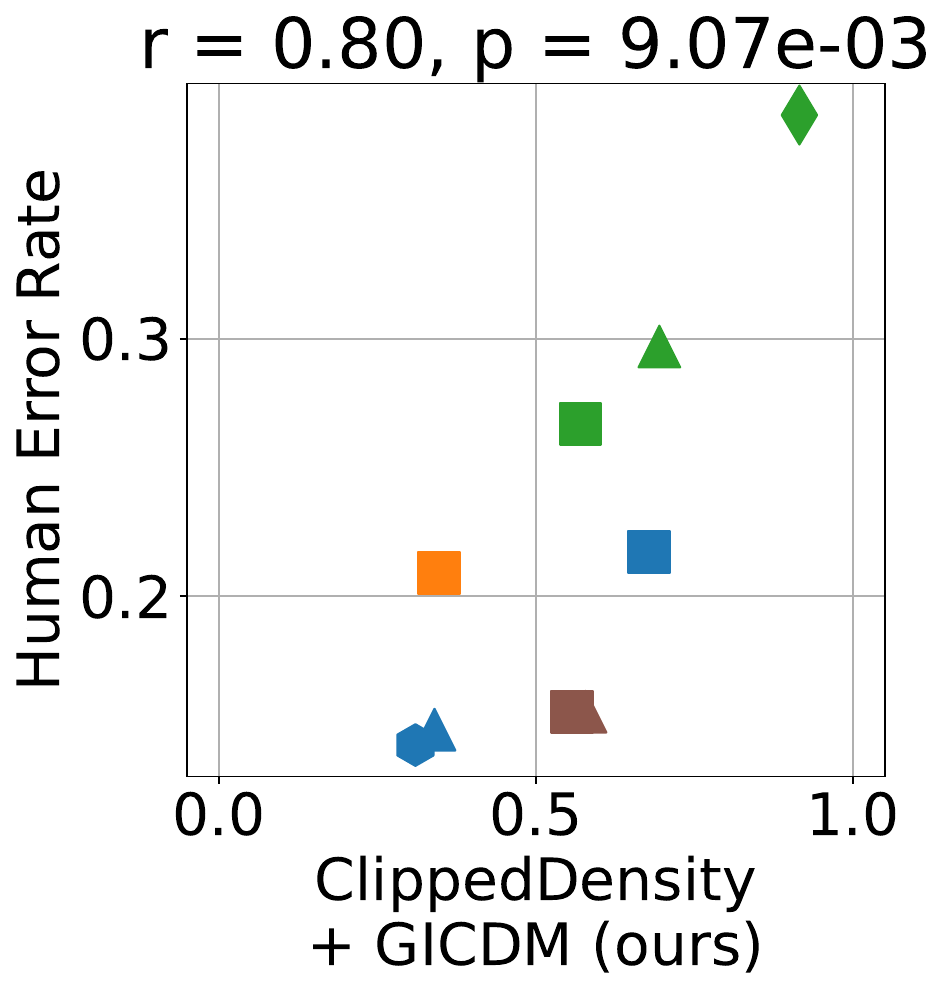}
    \end{subfigure}
    \hfill
    \begin{subfigure}[b]{0.239\textwidth}
        \centering
        \includegraphics[width=\textwidth]{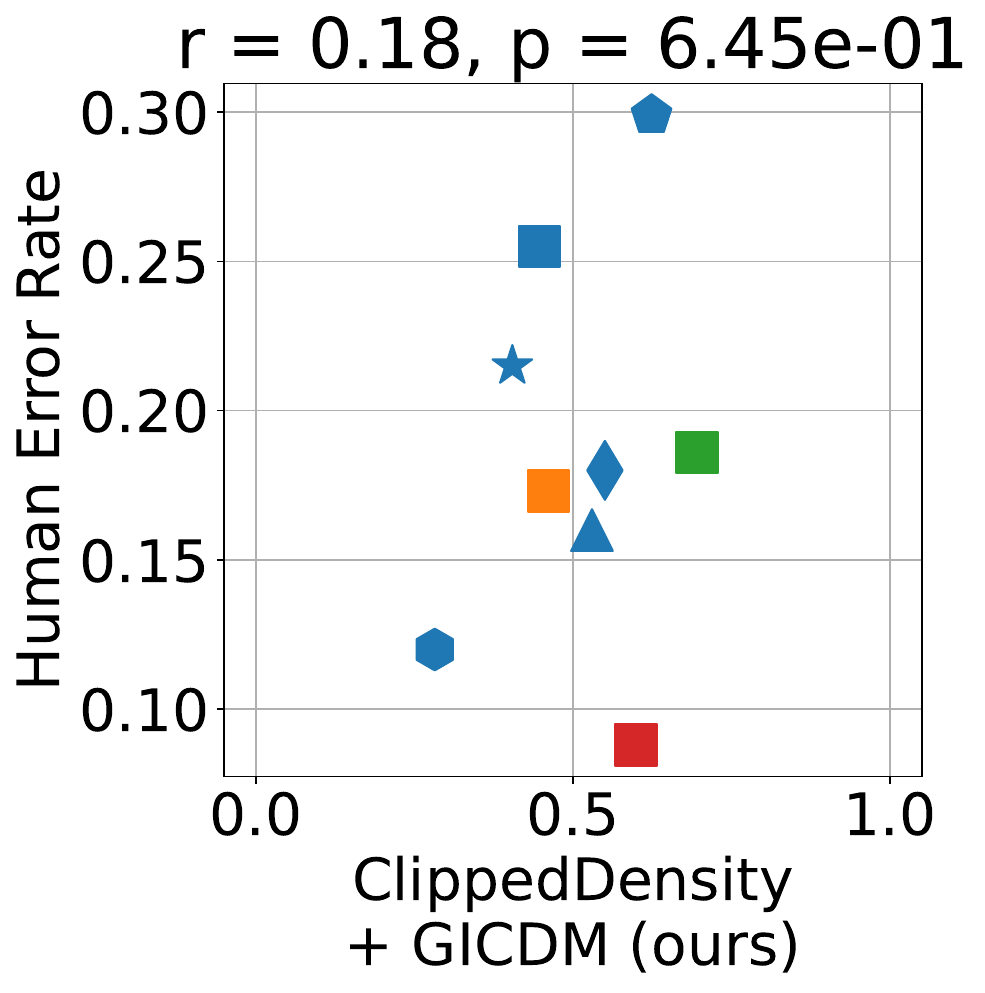}
    \end{subfigure}

    \begin{subfigure}[t]{0.01\textwidth}
    \end{subfigure}
    \hfill
    \begin{subfigure}[t]{0.229\textwidth}
        \centering
        \includegraphics[width=0.8\textwidth]{figures/fig6/DINOv2/Cifar10/legend_models.pdf}
    \end{subfigure}
    \hfill
    \begin{subfigure}[t]{0.24\textwidth}
       \centering
       \raisebox{15pt}[0pt][0pt]{\includegraphics[width=0.82\textwidth, trim={0 0 0 15}, clip]{figures/fig6/DINOv2/ImageNet256/legend_models.pdf}}
    \end{subfigure}   
    \hfill
    \begin{subfigure}[t]{0.24\textwidth}
        \centering
        \raisebox{29pt}[0pt][0pt]{\includegraphics[width=0.87\textwidth, trim={0 0 0 30}, clip]{figures/fig6/DINOv2/LSUN256/legend_models.pdf}}
    \end{subfigure}
    \hfill
    \begin{subfigure}[t]{0.24\textwidth}
        \centering
        \raisebox{30pt}[0pt][0pt]{\includegraphics[width=0.80\textwidth, trim={0 0 0 30}, clip]{figures/fig6/DINOv2/FFHQ256/legend_models.pdf}}
    \end{subfigure} 
    \caption{\textbf{Fidelity vs Human error rates in DINOv2 embeddings}: Human were tasked with discriminating real from generated images. Top row: Clipped Density vs human error rates. Bottom row: Clipped Density + GICDM vs human error rates. Each plot reports the Pearson correlation coefficient $r$ and its $p$-value. Results are not significant for FFHQ, consistent with what was reported for DINOv2 previously \cite{stein2024exposing}. For other datasets, correlation remains stable or improves slightly with GICDM.}
    \label{fig:dino_v2_human}
  \end{figure*}

\begin{figure*}[t!]
    \centering
    \includegraphics[width=0.95\textwidth]{figures/fig6/DINOv3/FFHQ256/legend_categories.pdf}
    \vspace{-0.1cm}

    \centering
    \begin{subfigure}[b]{0.01\textwidth}
    \end{subfigure}
    \hfill
    \begin{subfigure}[b]{0.24\textwidth}
        \centering
        {CIFAR-10}\\
    \end{subfigure}
    \hfill
    \begin{subfigure}[b]{0.24\textwidth}
        \centering
        {ImageNet}\\
    \end{subfigure}
    \hfill
    \begin{subfigure}[b]{0.24\textwidth}
        \centering
        {LSUN Bedroom}\\
    \end{subfigure}
    \hfill
    \begin{subfigure}[b]{0.24\textwidth}
        \centering
        {FFHQ}\\
    \end{subfigure}

    \centering
    \begin{subfigure}[b]{0.23\textwidth}
        \centering
        \includegraphics[width=1.0\textwidth]{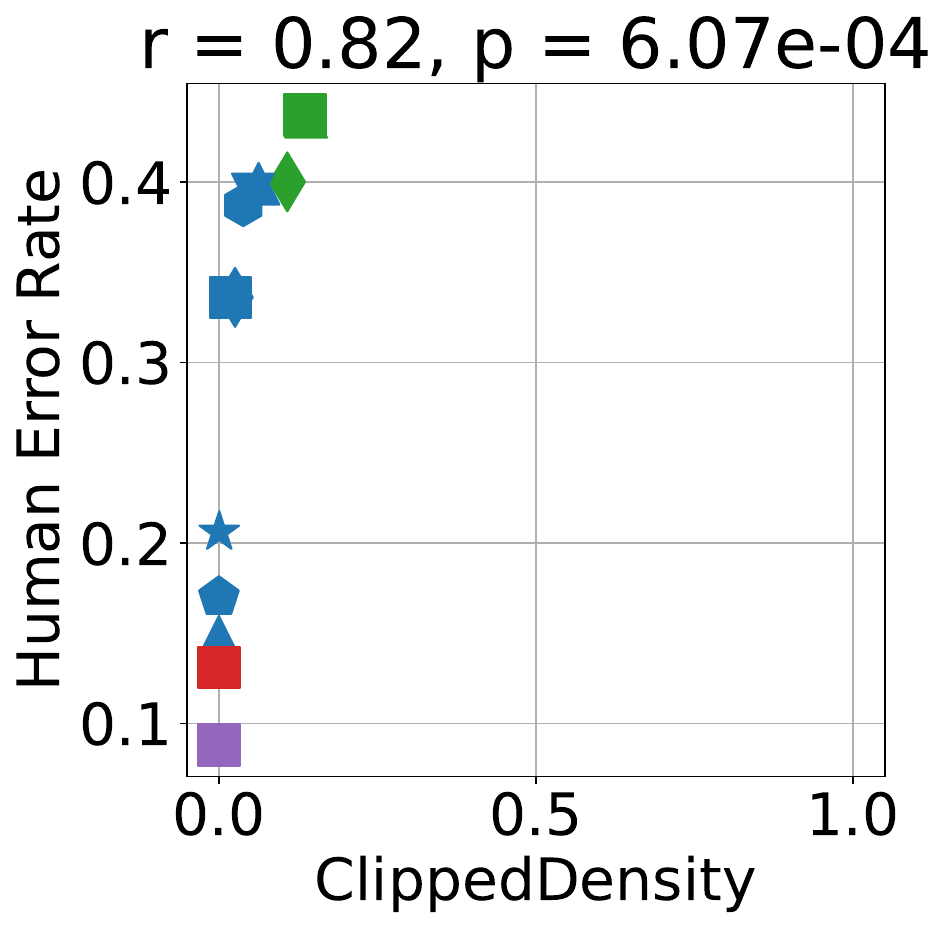}
    \end{subfigure}
    \hfill
    \begin{subfigure}[b]{0.239\textwidth}
        \centering
        \includegraphics[width=\textwidth]{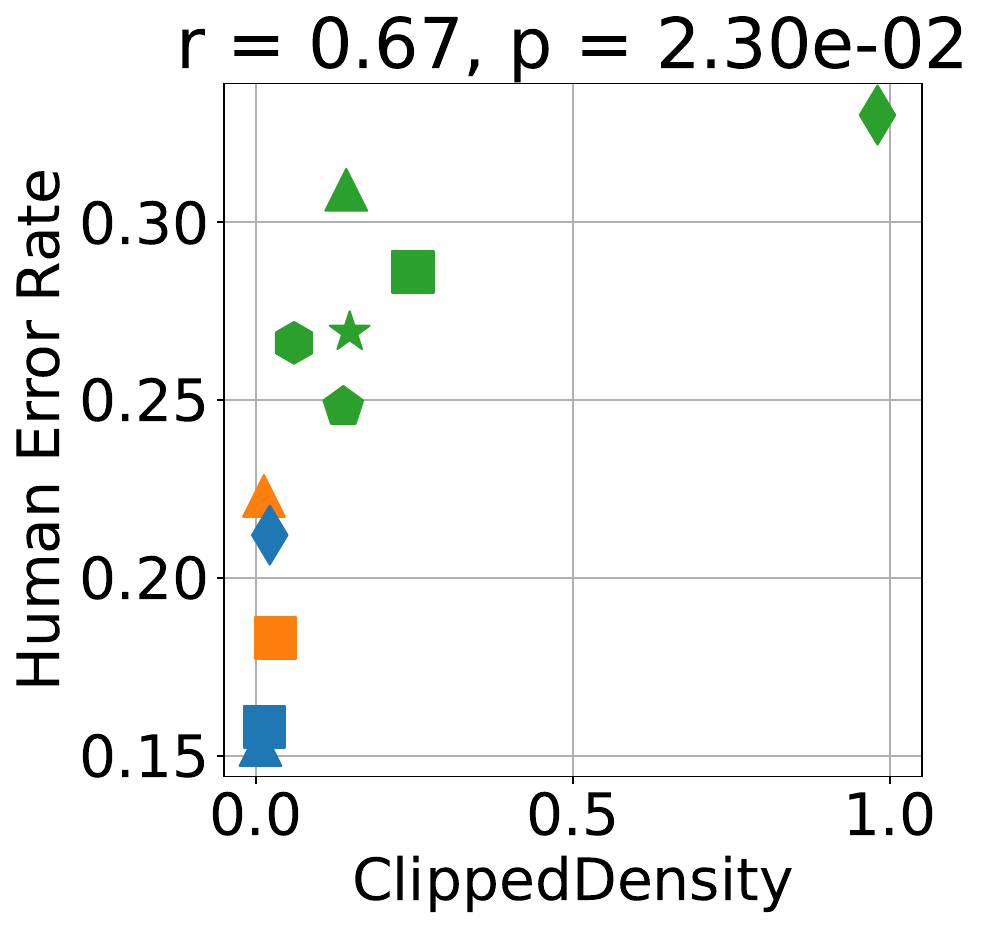}
    \end{subfigure}
    \hfill
    \begin{subfigure}[b]{0.23\textwidth}
        \centering
        \includegraphics[width=\textwidth]{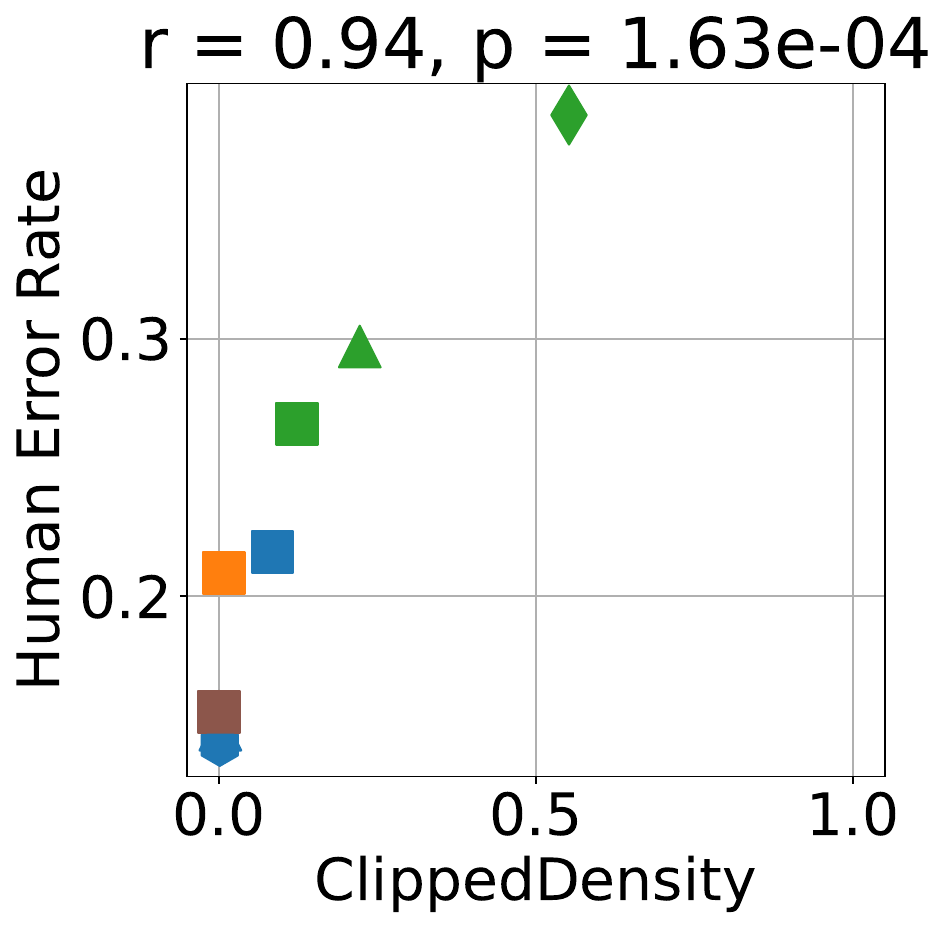}
    \end{subfigure}
    \hfill
    \begin{subfigure}[b]{0.242\textwidth}
        \centering
        \includegraphics[width=\textwidth]{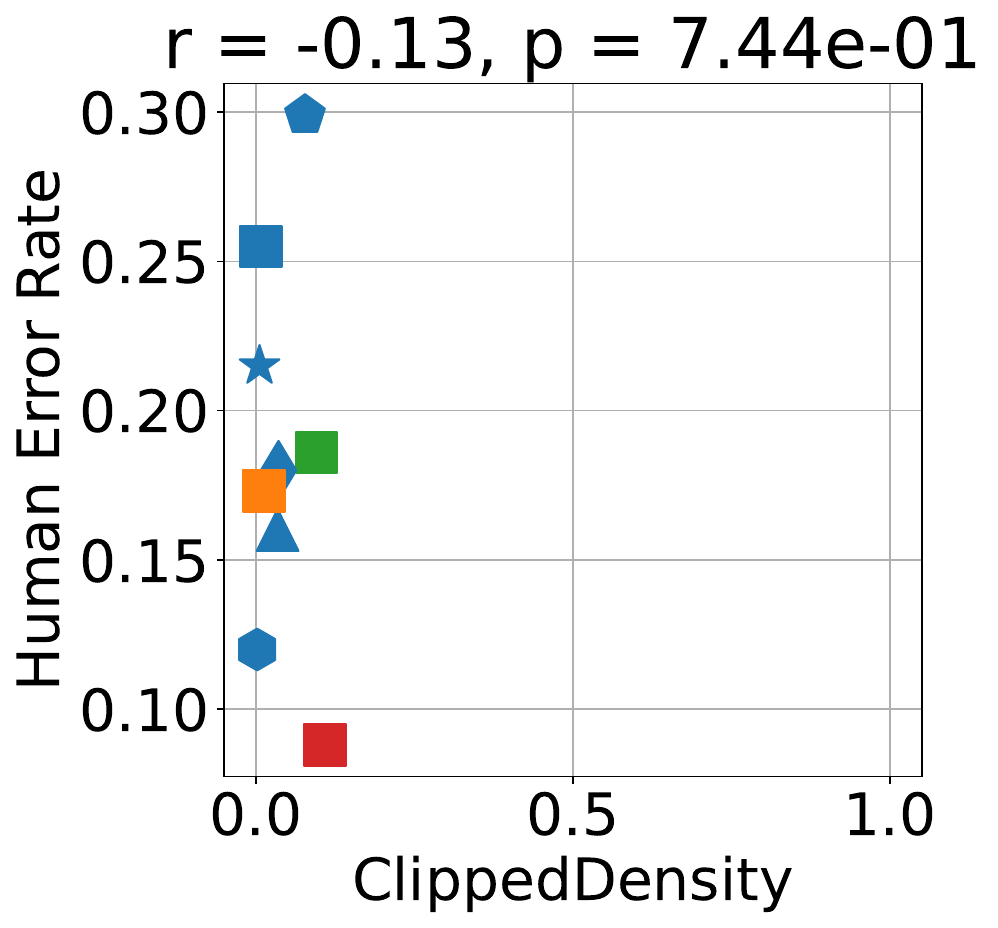}
    \end{subfigure}
        
    \centering
    \begin{subfigure}[b]{0.23\textwidth}
        \centering
        \includegraphics[width=\textwidth]{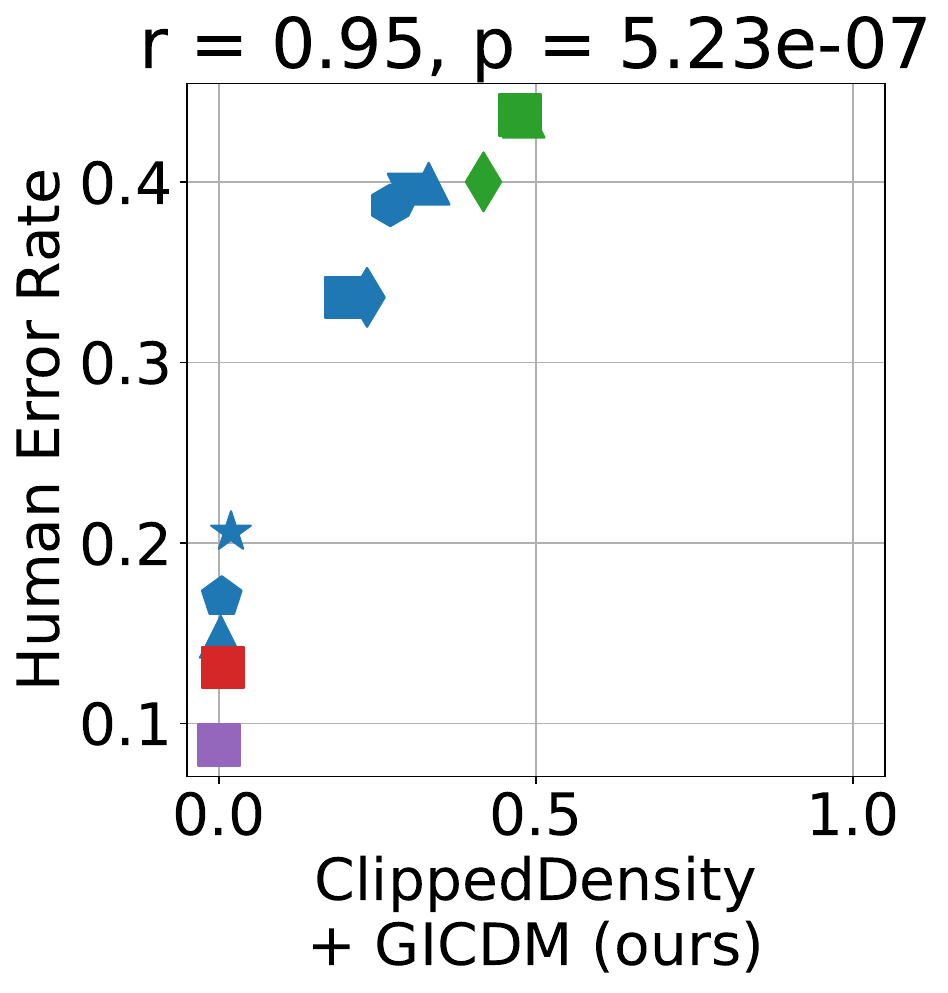}
    \end{subfigure}
    \hfill
    \begin{subfigure}[b]{0.239\textwidth}
        \centering
        \includegraphics[width=\textwidth]{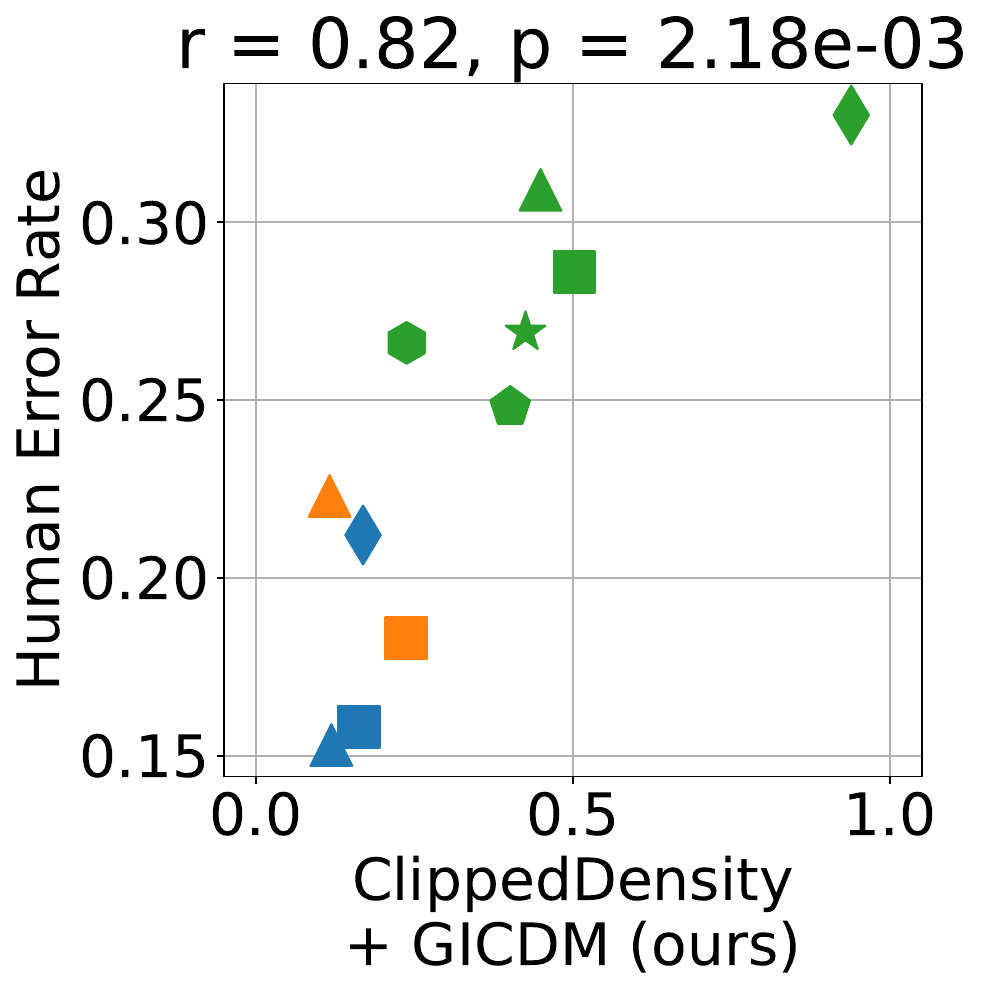}
    \end{subfigure}
    \hfill
    \begin{subfigure}[b]{0.23\textwidth}
        \centering
        \includegraphics[width=\textwidth]{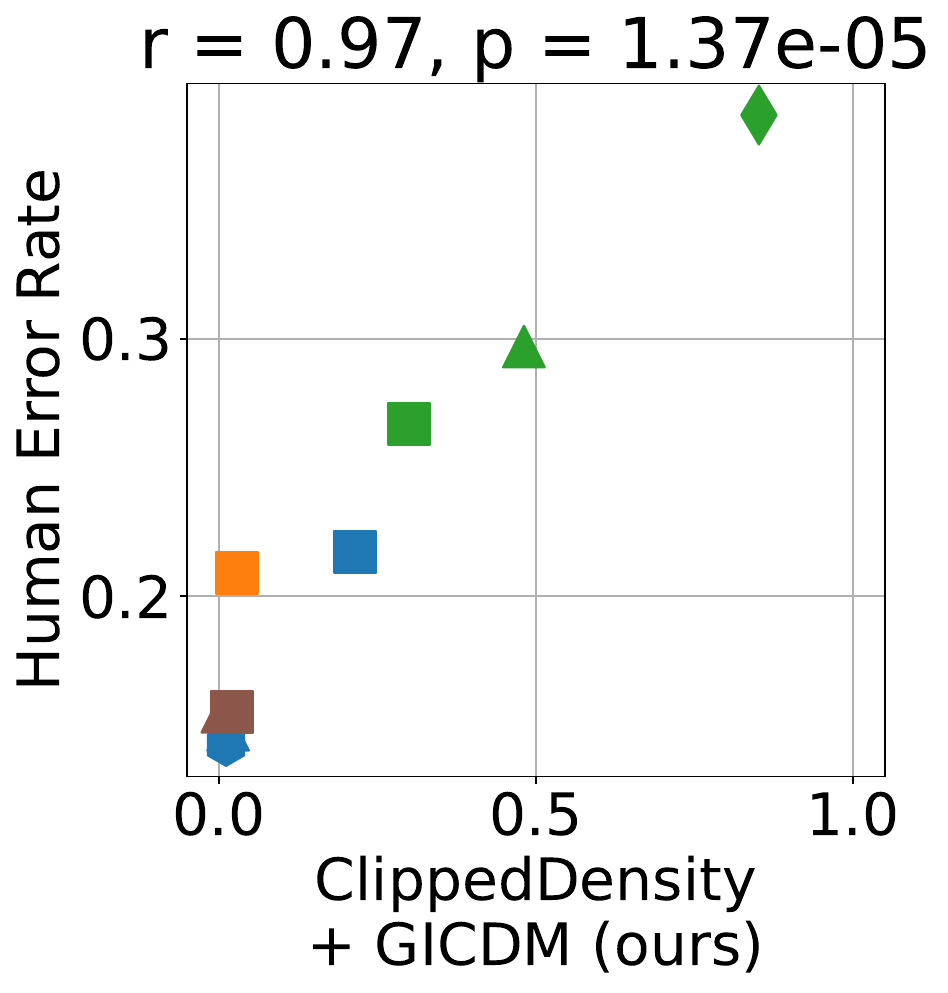}
    \end{subfigure}
    \hfill
    \begin{subfigure}[b]{0.242\textwidth}
        \centering
        \includegraphics[width=\textwidth]{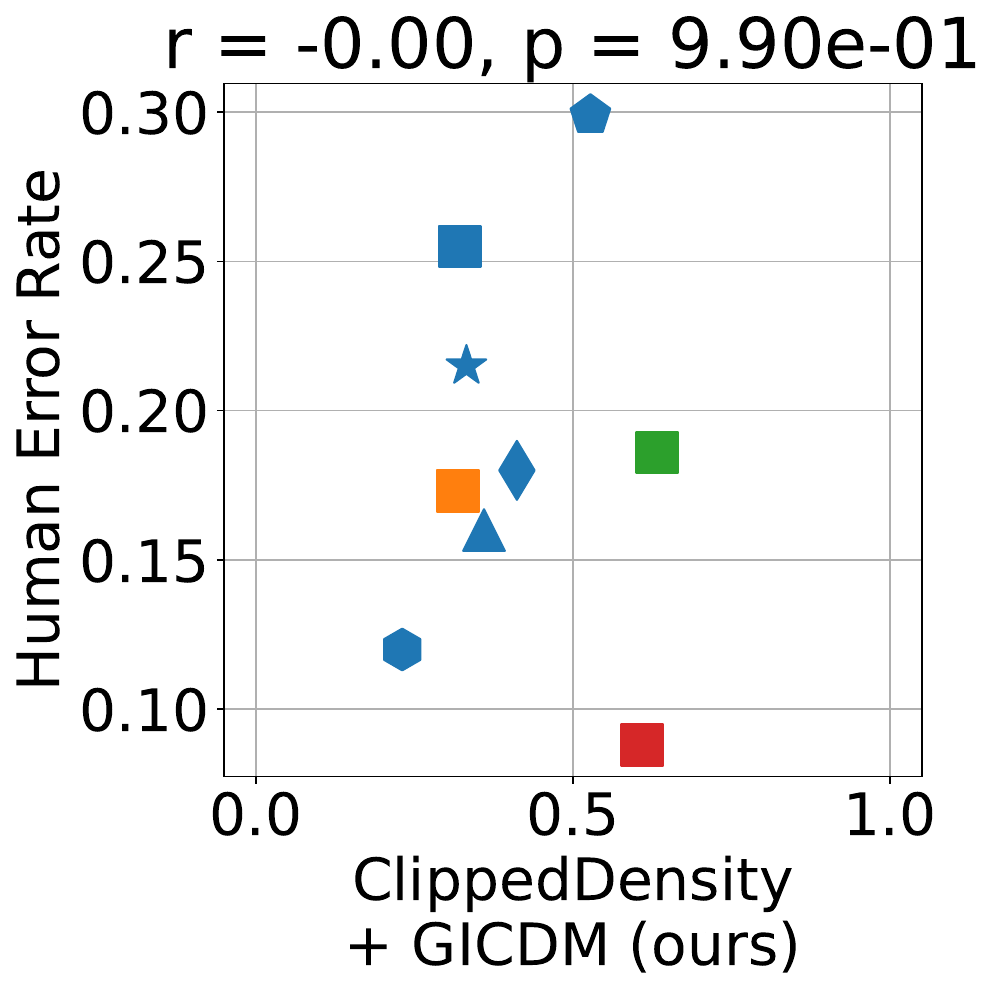}
    \end{subfigure}

    \begin{subfigure}[t]{0.01\textwidth}
    \end{subfigure}
    \hfill
    \begin{subfigure}[t]{0.229\textwidth}
        \centering
        \includegraphics[width=0.8\textwidth]{figures/fig6/DINOv3/Cifar10/legend_models.pdf}
    \end{subfigure}
    \hfill
    \begin{subfigure}[t]{0.24\textwidth}
       \centering
       \raisebox{15pt}[0pt][0pt]{\includegraphics[width=0.82\textwidth, trim={0 0 0 15}, clip]{figures/fig6/DINOv3/ImageNet256/legend_models.pdf}}
    \end{subfigure}   
    \hfill
    \begin{subfigure}[t]{0.24\textwidth}
        \centering
        \raisebox{29pt}[0pt][0pt]{\includegraphics[width=0.87\textwidth, trim={0 0 0 30}, clip]{figures/fig6/DINOv3/LSUN256/legend_models.pdf}}
    \end{subfigure}
    \hfill
    \begin{subfigure}[t]{0.24\textwidth}
        \centering
        \raisebox{30pt}[0pt][0pt]{\includegraphics[width=0.80\textwidth, trim={0 0 0 30}, clip]{figures/fig6/DINOv3/FFHQ256/legend_models.pdf}}
    \end{subfigure} 
    \caption{\textbf{Fidelity vs Human error rates in DINOv3 embeddings}: Same setup as \Cref{fig:dino_v2_human}, but using DINOv3 embeddings. As with DINOv2, results are not significant for FFHQ. For other datasets, correlation increases with GICDM, and the improvement is more pronounced than with DINOv2, since DINOv3 embeddings are more affected by hubness.}
    \label{fig:dino_v3_human}
  \end{figure*}

\begin{table}[t!]
  \caption{\textbf{Correlation with human scores.} Significant Pearson correlation coefficients between metric scores and human ratings for each dataset and embedding; non-significant values are indicated by --. The FFHQ column is omitted due to the absence of significant correlations. This is consistent with prior findings using DINOv2 embeddings \cite{stein2024exposing}.
  For metrics robust to real outliers (Clipped Density and symPrecision, which equals cPrecision in this context), adding GICDM maintains or improves correlation. For other metrics, adding GICDM generally worsens results, indicating that robustness to real outliers is necessary for GICDM to be effective.
  }
  \label{tab:correlation_human_full}
  \centering
  \begin{small}
  \begin{sc}
    \begin{tabular}{llccc}
    \toprule
    Embedding & Metric & Cifar10 & ImageNet & LSUN Bedroom \\
    \midrule
    \multirow{12}{*}{DINOv2}
      & Precision                & $-0.82$ & $0.76$ & -- \\
      & \ \ \ + GICDM        & $0.82$  & --     & $-0.78$ \\[0.5em]
      & Density                  & $-0.91$ & $0.68$ & -- \\
      & \ \ \ + GICDM           & $0.93$  & --     & $-0.72$ \\[0.5em]
      & P-precision              & $0.98$  & $0.75$ & $0.69$ \\
      & \ \ \ + GICDM       & $0.82$  & --     & $-0.78$ \\[0.5em]
      & symPrecision             & $0.87$  & $0.61$ & $0.99$ \\
      & \ \ \ + GICDM      & $0.97$  & $0.76$ & $0.98$ \\[0.5em]
      & \makecell[l]{Clipped Density\\(unclipped radii)} & $0.73$ & $-0.85$ & -- \\
      & \ \ \ + GICDM  & -- & $0.89$ & $-0.78$ \\[0.5em]
      & Clipped Density           & $0.93$  & $0.75$ & $0.81$ \\
      & \ \ \ + GICDM    & $0.97$  & $0.80$ & $0.80$ \\
    \midrule
    \multirow{12}{*}{DINOv3}
      & Precision                & $0.86$  & $0.74$ & $0.96$ \\
      & \ \ \ + GICDM         & $-0.96$ & --     & -- \\[0.5em]
      & Density                  & $0.83$  & $0.66$ & $0.95$ \\
      & \ \ \ + GICDM           & $-0.60$ & $-0.73$ & -- \\[0.5em]
      & P-precision              & $0.97$  & $0.79$ & $0.97$ \\
      & \ \ \ + GICDM           & $-0.96$ & --     & -- \\[0.5em]
      & symPrecision             & $0.86$  & $0.80$ & $0.94$ \\
      & \ \ \ + GICDM      & $0.97$  & $0.82$ & $0.95$ \\[0.5em]
      & \makecell[l]{Clipped Density\\(unclipped radii)} & $0.70$ & $0.85$ & $0.95$ \\
      & \ \ \ + GICDM  & $-0.62$ & $-0.94$ & -- \\[0.5em]
      & Clipped Density           & $0.82$  & $0.67$ & $0.94$ \\
      & \ \ \ + GICDM   & $0.95$  & $0.82$ & $0.97$ \\
    \bottomrule
  \end{tabular}

  \end{sc}
  \end{small}
\end{table}

Alongside generated images, \citet{stein2024exposing} also released human error rates for classifying real versus generated images. This discriminator error rate serves as a measure of fidelity: lower error rates indicate higher fidelity of the generated images. We evaluate the correlation between human scores and fidelity metrics, both with and without GICDM.

\Cref{fig:dino_v2_human,fig:dino_v3_human} show fidelity versus human error rates for DINOv2 and DINOv3 embeddings, respectively, using Clipped Density with and without GICDM. Each plot reports the Pearson correlation coefficient $r$ and its $p$-value.
\Cref{tab:correlation_human_full} summarizes Pearson correlation coefficients for all distance-based metrics, with and without GICDM, on DINOv2 and DINOv3 embeddings.
Results for FFHQ are omitted due to lack of significant correlation, as are VGG16 and Inceptionv3 embeddings, which showed almost no significant correlation with any metric.

For Clipped Density and symPrecision, adding GICDM consistently maintains or improves correlation with human scores.
For other metrics, this improvement is not observed, and adding GICDM often worsens correlation.

We hypothesize that this lack of improvement is due to insufficient robustness to real outliers, as Precision, Density, and P-precision are not robust to real outliers \cite{salvy2026enhanced}.
To test this, we computed Clipped Density without radius clipping, thereby removing its robustness to real outliers. In this case, it exhibited the same lack or worsening of correlation as the other non-robust metrics (\Cref{tab:correlation_human_full}).

Therefore, robustness to real outliers appears to be a necessary condition for GICDM to improve correlation with human scores, further supporting the recommendation to avoid non-robust metrics.

\newpage
\section{Real data Benchmark}
\label{sec:real_data_benchmark}

\begin{figure}[b!]
    \centering
      \begin{subfigure}[b]{0.48\textwidth}
        \centering
        \includegraphics[width=\textwidth]{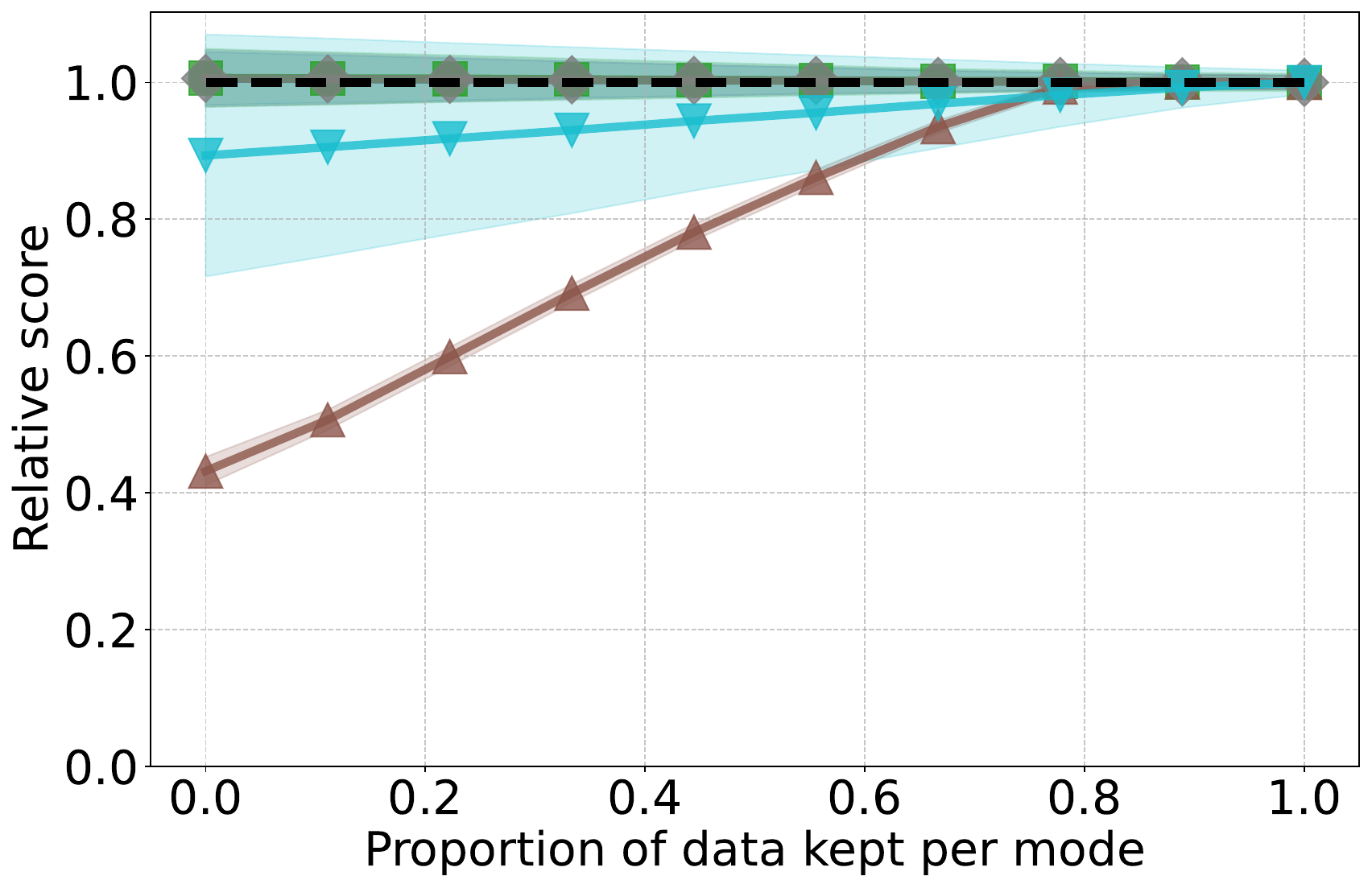}
        \caption{Simultaneously dropping modes (CIFAR-10)}
        \label{fig:real_fidelity_mode_dropping_sim}
    \end{subfigure}
    \hfill
    \begin{subfigure}[b]{0.48\textwidth}
        \centering
        \includegraphics[width=\textwidth]{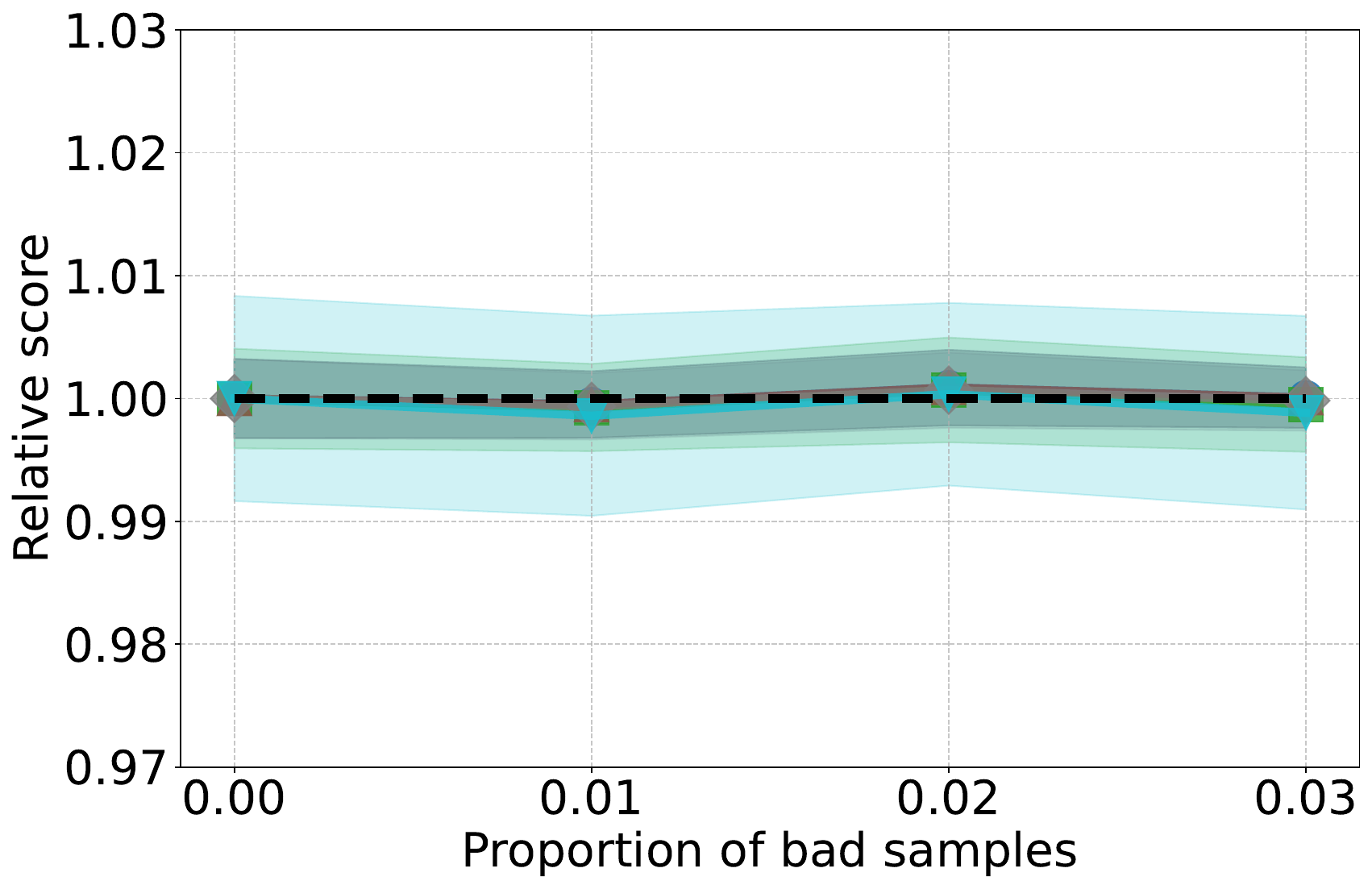}
        \caption{Introducing bad real \& syn. samples (CIFAR-10)}
        \label{fig:real_fidelity_ood_proportion_both}
    \end{subfigure}

    \vspace{0.3cm}

    \begin{subfigure}[b]{0.48\textwidth}
        \centering
        \includegraphics[width=\textwidth]{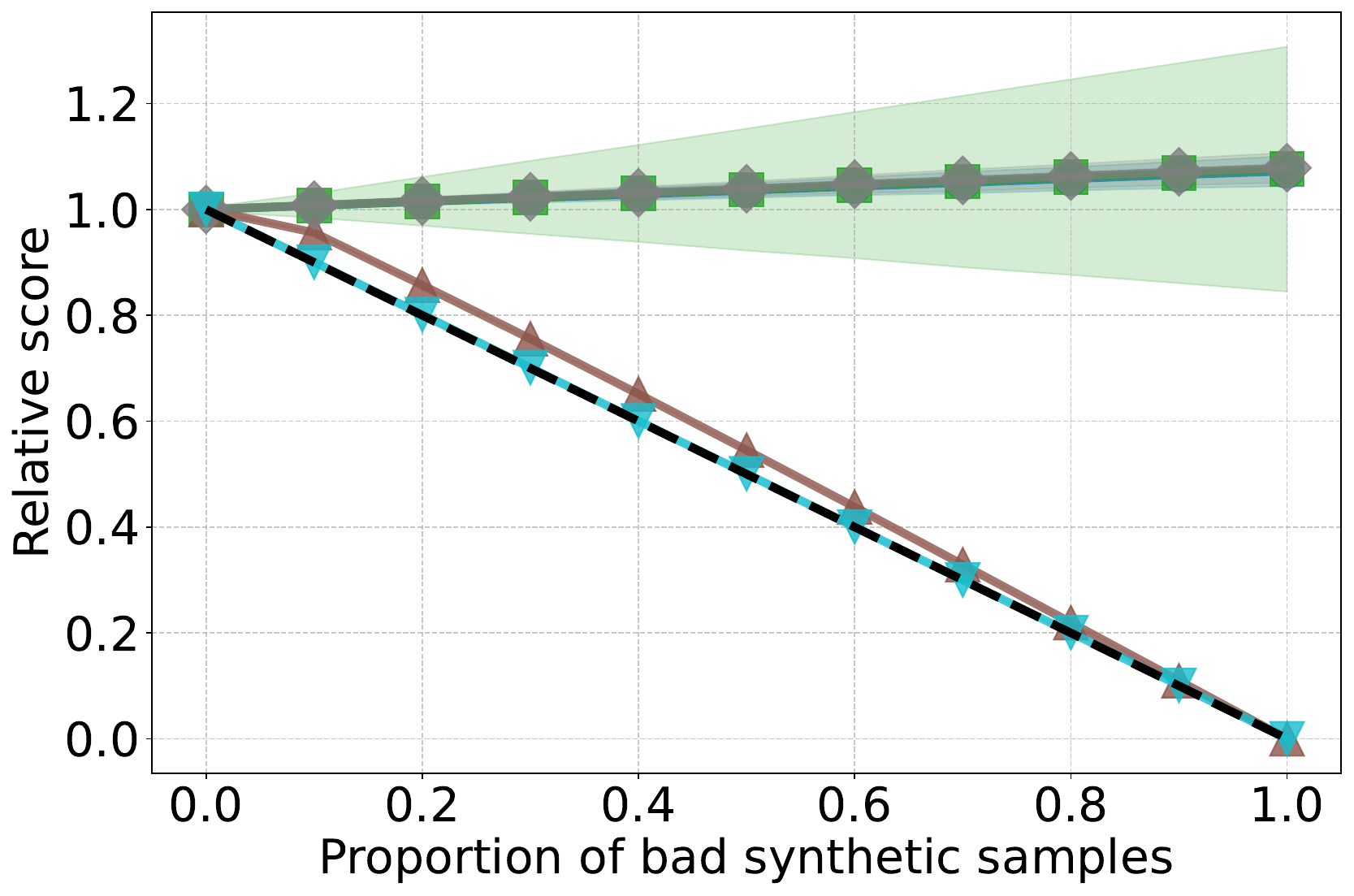}
        \caption{Fidelity reduction with bad samples (\textbf{CIFAR-10})}
        \label{fig:real_fidelity_ood_proportion_generated}
    \end{subfigure}
    
    \begin{subfigure}[b]{\textwidth}
        \centering
        \begin{tabular}{c|c|c|c}
            Legend & {\Cref{fig:real_fidelity_mode_dropping_sim}} & {\Cref{fig:real_fidelity_ood_proportion_both}} & {\Cref{fig:real_fidelity_ood_proportion_generated}} \\
            \hline
            \multirow{6}{*}{\includegraphics[width=0.29\textwidth, trim={0 0 0 10}, clip]{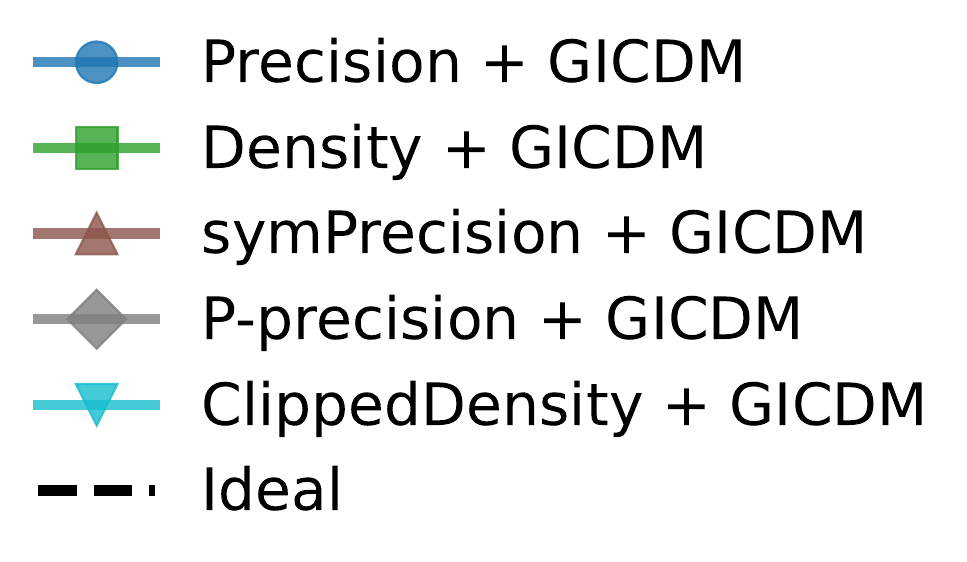}} & \cellcolor{green!20}\ding{51} & \cellcolor{green!20}\ding{51} & \cellcolor{red!20}\ding{55} \\
            \hline
            & \cellcolor{green!20}\ding{51} & \cellcolor{green!20}\ding{51} & \cellcolor{red!20}\ding{55} \\
            \hline
            & \cellcolor{red!20}\ding{55} & \cellcolor{green!20}\ding{51} & \cellcolor{green!20}\ding{51} \\
            \hline
            & \cellcolor{green!20}\ding{51} & \cellcolor{green!20}\ding{51} & \cellcolor{red!20}\ding{55} \\
            \hline
            & \cellcolor{green!20}\ding{51} & \cellcolor{green!20}\ding{51}  & \cellcolor{green!20}\ding{51} \\
            \hline
            \\
        \end{tabular}
        \caption{Legend and summary}
        \label{fig:fidelity_tests_summary}
    \end{subfigure}
    \caption{\textbf{Real data tests for fidelity metrics}. Fidelity metrics with GICDM are evaluated in several scenarios.
    (a) Simultaneous mode dropping: synthetic data from all but one CIFAR-10 class is progressively replaced with data from the remaining class. This test is related to "Dropping + Invention" in \Cref{tab:pos_fidelity}, but on real data. Similarly, only symPrecision fails.
    (b) Introducing real and synthetic out-of-distribution samples at equal rates: this tests the stability of the metrics. All succeed.
    }
    \label{fig:fidelitymetrics}
\end{figure} 

\begin{figure}[t!]
    \centering
      \begin{subfigure}[b]{0.48\textwidth}
        \centering
        \includegraphics[width=\textwidth]{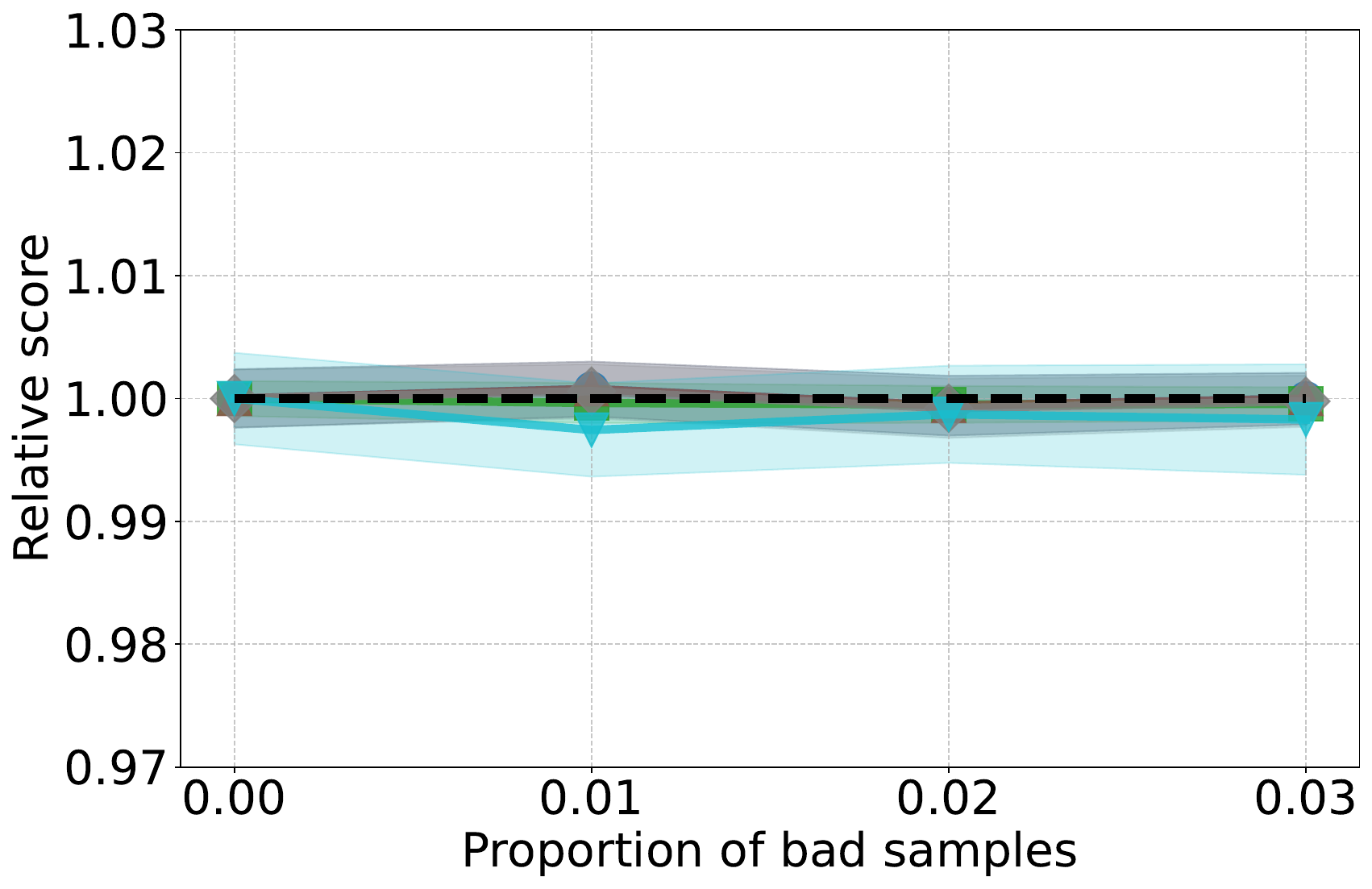}
        \caption{Introducing bad real \& syn. samples (CIFAR-10)}
        \label{fig:real_coverage_ood_proportion_both}
    \end{subfigure}
    \hfill
    \begin{subfigure}[b]{0.48\textwidth}
        \centering
        \includegraphics[width=\textwidth]{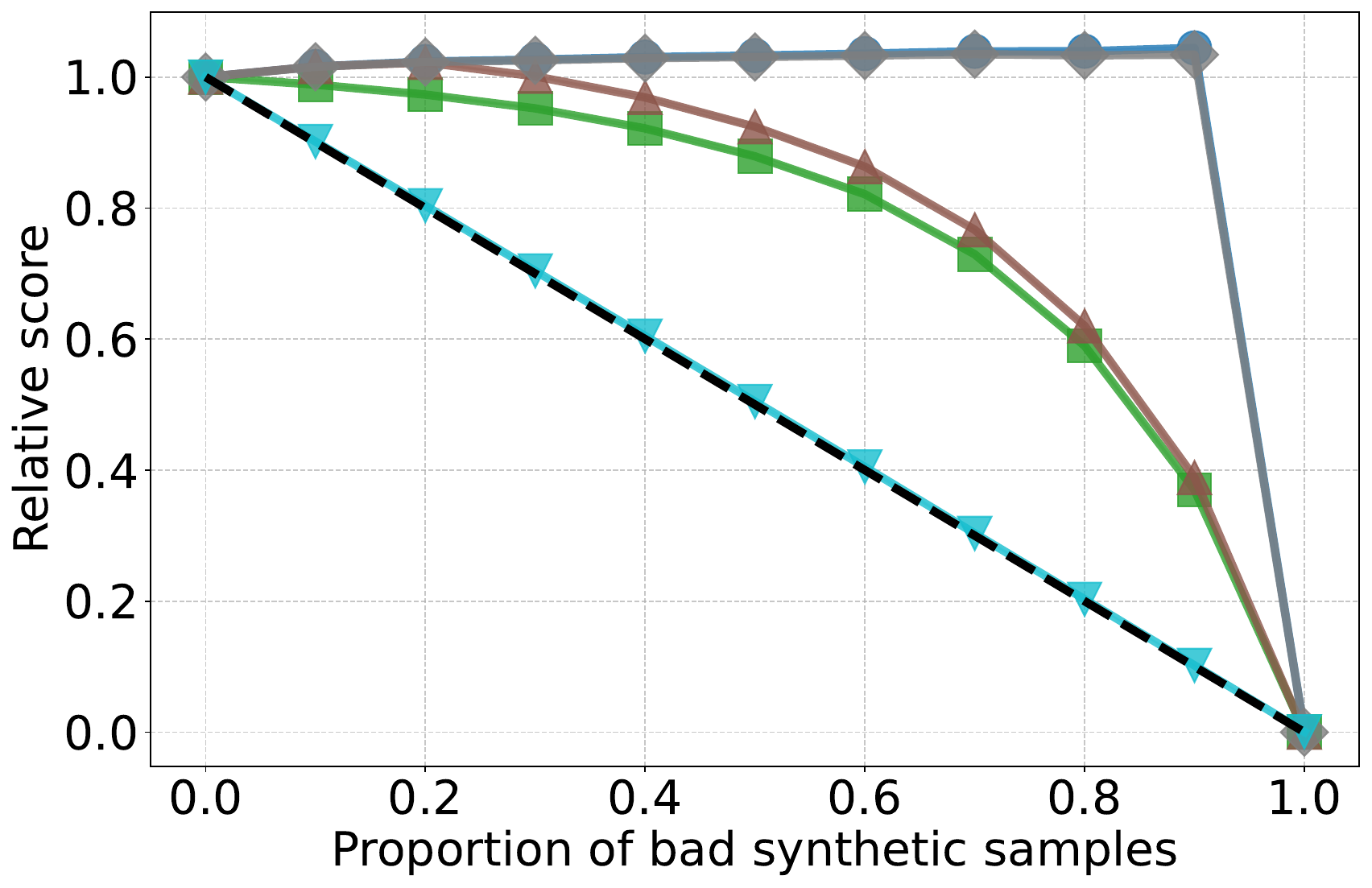}
        \caption{Coverage reduction with bad samples (CIFAR-10)}
        \label{fig:real_coverage_ood_proportion_generated}
    \end{subfigure}

    \vspace{0.3cm}
    
    \begin{subfigure}[b]{\textwidth}
        \centering
        \begin{tabular}{c|c|c}
            Legend & {\Cref{fig:real_coverage_ood_proportion_both}} & {\Cref{fig:real_coverage_ood_proportion_generated}} \\
            \hline
            \multirow{6}{*}{\includegraphics[width=0.3\textwidth, trim={0 0 0 10}, clip]{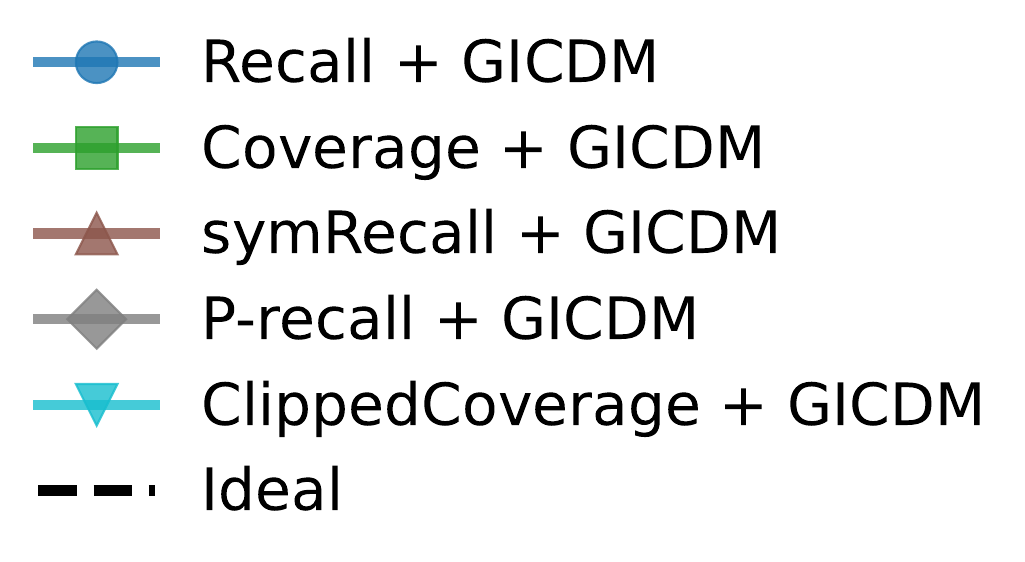}} & \cellcolor{green!20}\ding{51} & \cellcolor{red!20}\ding{55} \\
            \hline
            & \cellcolor{green!20}\ding{51} & \cellcolor{red!20}\ding{55} \\
            \hline
            & \cellcolor{green!20}\ding{51} & \cellcolor{red!20}\ding{55} \\
            \hline
            & \cellcolor{green!20}\ding{51} & \cellcolor{red!20}\ding{55} \\
            \hline
            & \cellcolor{green!20}\ding{51} & \cellcolor{green!20}\ding{51} \\
            \hline
            \\
        \end{tabular}
        \caption{Legend and summary}
        \label{fig:coverage_tests_summary}
    \end{subfigure}
    \caption{\textbf{Real data tests for coverage metrics}. Coverage metrics with GICDM are evaluated in several scenarios, analogous to those in \Cref{fig:fidelitymetrics}.
    (a) Introducing real and synthetic out-of-distribution samples at equal rates: all metrics succeed.
    (b) Progressively replacing good CIFAR-10 synthetic samples with bad ones. Results are consistent with those reported without GICDM: only Clipped Coverage succeeds \cite{salvy2026enhanced}.
    }
    \label{fig:coveragemetrics}
\end{figure} 

\citet{salvy2026enhanced} introduced a benchmark for fidelity and coverage metrics, primarily using DINOv2 embeddings of real CIFAR10 data. In this section, we evaluate distance-based metrics with GICDM on this benchmark.
We omit their "translating synthetic Gaussian test" on synthetic data, as it is similar to the "Gaussian Mean Difference + Outlier" test in \Cref{appendix:position_benchmark}.

Compared to the results reported by \citet{salvy2026enhanced} without GICDM, two main differences arise. 
First, with GICDM, P-precision and P-recall now succeed in the "Introducing bad real and synthetic samples" test.
Second, when progressively replacing good synthetic samples with bad ones for fidelity, only Clipped Density and symPrecision succeed, whereas previously all fidelity metrics succeeded \cite{salvy2026enhanced}. This failure occurs for the same metrics that failed in the human correlation study in \Cref{sec:human_correlation}, specifically those lacking robustness to real outliers \cite{salvy2026enhanced}.

\newpage
\section{Visualization of hubness effects and GICDM correction}

For LSUN Bedroom data embedded with DINOv2, \Cref{fig:intro_viz} (top row) displays the real image that appears most frequently among the 5-nearest-neighbors of other real images, a \emph{hub}, and the set of real images that include it as a neighbor. This hub is present in 62 neighborhoods, highlighting the severity of hubness. After applying ICDM \cite{jegou2010contextual}, only 5 real images retain this hub as a neighbor.

The bottom row of \Cref{fig:intro_viz} presents the same analysis for generated images (from ADM-dropout). The most frequent generated hub is included in the neighborhoods of 64 real images, but after applying GICDM, this number drops to 4.

In both cases, before hubness reduction, many of the images listing the hub as a neighbor are not strongly semantically related to it. After hubness reduction, the remaining images are more semantically similar to the hub image.

\begin{figure}[h]
  \centering
  \begin{subfigure}[b]{0.2\textwidth}
    \centering
    \textbf{Query Image}
  \end{subfigure}
  \hfill
  \begin{subfigure}[b]{0.78\textwidth}
    \centering
    \textbf{Real images including the left Hub in their 5-Nearest-Neighbors}
  \end{subfigure}

  \vspace{0.3em}

  \begin{subfigure}[b]{0.2\textwidth}
    \centering
    LSUN Bedroom Hub
  \end{subfigure}
  \hfill
  \begin{subfigure}[b]{0.5\textwidth}
    \centering
    Original ($n=62$)
  \end{subfigure}
  \hfill
  \begin{subfigure}[b]{0.29\textwidth}
    \centering
    After ICDM ($n=5$)
  \end{subfigure}

  \begin{subfigure}[c]{0.2\textwidth}
    \centering
    \includegraphics[height=3.35cm]{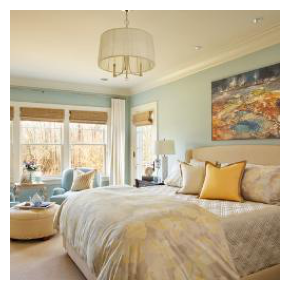}
  \end{subfigure}
  \hfill
  \begin{subfigure}[c]{0.5\textwidth}
    \centering
    \includegraphics[height=3.2cm]{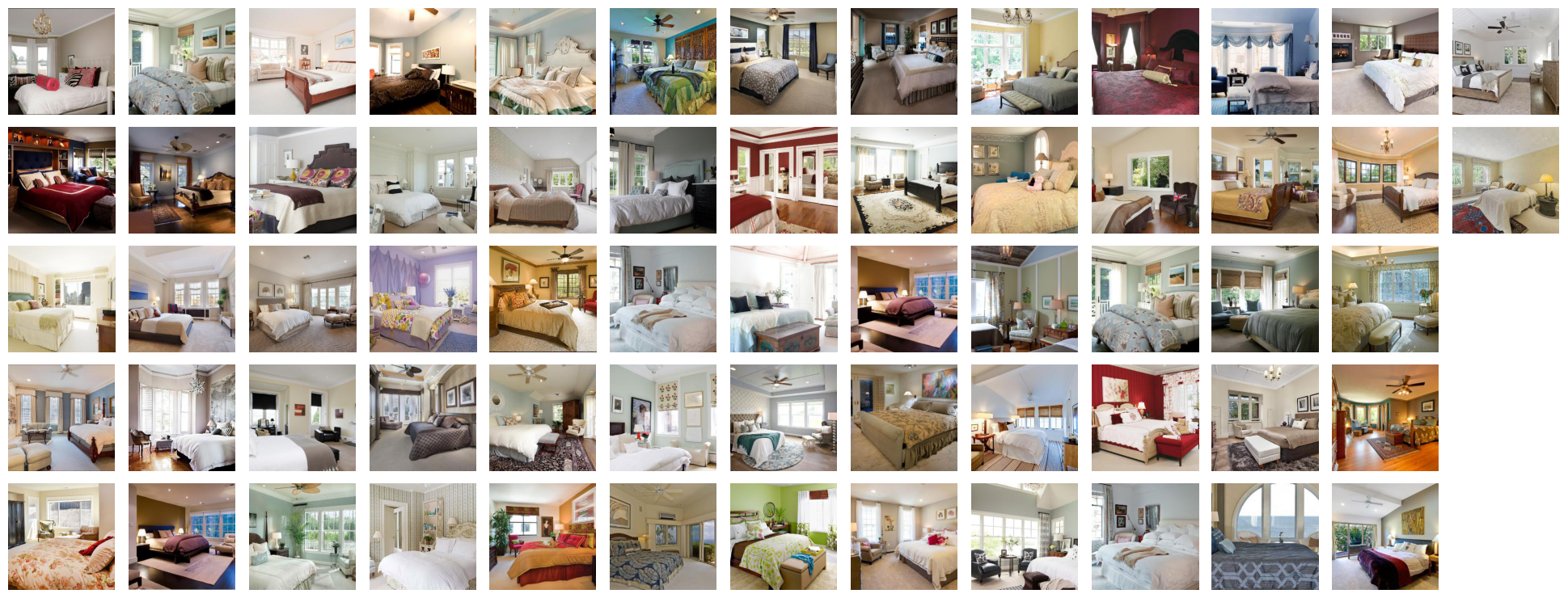}
  \end{subfigure}
  \hfill
  \begin{subfigure}[c]{0.29\textwidth}
    \centering
    \includegraphics[height=3.35cm]{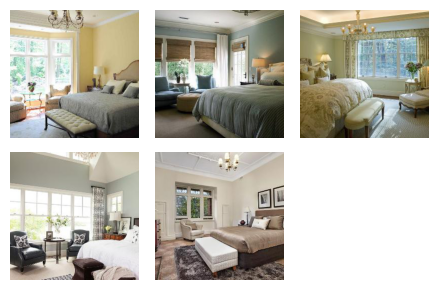}
  \end{subfigure}
  
  \vspace{0.3em}

  \begin{subfigure}[b]{0.2\textwidth}
    \centering
    \textbf{Generated} Hub
  \end{subfigure}
  \hfill
  \begin{subfigure}[b]{0.5\textwidth}
    \centering
    Original ($n=64$)
  \end{subfigure}
  \hfill
  \begin{subfigure}[b]{0.29\textwidth}
    \centering
    After GICDM (Ours) ($n=4$)
  \end{subfigure}

  \begin{subfigure}[c]{0.2\textwidth}
    \centering
    \includegraphics[height=3.35cm]{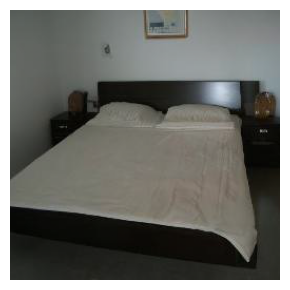}
  \end{subfigure}
  \hfill
  \begin{subfigure}[c]{0.5\textwidth}
    \centering
    \includegraphics[height=3.2cm]{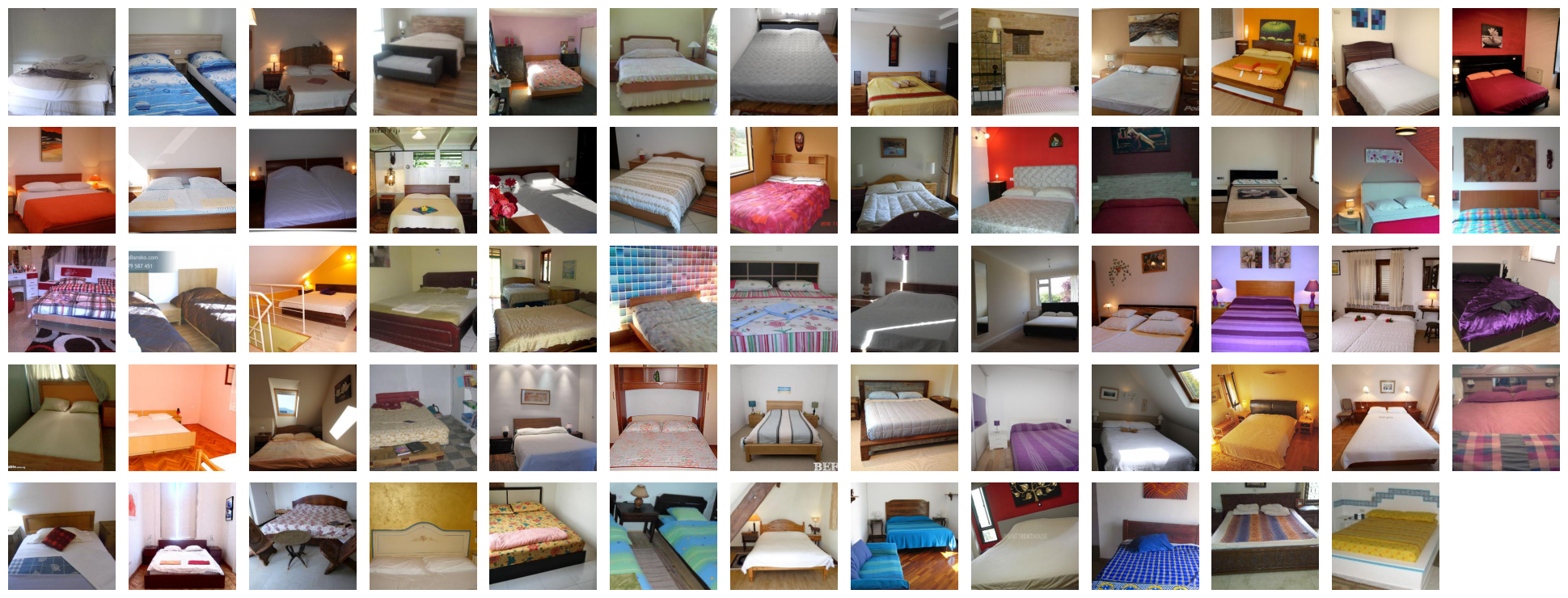}
  \end{subfigure}
  \hfill
  \begin{subfigure}[c]{0.29\textwidth}
    \centering
    \includegraphics[height=3.35cm]{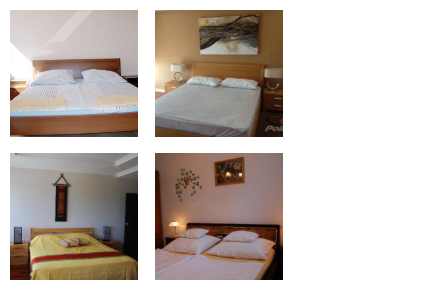}
  \end{subfigure}

  \caption{\textbf{Hubness in DINOv2-embedded LSUN Bedroom and generated sets.} 
  The first column shows the most occurring real (top) and generated (bottom, ADM-dropout) images within the 5-nearest neighborhoods of real images. 
  The middle column displays the real images that include them as neighbors. In both cases, there are more than 60 such real images, identifying the images in the first column as \emph{hubs} and highlighting the asymmetry of neighborhood relationships. After applying hubness reduction (ICDM for real data, GICDM for generated data), the number of real images including these hubs in their neighborhoods drops significantly to 5 and 4, respectively. In both cases, the remaining images listing the hub as a neighbor are more semantically related to it than many of the original ones.}
  \label{fig:intro_viz}
\end{figure}

\newpage
\section{Complexity and Computation time}
\label{sec:complexity}

Let $N$ be the number of real samples and $M$ be the number of generated samples. Computing pairwise distances takes $O(N^2)$ (real-real) and $O(NM)$ (real-generated). Given all pairwise distances, finding the $k$-NN for each
sample matches this cost. ICDM repeats this process $T$ times. Once these distances are computed, the remaining operations (computing $\delta$, ratios, and the threshold) are linear. Thus, the overall computational complexity of GICDM is $O(N(TN + M))$.

Note that standard distance-based metrics also scale as $O(N^2)$ in high dimensions because the most efficient method for
finding the $k$-NN of all $N$ points is a brute-force approach \cite{DBLP:journals/corr/KomarovDD13}, which constructs the distance
matrix and identifies the smallest $k$ elements in each row. Both steps operate in $O(N^2)$ time.

As mentioned in \Cref{sec:experiments}, all experiments were conducted on a single H100 GPU with 80GB of memory.
\begin{itemize}
  \item For $M=50000$ and $N=50000$, computing scores for all metrics discussed in this work takes approximately 30 minutes
  (GICDM only needs to be run once). 42 generated sets are evaluated across 4 different embedding spaces for a total of about 84 hours.
  \item Embedding $50000$ samples takes approximately 4 minutes for Inceptionv3 and DINOv2, 5 minutes for VGG16, and 54 minutes
  for DINOv3. Performing this for 42 generated sets and 4 real sets totals around 51 hours.
  \item Computing all hubness reduction metrics across all tested parameter values takes on average 1.5 hours per embedding-dataset pair. Repeating this for 17 pairs totals about 25 hours.
  \item The real data benchmark in \Cref{sec:real_data_benchmark} required around 50 hours of compute, while the synthetic benchmark in \Cref{appendix:position_benchmark} took roughly 25 hours.
\end{itemize}
The runtime of the remaining experiments is negligible compared to those stated above. In total, the computation time required to reproduce the results in this paper is approximately 250 GPU hours, while the entire research project (including preliminary experiments) used around 500 GPU hours.

For reference, although we had access to the generated samples from \citet{stein2024exposing}, generating 50000 samples on a single H100 GPU takes approximately 70 hours with ADM \cite{dhariwal2021diffusion} and about 3 hours with recent latent diffusion models.

\newpage
\section{Hypersphere test results for each metric}
\label{sec:hypersphere_metric_wise_results}

\Cref{fig:hypersphere_metrics} shows the individual results of standard distance-based metrics on the hypersphere test from \Cref{fig:hypersphere_test}, while \Cref{fig:hypersphere_metrics_gicdm} displays the results after applying GICDM. Without GICDM, all metrics fail the test, whereas with GICDM, all metrics remain at 0, demonstrating that GICDM effectively resolves hubness-related failures.

\begin{figure}[h!]
  \begin{center}
    \subfloat[Precision]{\includegraphics[width=0.23\textwidth]{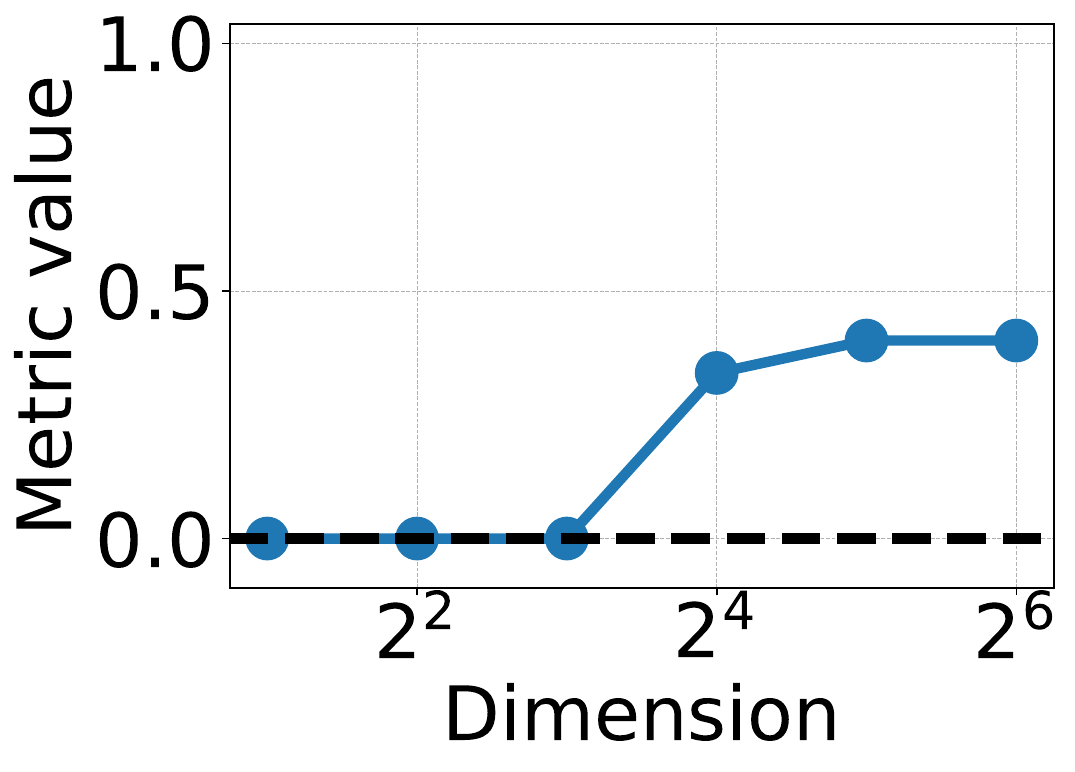}}
    \subfloat[Density]{\includegraphics[width=0.23\textwidth]{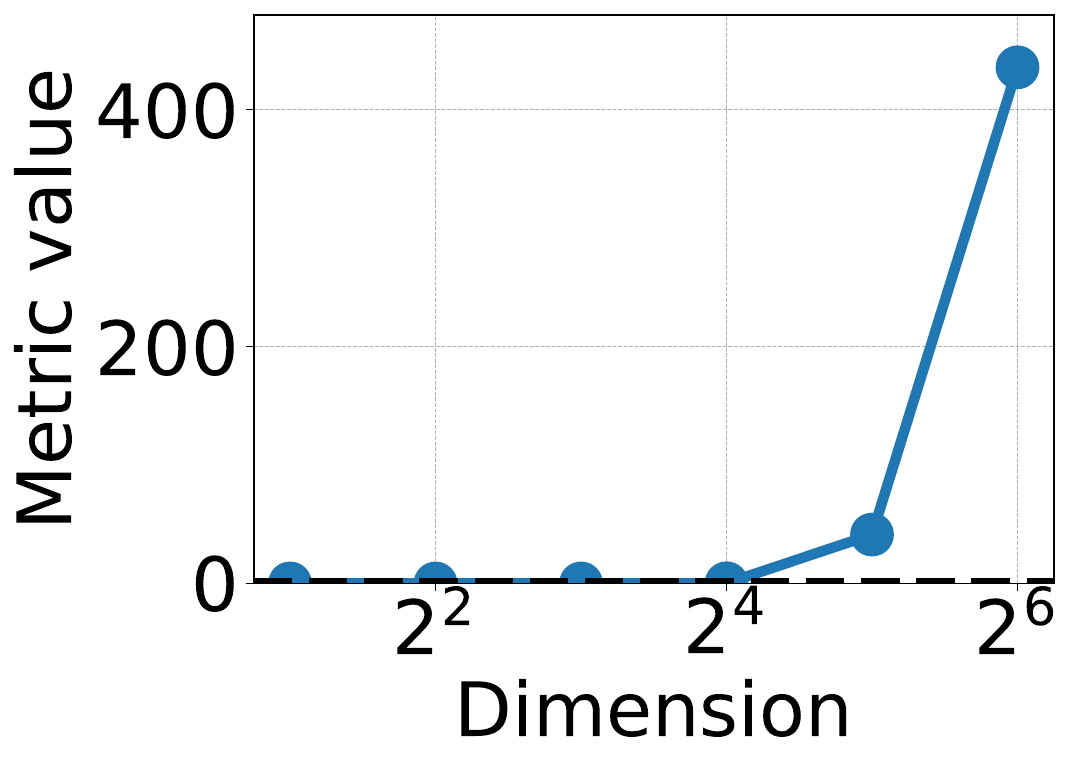}}
    \subfloat[symPrecision]{\includegraphics[width=0.23\textwidth]{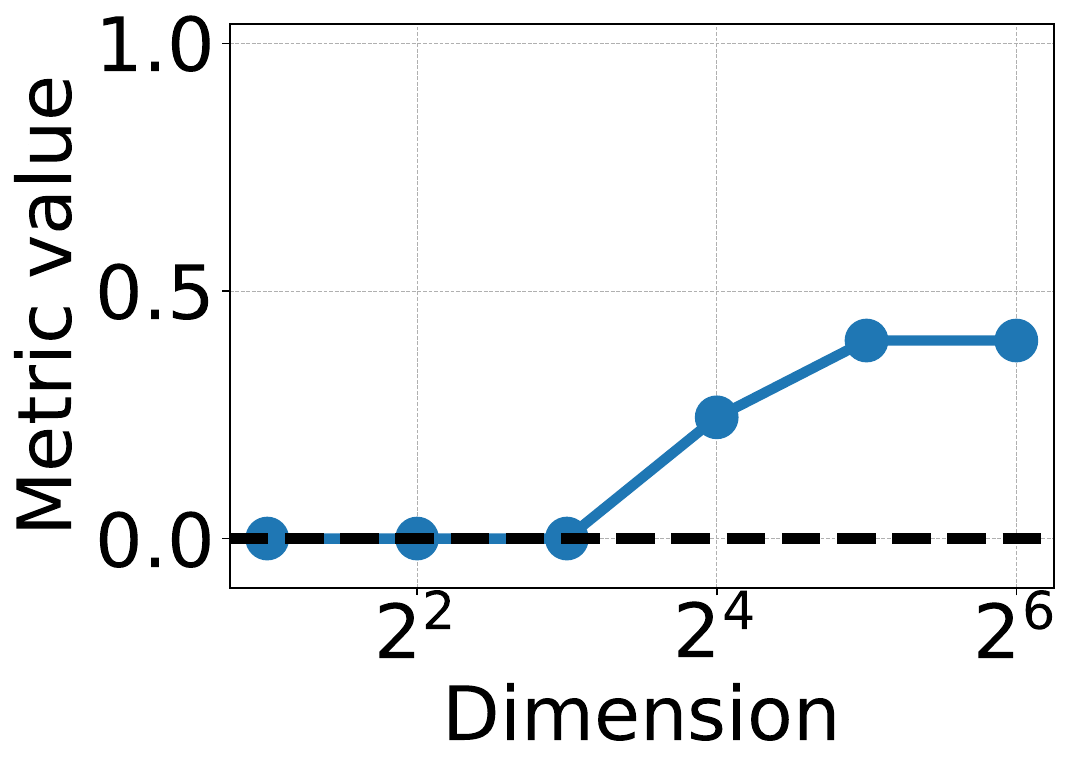}}
    \vfill
    \vspace{0.9em}
    \subfloat[P-precision]{\includegraphics[width=0.23\textwidth]{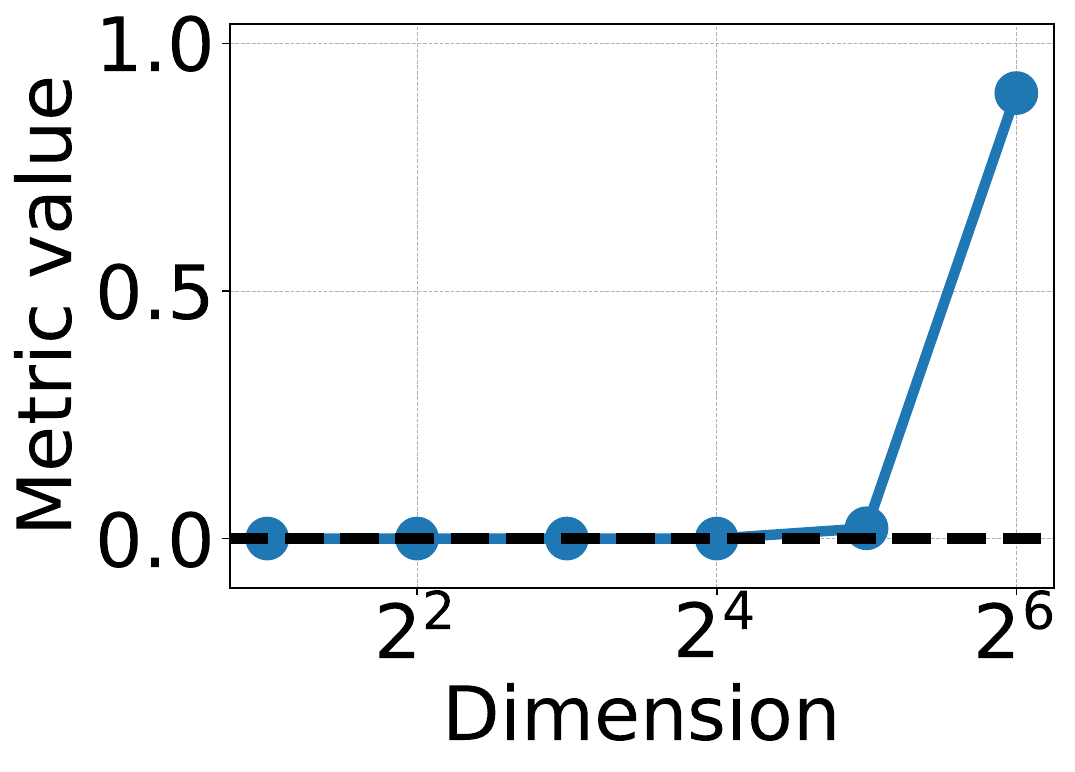}}
    \subfloat[Clipped Density]{\includegraphics[width=0.23\textwidth]{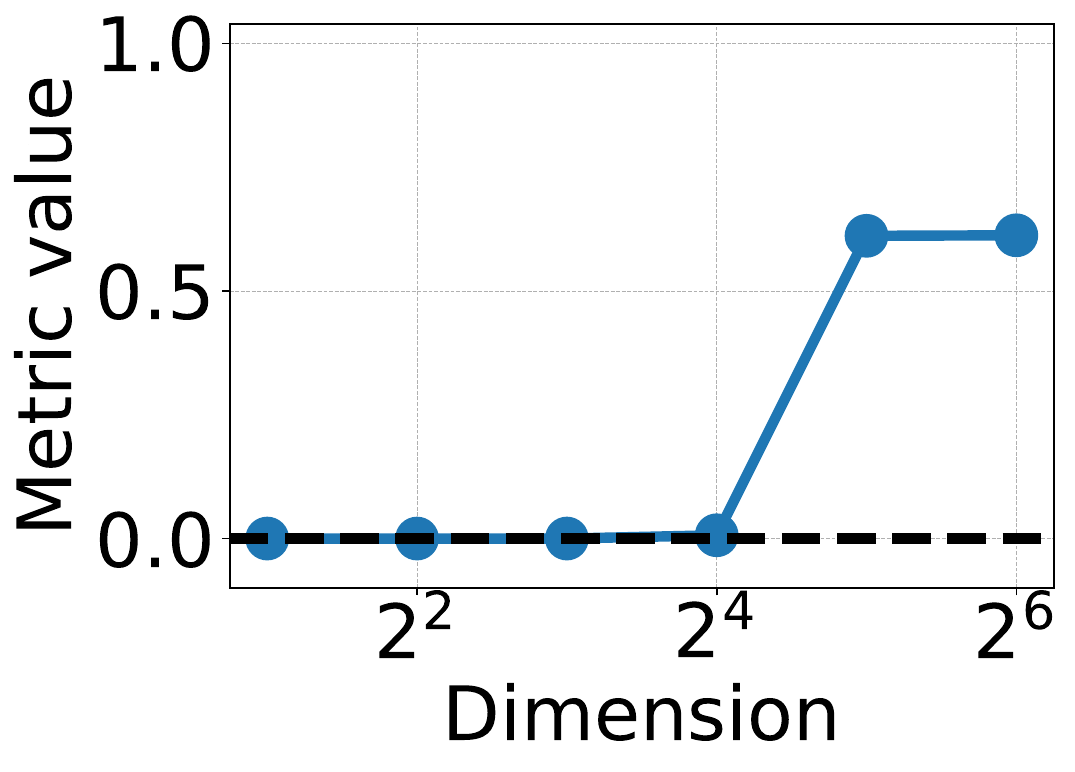}}
    \vfill
    \vspace{0.9em}
    \subfloat[Recall]{\includegraphics[width=0.23\textwidth]{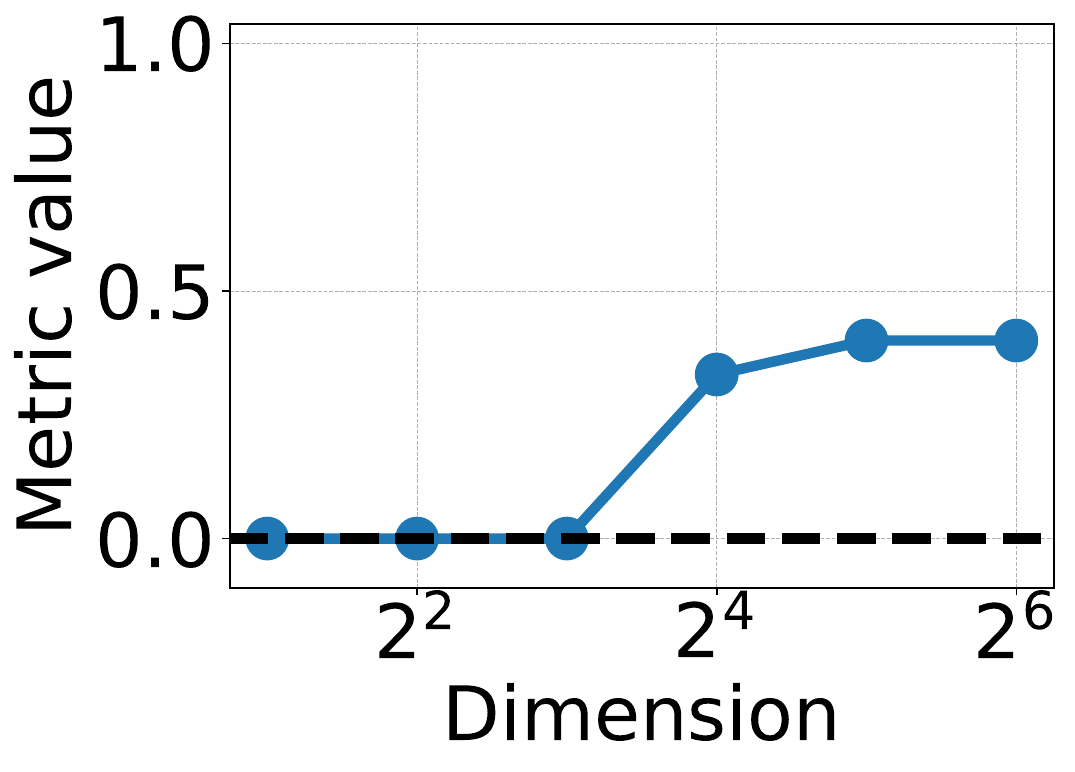}}
    \subfloat[Coverage]{\includegraphics[width=0.23\textwidth]{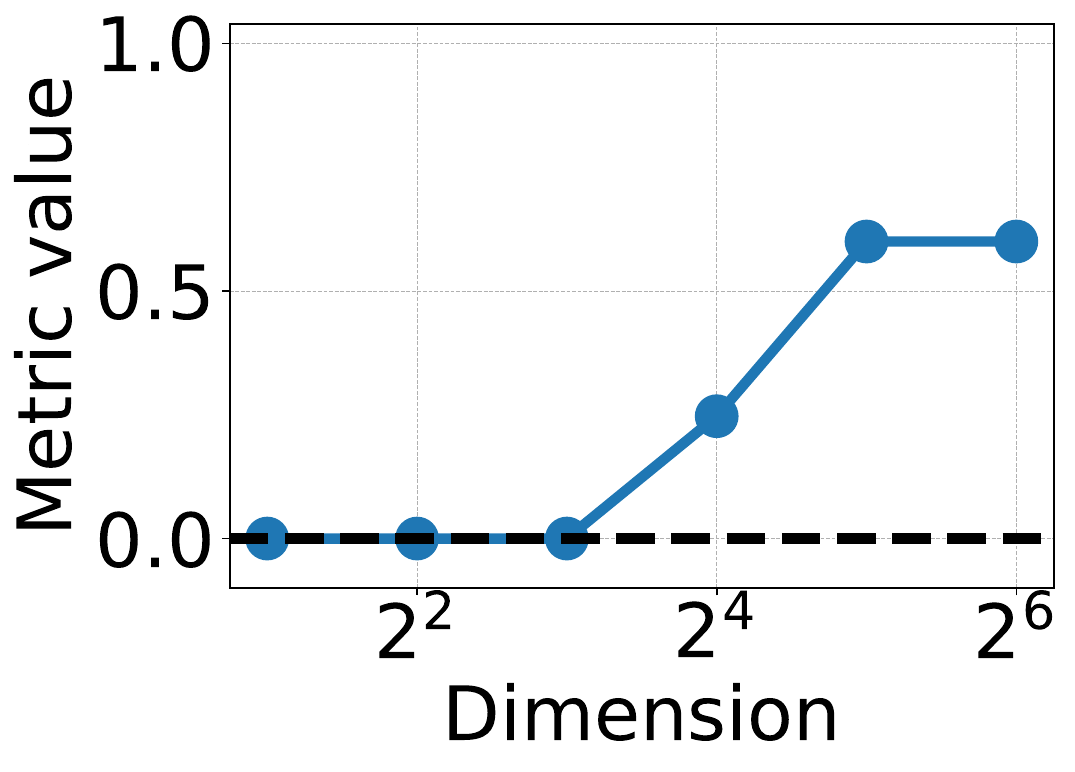}}
    \subfloat[symRecall]{\includegraphics[width=0.23\textwidth]{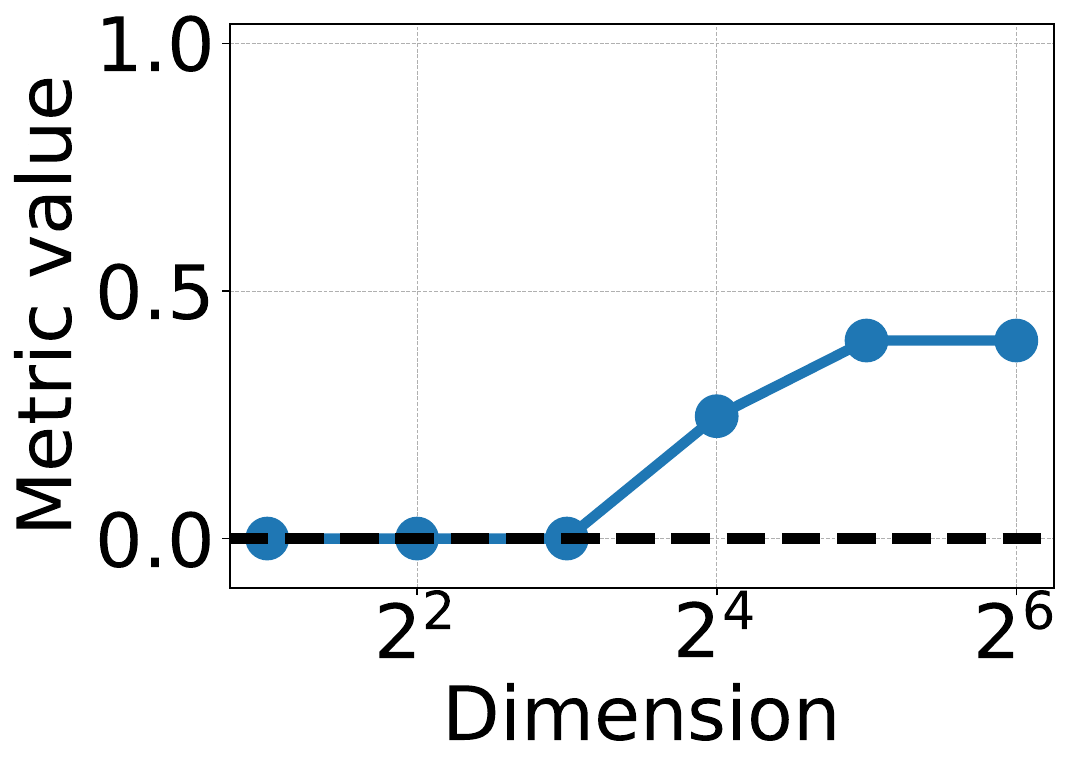}}
    \vfill
    \vspace{0.9em}
    \subfloat[P-recall]{\includegraphics[width=0.23\textwidth]{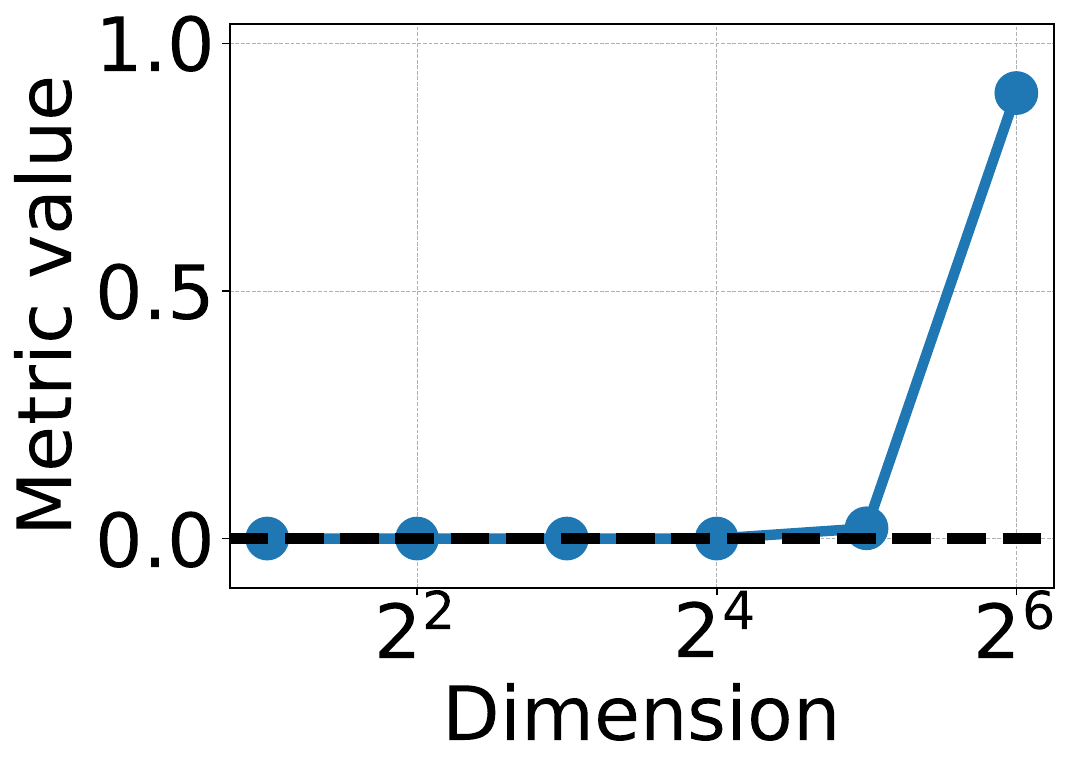}}
    \subfloat[Clipped Coverage]{\includegraphics[width=0.23\textwidth]{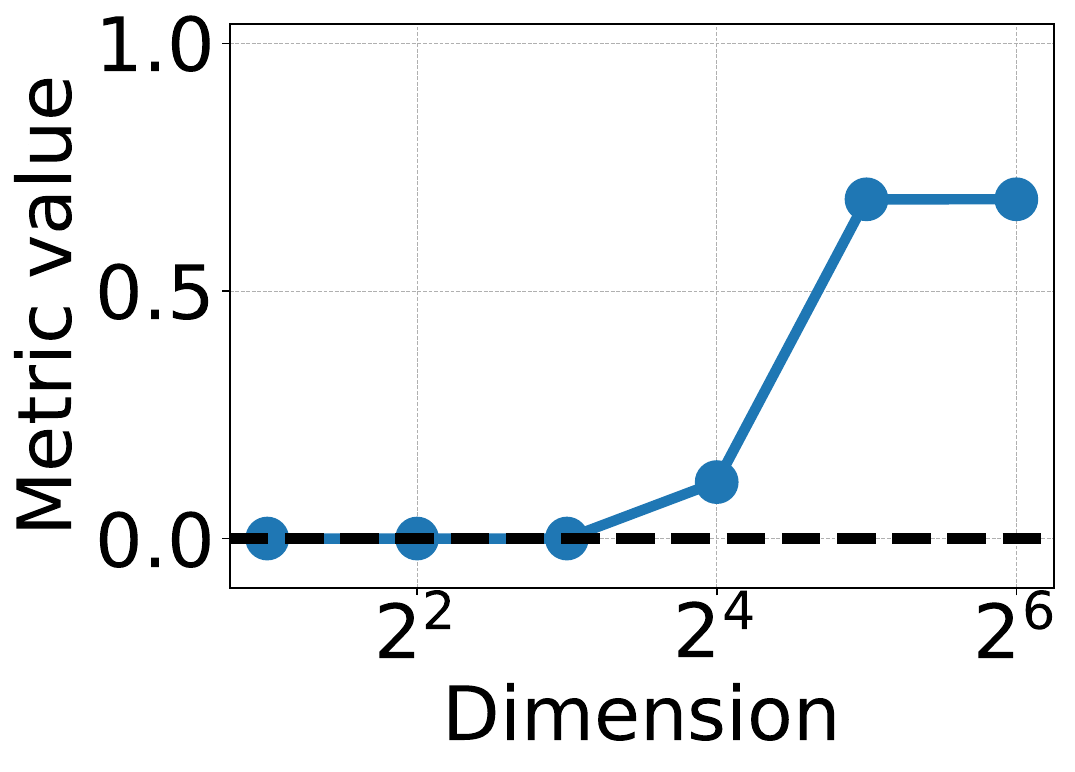}}
    \caption{\textbf{Metrics without GICDM on the Hypersphere Test}: Each subplot shows how a standard metric behaves as the dimension increases in the hypersphere test scenario described in \Cref{fig:hypersphere_test}. Ideally, all metrics should remain at zero for all dimensions, since the real and generated distributions are disjoint.}
    \label{fig:hypersphere_metrics}
  \end{center}
\end{figure}

\begin{figure}[t!]
  \begin{center}
    \subfloat[Precision]{\includegraphics[width=0.23\textwidth]{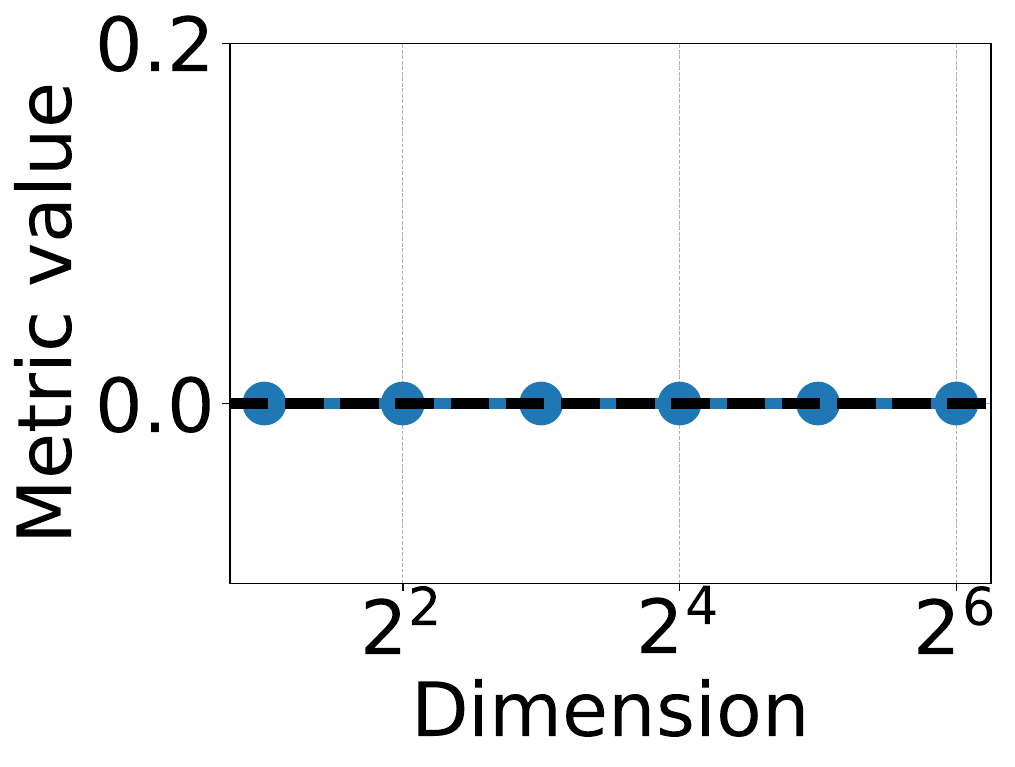}}
    \subfloat[Density]{\includegraphics[width=0.23\textwidth]{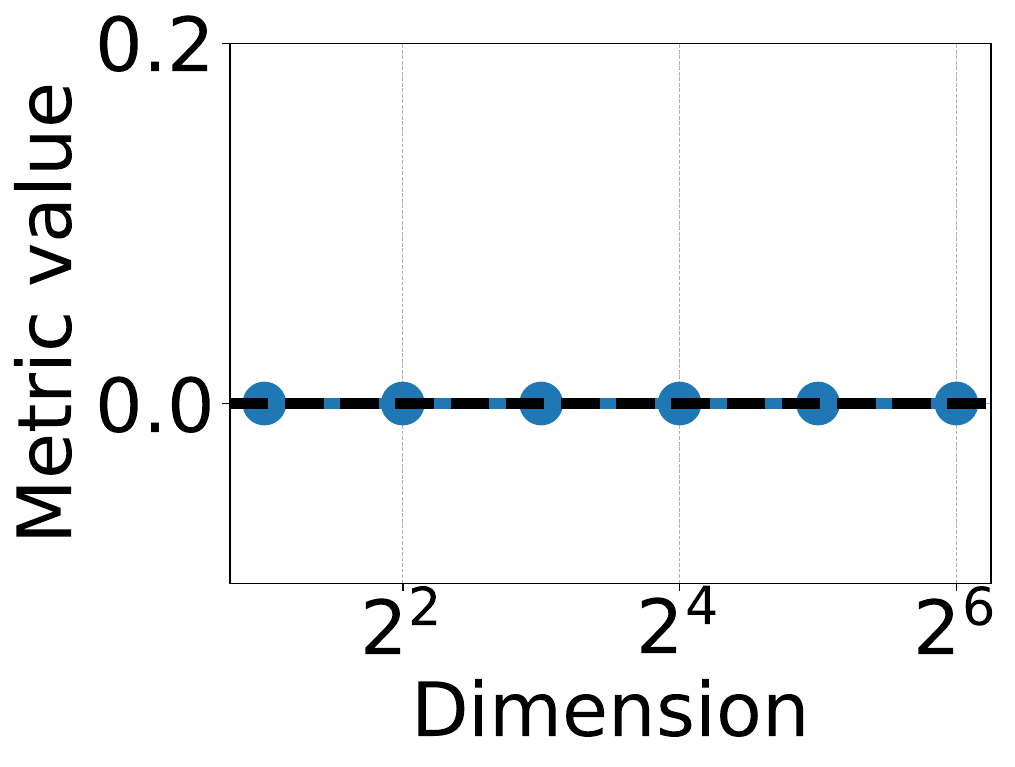}}
    \subfloat[symPrecision]{\includegraphics[width=0.23\textwidth]{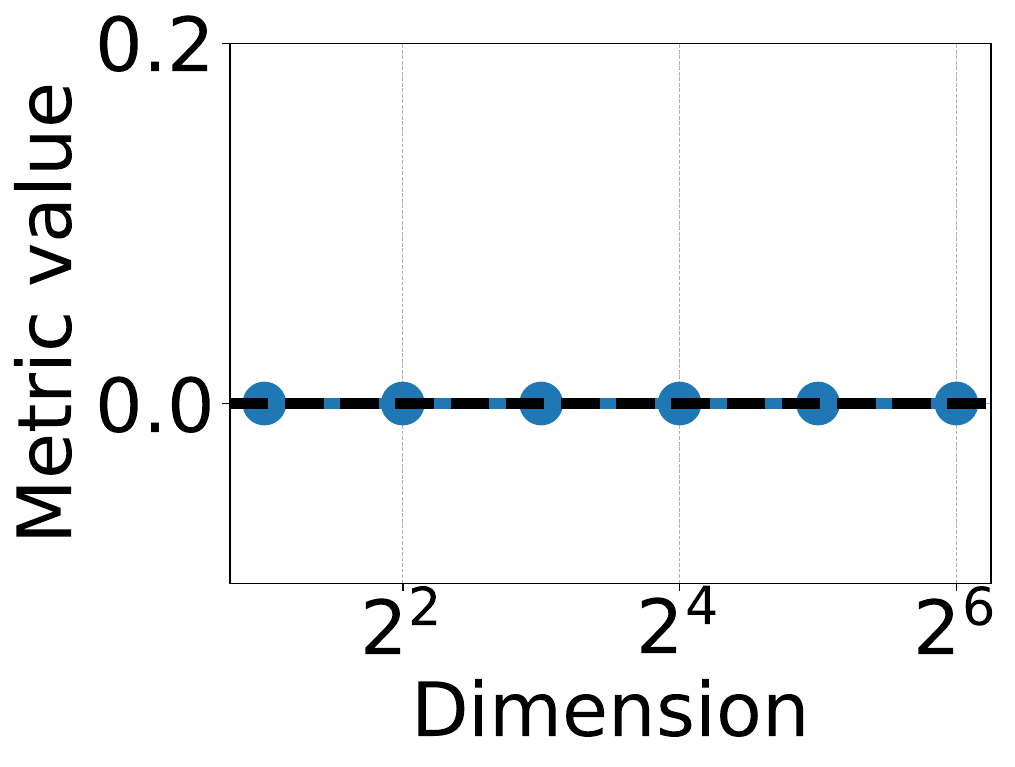}}
    \vfill
    \vspace{0.9em}
    \subfloat[P-precision]{\includegraphics[width=0.23\textwidth]{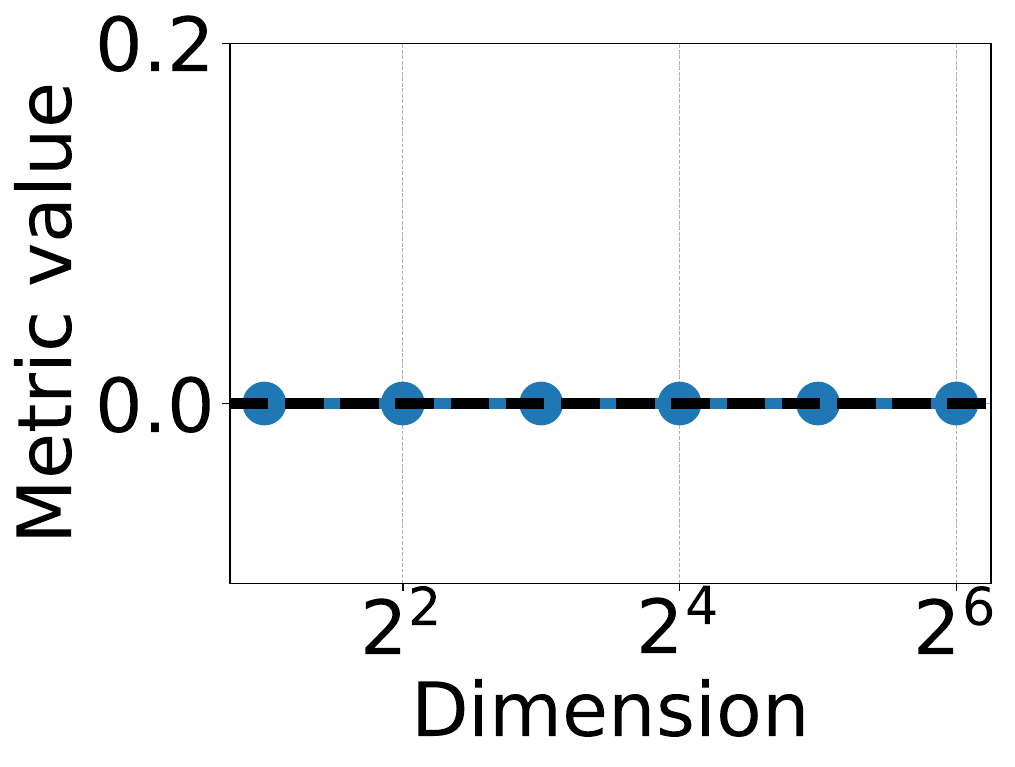}}
    \subfloat[Clipped Density]{\includegraphics[width=0.23\textwidth]{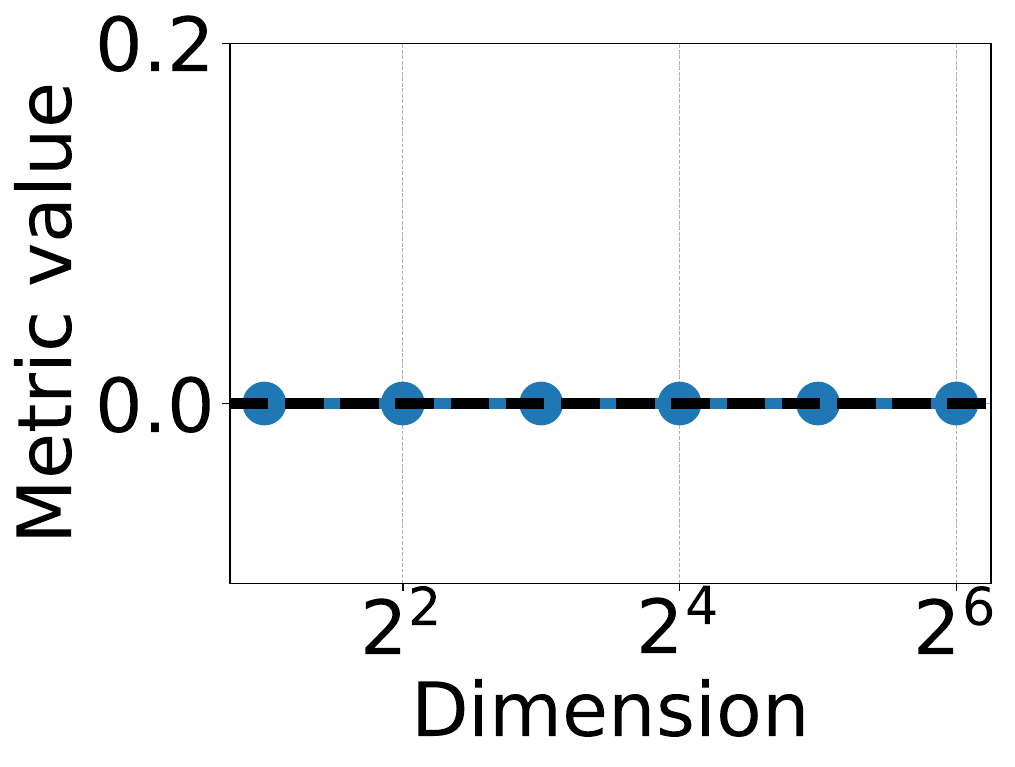}}
    \vfill
    \vspace{0.9em}
    \subfloat[Recall]{\includegraphics[width=0.23\textwidth]{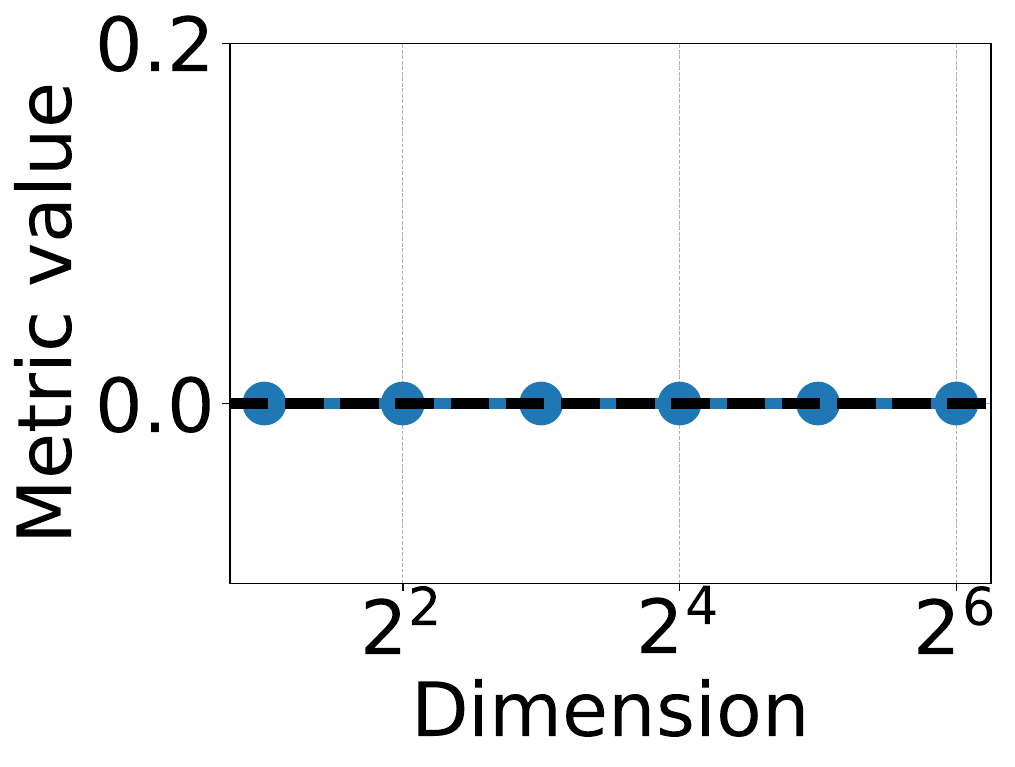}}
    \subfloat[Coverage]{\includegraphics[width=0.23\textwidth]{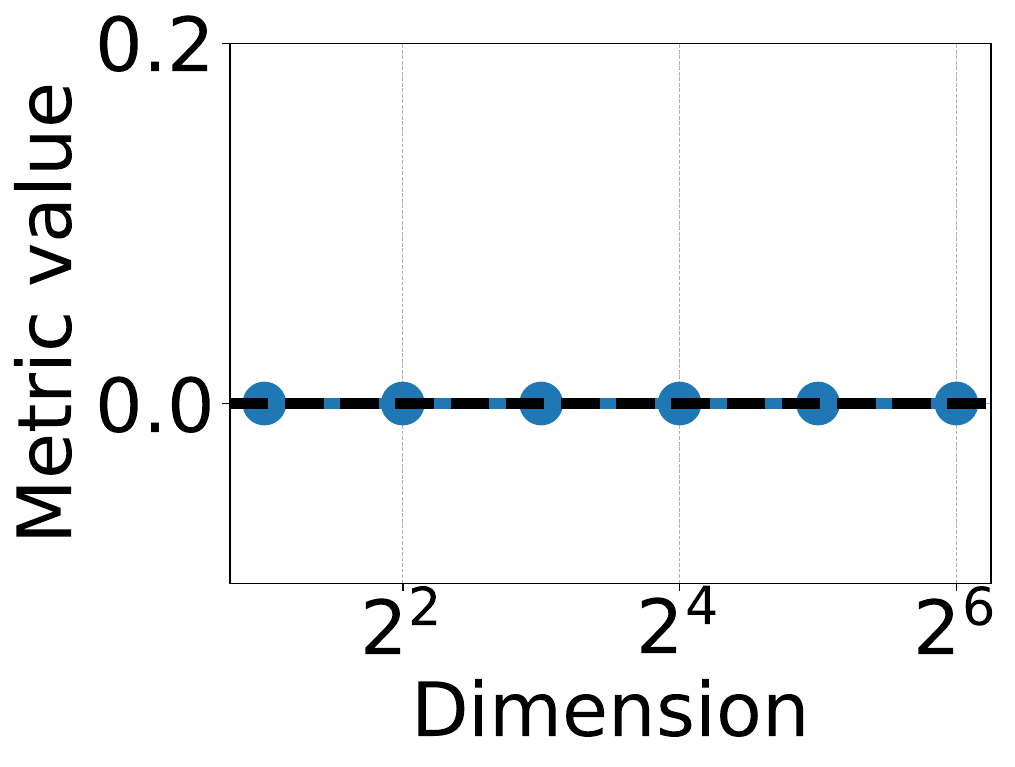}}
    \subfloat[symRecall]{\includegraphics[width=0.23\textwidth]{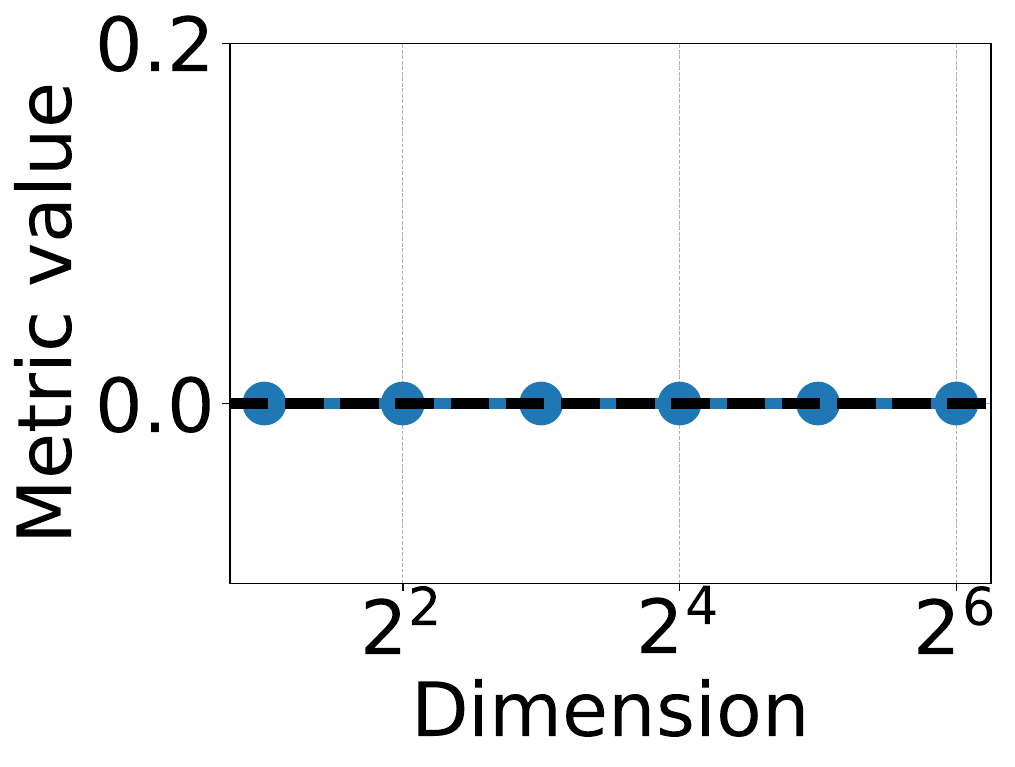}}
    \vfill
    \vspace{0.9em}
    \subfloat[P-recall]{\includegraphics[width=0.23\textwidth]{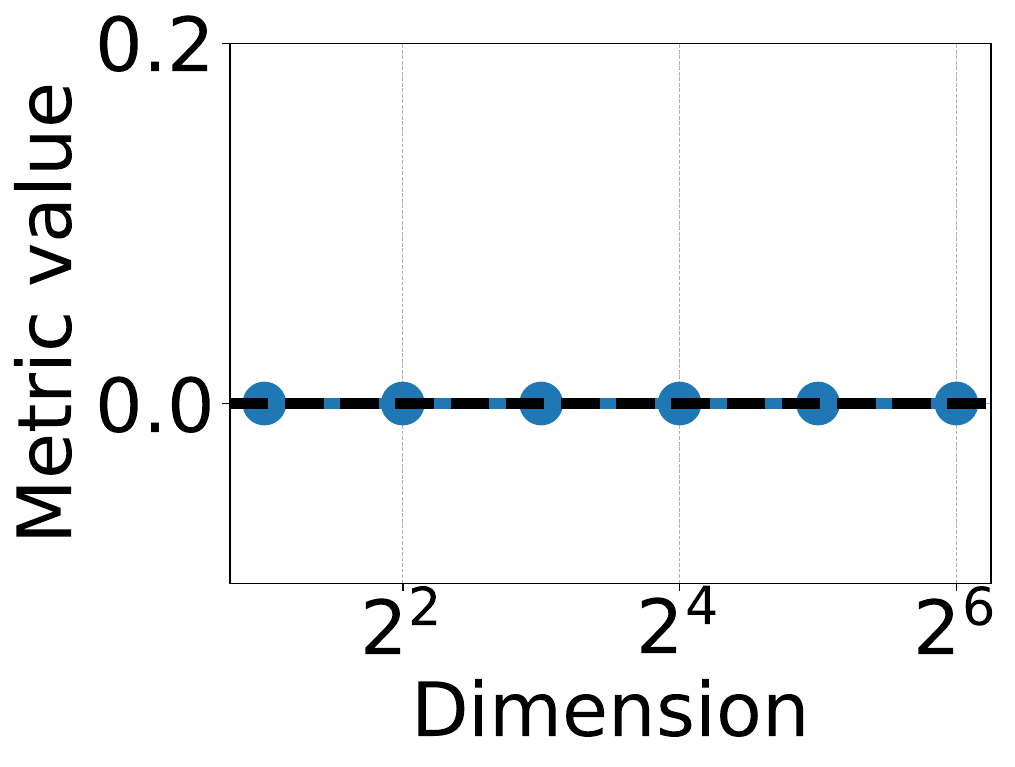}}
    \subfloat[Clipped Coverage]{\includegraphics[width=0.23\textwidth]{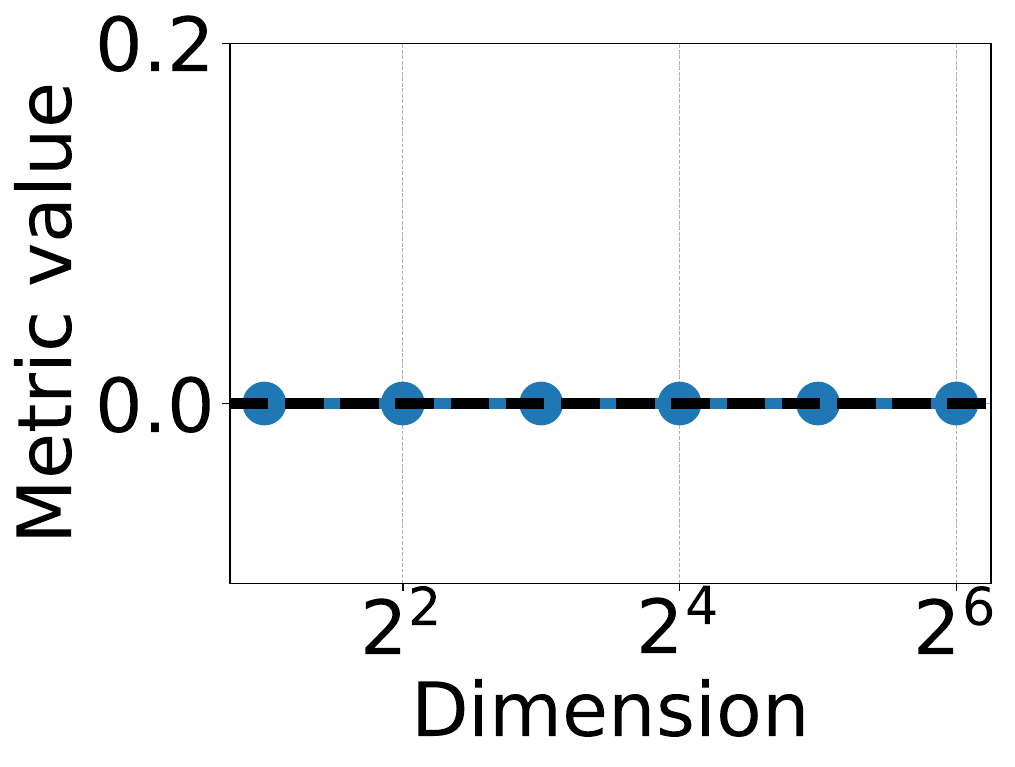}}
    \caption{\textbf{Metrics with GICDM on the Hypersphere Test}: Each subplot shows how a metric, after applying GICDM, behaves as the dimension increases in the hypersphere scenario described in \Cref{fig:hypersphere_test}. GICDM successfully corrects hubness-related failures, keeping metric values at 0.}
    \label{fig:hypersphere_metrics_gicdm}
  \end{center}
\end{figure}

\newpage
\ 

\newpage
\section{In-sample hubness reduction full results}
\label{sec:hubness_reduction_full_results}

For all hubness reduction methods described in \Cref{sec:prior_work}, as well as projection onto a sphere, \Cref{tab:hubness_reduction_full_1,tab:hubness_reduction_full_2} present detailed results for each image dataset and embedding, while \Cref{tab:hubness_reduction_full_3} does so for a music dataset. For methods requiring a neighborhood size $K$, we report the $K$ value in $\{5,10,20,50\}$ that yields the lowest $h_1^5(1\%)$. Results are sorted by dataset and then by $h_1^5(1\%)$.

ICDM consistently achieves the lowest hubness scores across all datasets and embeddings, with optimal $K$ values that are very stable (always $20$, except for Inceptionv3 on ImageNet, where it is $10$). Its $h_1^5(1\%)$ values are always below $2.0$ and $A^5$ remains close to $0.00$, indicating that hubness is effectively eliminated.
Some hubness reduction methods can occasionally perform worse than no correction, particularly on DINOv2 ImageNet, where DSL, LS, and Sphere projection all increase $h_1^5(1\%)$.

To the best of our knowledge, this is the first demonstration of the capabilities of ICDM in that context.

\begin{table}[b!]
  \caption{\textbf{Hubness reduction comparison}: using either DINOv2 (left) or DINOv3 (right). The metrics are $h_1^5(1\%)$, which measures how much more frequently the top $1\%$ of hubs appear in the 5-nearest neighborhoods compared to the average, and $A^5$, the proportion of antihubs. In both cases, lower values are better. On ImageNet with DINOv2 embeddings, some methods perform worse than no correction.
  }
  \label{tab:hubness_reduction_full_1}
  \centering
  \begin{minipage}[r]{0.47\textwidth}
    \begin{small}
    \begin{sc}
    \begin{tabular}{clccc}
    \toprule
    & \multicolumn{4}{c}{DINOv2 Embeddings} \\
    \cmidrule(r){2-5}
    Dataset & Method & K & $h_1^5(1\%)$ & $A^5$ \\
    \midrule
    \multirow{8}{*}{CIFAR10}
      & Original & -- & 6.5 & 0.14 \\
      & Sphere & -- & 6.2 & 0.12 \\
      & $\text{MP}^\text{Gauss}$ & -- & 4.5 & 0.08 \\
      & LS & 10 & 4.0 & 0.03 \\
      & CSLS & 20 & 2.9 & 0.02 \\
      & NICDM & 20 & 2.8 & 0.02 \\
      & DSL & 50 & 2.8 & 0.01 \\
      & ICDM & 20 & 1.7 & 0.00 \\
    \midrule
    \multirow{8}{*}{FFHQ}
      & Original & -- & 5.8 & 0.12 \\
      & Sphere & -- & 5.5 & 0.11 \\
      & $\text{MP}^\text{Gauss}$ & -- & 4.5 & 0.07 \\
      & LS & 10 & 3.6 & 0.02 \\
      & DSL & 50 & 2.8 & 0.01 \\
      & CSLS & 20 & 2.6 & 0.02 \\
      & NICDM & 20 & 2.6 & 0.02 \\
      & ICDM & 20 & 1.7 & 0.00 \\
    \midrule
    \multirow{8}{*}{LSUN Bedroom}
      & Sphere & -- & 6.0 & 0.14 \\
      & Original & -- & 5.8 & 0.13 \\
      & $\text{MP}^\text{Gauss}$ & -- & 4.7 & 0.09 \\
      & LS & 20 & 3.7 & 0.03 \\
      & DSL & 50 & 3.0 & 0.01 \\
      & CSLS & 20 & 2.7 & 0.02 \\
      & NICDM & 20 & 2.7 & 0.02 \\
      & ICDM & 20 & 1.7 & 0.00 \\
    \midrule
    \multirow{8}{*}{ImageNet}
      & DSL & 5 & 7.3 & 0.01 \\
      & LS & 5 & 5.5 & 0.02 \\
      & Sphere & -- & 4.4 & 0.10 \\
      & Original & -- & 4.4 & 0.10 \\
      & $\text{MP}^\text{Gauss}$ & -- & 2.8 & 0.03 \\
      & CSLS & 10 & 2.6 & 0.01 \\
      & NICDM & 10 & 2.5 & 0.01 \\
      & ICDM & 20 & 1.8 & 0.00 \\
    \bottomrule
    \end{tabular}
    \end{sc}
    \end{small}
  \end{minipage}
  \begin{minipage}{0.35\textwidth}
    \begin{small}
    \begin{sc}
    \begin{tabular}{lccc}
    \toprule
    \multicolumn{4}{c}{DINOv3 Embeddings} \\
    \cmidrule(r){1-4}
    Method & K & $h_1^5(1\%)$ & $A^5$ \\
    \midrule

    Original & -- & 9.0 & 0.24 \\
    Sphere & -- & 6.1 & 0.12 \\
    $\text{MP}^\text{Gauss}$ & -- & 5.0 & 0.10 \\
    LS & 5 & 4.5 & 0.05 \\
    CSLS & 20 & 3.4 & 0.05 \\
    NICDM & 20 & 3.2 & 0.04 \\
    DSL & 50 & 3.1 & 0.01 \\
    ICDM & 20 & 1.8 & 0.00 \\
    \midrule

    Original & -- & 10.1 & 0.26 \\
    DSL & 50 & 6.3 & 0.02 \\
    Sphere & -- & 5.6 & 0.10 \\
    $\text{MP}^\text{Gauss}$ & -- & 4.9 & 0.10 \\
    LS & 20 & 4.2 & 0.07 \\
    CSLS & 20 & 3.6 & 0.06 \\
    NICDM & 20 & 3.3 & 0.05 \\
    ICDM & 20 & 1.8 & 0.00 \\
    \midrule
    Original & -- & 9.0 & 0.23 \\
    Sphere & -- & 7.0 & 0.15 \\
    $\text{MP}^\text{Gauss}$ & -- & 5.4 & 0.12 \\
    LS & 20 & 4.2 & 0.07 \\
    DSL & 50 & 3.5 & 0.01 \\
    CSLS & 20 & 3.4 & 0.06 \\
    NICDM & 10 & 3.2 & 0.04 \\
    ICDM & 20 & 1.8 & 0.00 \\
    \midrule
    Original & -- & 8.0 & 0.20 \\
    LS & 50 & 5.0 & 0.14 \\
    Sphere & -- & 4.7 & 0.12 \\
    $\text{MP}^\text{Gauss}$ & -- & 4.0 & 0.10 \\
    CSLS & 20 & 3.3 & 0.06 \\
    DSL & 20 & 3.1 & 0.01 \\
    NICDM & 20 & 3.1 & 0.05 \\
    ICDM & 20 & 1.9 & 0.00 \\
    \bottomrule
    \end{tabular}
    \end{sc}
    \end{small}
  \end{minipage}
\end{table}

\begin{table}[t]
  \caption{\textbf{Hubness reduction comparison}: using either Inceptionv3 (left) or VGG16 (right) embeddings. The metrics and datasets are the same as in \Cref{tab:hubness_reduction_full_1}. The initial hubness levels are generally higher than with DINO embeddings, but ICDM still consistently achieves the lowest hubness scores across all datasets and embeddings.}
  \label{tab:hubness_reduction_full_2}
  \centering
  \begin{minipage}[r]{0.47\textwidth}
    \begin{small}
    \begin{sc}
    \begin{tabular}{clccc}
    \toprule
    & \multicolumn{4}{c}{Inceptionv3 Embeddings} \\
    \cmidrule(r){2-5}
    Dataset & Method & K & $h_1^5(1\%)$ & $A^5$ \\
    \midrule
    \multirow{8}{*}{CIFAR10}
     & Original & -- & 9.0 & 0.22 \\
     & Sphere & -- & 6.2 & 0.12 \\
     & $\text{MP}^\text{Gauss}$ & -- & 4.5 & 0.07 \\
     & LS & 20 & 3.8 & 0.04 \\
     & CSLS & 50 & 3.3 & 0.04 \\
     & NICDM & 20 & 3.2 & 0.03 \\
     & DSL & 50 & 2.9 & 0.01 \\
     & ICDM & 20 & 1.7 & 0.00 \\
    \midrule
    \multirow{8}{*}{FFHQ}
     & Original & -- & 8.3 & 0.20 \\
     & Sphere & -- & 8.1 & 0.17 \\
     & $\text{MP}^\text{Gauss}$ & -- & 4.9 & 0.08 \\
     & LS & 20 & 3.9 & 0.04 \\
     & CSLS & 50 & 3.2 & 0.04 \\
     & NICDM & 20 & 3.1 & 0.03 \\
     & DSL & 50 & 2.6 & 0.01 \\
     & ICDM & 20 & 1.7 & 0.00 \\
    \midrule
    \multirow{8}{*}{LSUN Bedroom}
     & Original & -- & 12.9 & 0.28 \\
     & Sphere & -- & 9.0 & 0.19 \\
     & $\text{MP}^\text{Gauss}$ & -- & 5.5 & 0.10 \\
     & LS & 20 & 4.1 & 0.05 \\
     & CSLS & 50 & 3.8 & 0.06 \\
     & NICDM & 20 & 3.6 & 0.04 \\
     & DSL & 50 & 3.1 & 0.01 \\
     & ICDM & 20 & 1.7 & 0.00 \\
    \midrule
    \multirow{8}{*}{ImageNet}
     & Original & -- & 5.7 & 0.20 \\
     & Sphere & -- & 4.7 & 0.11 \\
     & LS & 50 & 4.7 & 0.09 \\
     & $\text{MP}^\text{Gauss}$ & -- & 4.0 & 0.08 \\
     & DSL & 20 & 3.4 & 0.01 \\
     & CSLS & 20 & 3.0 & 0.04 \\
     & NICDM & 20 & 2.9 & 0.03 \\
     & ICDM & 10 & 1.7 & 0.00 \\
    \bottomrule
    \end{tabular}
    \end{sc}
    \end{small}
  \end{minipage}
  \begin{minipage}{0.35\textwidth}
    \begin{small}
    \begin{sc}
    \begin{tabular}{lccc}
    \toprule
    \multicolumn{4}{c}{VGG16 Embeddings} \\
    \cmidrule(r){1-4}
    Method & K & $h_1^5(1\%)$ & $A^5$ \\
    \midrule
    Original & -- & 10.4 & 0.21 \\
    Sphere & -- & 5.6 & 0.10 \\
    $\text{MP}^\text{Gauss}$ & -- & 4.3 & 0.06 \\
    LS & 50 & 3.7 & 0.04 \\
    CSLS & 50 & 3.3 & 0.04 \\
    NICDM & 20 & 3.2 & 0.03 \\
    DSL & 50 & 2.4 & 0.01 \\
    ICDM & 20 & 1.7 & 0.00 \\
    \midrule
    Original & -- & 9.8 & 0.23 \\
    Sphere & -- & 6.2 & 0.12 \\
    $\text{MP}^\text{Gauss}$ & -- & 4.8 & 0.08 \\
    LS & 20 & 3.7 & 0.04 \\
    CSLS & 20 & 3.3 & 0.04 \\
    NICDM & 20 & 3.2 & 0.03 \\
    DSL & 50 & 2.5 & 0.01 \\
    ICDM & 20 & 1.7 & 0.00 \\
    \midrule
    Original & -- & 16.6 & 0.31 \\
    Sphere & -- & 7.0 & 0.11 \\
    $\text{MP}^\text{Gauss}$ & -- & 4.6 & 0.07 \\
    LS & 50 & 4.3 & 0.06 \\
    CSLS & 50 & 4.0 & 0.06 \\
    NICDM & 20 & 3.8 & 0.04 \\
    DSL & 50 & 2.4 & 0.01 \\
    ICDM & 20 & 1.7 & 0.00 \\
    \midrule
     Original & -- & 12.6 & 0.20 \\
      Sphere & -- & 4.5 & 0.11 \\
      $\text{MP}^\text{Gauss}$ & -- & 3.9 & 0.07 \\
      LS & 20 & 3.9 & 0.05 \\
      CSLS & 20 & 3.2 & 0.04 \\
      NICDM & 20 & 3.1 & 0.03 \\
      DSL & 20 & 2.8 & 0.00 \\
      ICDM & 20 & 1.7 & 0.00 \\
    \bottomrule
    \end{tabular}
    \end{sc}
    \end{small}
  \end{minipage}
\end{table}

\begin{table}[b]
  \caption{\textbf{Hubness reduction comparison}: for CLAP embeddings of 7-second music samples from the Free Music Archive (FMA) dataset. The metrics are the same as in \Cref{tab:hubness_reduction_full_1,tab:hubness_reduction_full_2}. Even in this different modality, hubness is present, and ICDM achieves the best performance.}
  \label{tab:hubness_reduction_full_3}
  \centering
  \begin{small}
  \begin{sc}
  \begin{tabular}{clccc}
    \toprule
    & \multicolumn{4}{c}{CLAP Embeddings} \\
    \cmidrule(r){2-5}
    Dataset & Method & K & $h_1^5(1\%)$ & $A^5$ \\
    \midrule
    \multirow{8}{*}{FMA} & Original      & --  & 4.9 & 0.09 \\
     & Sphere        & --  & 4.6 & 0.08 \\
     & $\text{MP}^\text{Gauss}$  & --  & 3.9 & 0.06 \\
     & LS            & 20  & 3.3 & 0.02 \\
     & CSLS          & 20  & 2.6 & 0.01 \\
     & NICDM         & 20  & 2.5 & 0.01 \\
     & DSL           & 50  & 2.4 & 0.01 \\
     & ICDM          & 20  & 1.8 & 0.00 \\
    \bottomrule
  \end{tabular}
  \end{sc}
  \end{small}
\end{table}

\end{document}